\documentclass[letterpaper]{article}
\usepackage{aaai}
\usepackage{times}
\usepackage[scaled]{helvet}
\usepackage{courier}
\usepackage{amssymb}
\usepackage{amsthm}
\usepackage{amsmath}
\usepackage[final]{graphicx}
\usepackage{xspace}
\usepackage{paralist}
\usepackage[noline,noend,ruled,titlenumbered,linesnumbered]{algorithm2e}
\usepackage{multirow}
\usepackage[draft]{ifdraft}
\usepackage{bussproofs}
\usepackage[table]{xcolor}
\usepackage{thmtools}
\usepackage{thm-restate}
\usepackage{url}
\usepackage{macros}
\allowdisplaybreaks
\title{Answering Conjunctive Queries over $\mathcal{EL}$ Knowledge Bases with \\ Transitive and Reflexive Roles}
\author{%
    Giorgio Stefanoni \and Boris Motik\\
    Department of Computer Science, University of Oxford\\
    Wolfson Building, Parks Road,\\
    Oxford, OX1 3QD, UK
}

\begin{document}

\maketitle

\begin{abstract}
Answering conjunctive queries (CQs) over \el knowledge bases (KBs) with complex
role inclusions is \pspace-hard and in \pspace in certain cases; however, if
complex role inclusions are restricted to role transitivity, a tight upper
complexity bound has so far been unknown. Furthermore, the existing algorithms
cannot handle reflexive roles, and they are not practicable. Finally, the
problem is tractable for acyclic CQs and $\mathcal{ELH}$, and \np-complete for
unrestricted CQs and \elho KBs. In this paper we complete the complexity
landscape of CQ answering for several important cases. In particular, we
present a practicable \np algorithm for answering CQs over \elhso KBs---a logic
containing all of OWL 2 EL, but with complex role inclusions restricted to role
transitivity. Our preliminary evaluation suggests that the algorithm can be
suitable for practical use. Moreover, we show that, even for a restricted class
of so-called \emph{arborescent} acyclic queries, CQ answering over \el KBs
becomes \np-hard in the presence of either transitive or reflexive roles.
Finally, we show that answering arborescent CQs over \elho KBs is tractable,
whereas answering acyclic CQs is \np-hard.

\end{abstract}

\section{Introduction}

Description logics (DLs)~\cite{DLHB-2007-ed} are a family of knowledge
representation languages that logically underpin the Web Ontology Language (OWL
2)~\cite{ghmppss08next-steps}. DL knowledge bases (KBs) provide modern
information systems with a flexible graph-like data model, and answering
conjunctive queries (CQs) over such KBs is a core reasoning service in various
applications~\cite{SWJ-2011}. Thus, the investigation of the computational
properties of CQ answering, as well as the development of practicable
algorithms, have received a lot of attention lately.

For expressive DLs, CQ answering is at least exponential in \emph{combined
complexity}~\cite{GHLS08a,DBLP:conf/ijcai/OrtizRS11}---that is, measured in the
combined size of the query and the KB. The problem is easier for the
DL-Lite~\cite{JAR-2007} and the \el~\cite{babl05} families of DLs, which
logically underpin the QL and the EL profiles of OWL 2, respectively, and
worst-case optimal, yet practicable algorithms are known
\cite{KoLuToWoZa-IJCAI11,DBLP:conf/semweb/Rodriguez-MuroKZ13,DBLP:conf/aaai/EiterOSTX12,DBLP:journals/jodsn/VenetisSS14}.
One can reduce the complexity by restricting the query shape; for example,
answering \emph{acyclic} CQs~\cite{Yannakakis:1981} is tractable in relational
databases. \citeA{DBLP:conf/ijcai/BienvenuOSX13} have shown that answering
acyclic CQs in DL-Lite$_{\textit{core}}$ and $\mathcal{ELH}$ is tractable,
whereas \citeA{DBLP:journals/ai/GottlobKKPSZ14} have shown it to be \np-hard in
DL-Lite$_{\mathcal{R}}$.

In this paper, we consider answering CQs over KBs in the \el family of
languages. No existing practical approach for \el supports complex role
inclusions---a prominent feature of OWL 2 EL that can express complex
properties of roles, including role transitivity. The known upper bound for
answering CQs over \el KBs with complex role inclusions~\cite{KRH:elcq07} runs
in \pspace and uses automata techniques that are not practicable due to
extensive don't-know nondeterminism. Moreover, this algorithm does not handle
transitive roles specifically, but considers complex role inclusions. Hence, it
is not clear whether the \pspace upper bound is optimal in the presence of
transitive roles only; this is interesting because role transitivity suffices
to express simple graph properties such as reachability, and it is a known
source of complexity of CQ answering~\cite{DBLP:conf/ijcai/EiterLOS09}. Thus,
to complete the landscape, we study the combined complexity of answering CQs
over various extensions of \el and different classes of CQs. Our contributions
can be summarised as follows.

In Section \ref{sec:np_algorithm} we present a novel algorithm running in \np
for answering CQs over \elhso KBs---a logic containing all of OWL 2 EL, but
with complex role inclusions restricted to role transitivity---and thus settle
the open question of the complexity for transitive and (locally) reflexive
roles. Our procedure generalises the \emph{combined approach with
filtering}~\cite{DBLP:conf/semweb/LutzSTW13} for \elho by
\citeA{DBLP:conf/aaai/StefanoniMH13}. We capture certain consequences of an
\elhso KB by a datalog program; then, to answer a CQ, we evaluate the query
over the datalog program to obtain \emph{candidate answers}, and then we filter
out unsound candidate answers. Transitive and reflexive roles, however,
increase the complexity of the filtering step: unlike the filtering procedure
for $\elho$, our filtering procedure runs in nondeterministic polynomial time,
and we prove that this is worst-case optimal---that is, checking whether a
candidate answer is sound is an \np-hard problem. To obtain a goal-directed
filtering procedure, we developed optimisations that reduce the number of
nondeterministic choices. Finally, our filtering procedure runs in \np only for
candidate answers that depend on both the existential knowledge in the KB, and
transitive or reflexive roles---that is, our algorithm exhibits pay-as-you-go
behaviour. To evaluate the feasibility of our approach, we implemented a
prototypical CQ answering system and we carried out a preliminary evaluation.
Our results suggest that, although some queries may be challenging, our
algorithm can be practicable in many cases.

In Section \ref{sec:acyclic} we study the complexity of answering acyclic CQs
over KBs expressed in various extensions of \el. We introduce a new class of
\emph{arborescent} queries---tree-shaped acyclic CQs in which all roles point
towards the parent. We prove that answering arborescent queries over \el KBs
with either a single transitive role or a single reflexive role is \np-hard;
this is interesting because \citeA{DBLP:conf/ijcai/BienvenuOSX13} show that
answering acyclic queries over $\mathcal{ELH}$ KBs is tractable, and it shows
that our algorithm from Section \ref{sec:np_algorithm} is optimal for
arborescent (and thus also acyclic) queries. Moreover, we show that answering
unrestricted acyclic CQs is \np-hard for \elho, but it becomes tractable for
arborescent queries.

\ifdraft{All proofs of our results are provided in the appendix.}{All proofs of
our results are provided in a technical report~\cite{TR}.}

\section{Preliminaries}

We use the standard notions of constants, (ground) terms, atoms, and formulas
of first-order logic with the equality predicate
$\approx$~\cite{Fitting:1996:FLA:230183}; we assume that $\top$ and $\bot$ are
unary predicates without any predefined meaning; and we often identify a
conjunction with the set of its conjuncts. A substitution $\sigma$ is a partial
mapping of variables to terms; $\dom{\sigma}$ and $\rng{\sigma}$ are the domain
and the range of $\sigma$, respectively; for convenience, we extend each
$\sigma$ to identity on ground terms; $\sproj{\sigma}{S}$ is the restriction of
$\sigma$ to a set of variables $S$; and, for $\alpha$ a term or a formula,
$\sigma(\alpha)$ is the result of simultaneously replacing each free variable
$x$ occurring in $\alpha$ with $\sigma(x)$. Finally, $\interval{i}{j}$ is the
set ${\{ i, i+1, \ldots, j-1, j \}}$.

\medskip

\textbf{Rules and Conjunctive Queries}\;\; An \emph{existential rule} is a
formula ${\forall \vec{x}\,\forall\vec{y}. \varphi(\vec{x}, \vec{y})
\rightarrow \exists \vec{z}. \psi(\vec{x},\vec{z})}$ where $\varphi$ and $\psi$
are conjunctions of function-free atoms over variables ${\vec{x} \cup \vec{y}}$
and ${\vec{x} \cup \vec{z}}$, respectively. An \emph{equality rule} is a
formula of the form ${\forall \vec{x}. \varphi(\vec{x}) \rightarrow s \approx
t}$ where $\varphi$ is a conjunction of function-free atoms over variables
${\vec{x}}$, and $s$ and $t$ are function-free terms with variables in
$\vec{x}$. A \emph{rule base} $\Sigma$ is a finite set of rules and
function-free ground atoms; $\Sigma$ is a \emph{datalog program} if ${\vec{z} =
\emptyset}$ for each existential rule in $\Sigma$. Please note that $\Sigma$ is
always satisfiable, as $\top$ and $\bot$ are ordinary unary predicates. We
typically omit universal quantifiers in rules.

A \emph{conjunctive query} (CQ) is a formula ${q = \exists \vec{y}.
\psi(\vec{x},\vec{y})}$ where $\psi$ is a conjunction of function-free atoms
over variables ${\vec{x} \cup \vec{y}}$. Variables $\vec{x}$ are the
\emph{answer variables} of $q$. Let ${\vars{q} = \vec{x} \cup \vec{y}}$ and let
$\terms{q}$ be the set of terms occurring in $q$. When $\vec{x}$ is empty, we
call $q$ a \emph{Boolean CQ}.

For $\tau$ a substitution, let ${\tau(q) = \exists \vec{z}. \tau(\psi)}$, where
$\vec{z}$ is obtained from $\vec{y}$ by removing each variable ${y \in
\vec{y}}$ such that $\tau(y)$ is a constant, and by replacing each variable ${y
\in \vec{y}}$ such that $\tau(y)$ is a variable with $\sigma(y)$.

Let $\Sigma$ be a rule base and let ${q = \exists \vec{y}. \psi(\vec{x},
\vec{y})}$ be a CQ over the predicates in $\Sigma$. A substitution $\pi$ is a
\emph{certain answer} to $q$ over $\Sigma$, written ${\Sigma \models \pi(q)}$,
if ${\dom{\pi} = \vec{x}}$, each element of $\rng{\pi}$ is a constant, and ${\I
\models \pi(q)}$ for each model $\I$ of $\Sigma$.

\begin{table}[tb]
    \centering
    \footnotesize
    \caption{Translating \elhso Axioms into Rules}\label{table:Xi}
    \begin{tabular}{r|lcr}
        \textbf{Type}       & \textbf{Axiom}                &                               & \textbf{Rule} \\
        \hline
        1                   & $A \ISA B$                    & $\leadsto$                    & ${A}(x) \rightarrow B(x)$ \\[1ex]
        2                   & $A \ISA \setof{a}$            & $\leadsto$                    & ${A}(x) \rightarrow x \approx a$ \\[1ex]
        3                   & $A_1 \sqcap A_2 \ISA A$       & $\leadsto$                    & ${A_1}(x) \wedge {A_2}(x) \rightarrow A(x)$ \\[1ex]
        4                   & $\SOME{R}{A_1} \ISA A$        & $\leadsto$                    & $R(x,y) \wedge {A_1}(y) \rightarrow {A}(x)$ \\[1ex]
        \multirow{2}{*}{5}  & \multirow{2}{*}{$S \ISA R$}   &\multirow{2}{*}{$\leadsto$}    & $S(x,y) \rightarrow R(x,y)$ \\
                            &                               &                               & $\SELF_{S}(x) \rightarrow \SELF_R(x)$ \\[1ex]
        6                   & $\range{R}{A}$                & $\leadsto$                    & $R(x,y) \rightarrow A(y)$  \\[1ex]
        7                   & $A_1 \ISA \SOME{R}{A}$        & $\leadsto$                    & ${A_1}(x) \rightarrow \exists z. R(x,z)\wedge A(z)$ \\[1ex]
        8                   & $\TRANS{R}$                   & $\leadsto$                    & $R(x,y)\wedge R(y,z) \rightarrow R(x,z)$ \\[1ex]
        9                   & $\REFL{R}$                    & $\leadsto$                    & ${\top}(x) \rightarrow R(x,x)\wedge \SELF_{R}(x)$ \\[1ex]
        10                  & $A \ISA \SOME{R}{\SELF}$      & $\leadsto$                    & ${A}(x) \rightarrow R(x,x)\wedge \SELF_{R}(x)$ \\[1ex]
        11                  & $\SOME{R}{\SELF} \ISA A$      & $\leadsto$                    & $\SELF_R(x) \rightarrow A(x)$ \\
        \hline
    \end{tabular}
\end{table}

\medskip

\textbf{The DL \elhso} is defined w.r.t.\ a signature consisting of mutually
disjoint and countably infinite sets \conceptnames, \rolenames, and \indnames
of \emph{atomic concepts} (i.e., unary predicates), \emph{roles} (i.e., binary
predicates), and \emph{individuals} (i.e., constants), respectively. We assume
that $\top$ and $\bot$ do not occur in $\conceptnames$. Each \elhso knowledge
base can be \emph{normalised} in polynomial time without affecting CQ answers
\cite{Kroetzsch10:elreason}, so we consider only normalised KBs. An \elhso
\emph{TBox} $\T$ is a finite set of axioms of the form shown in the left-hand
side of Table \ref{table:Xi}, where ${A_{(i)} \in \conceptnames \cup
\setof{\top}}$, ${B \in \conceptnames \cup \setof{\top,\bot}}$,
$S,R\in\rolenames$, and ${a \in \indnames}$; furthermore, TBox $\T$ is in \elho
if it contains only axioms of types $1$--$7$. Relation $\subrole$ is the
smallest reflexive and transitive relation on the set of roles occurring in
$\T$ such that ${S \subrole R}$ for each ${S \ISA R \in \T}$. A role $R$ is
\emph{simple} in $\T$ if ${\TRANS{S} \not\in \T}$ for each ${S \in \rolenames}$
with ${S \subrole R}$. An \emph{ABox} \A is a finite set of ground atoms
constructed using the symbols from the signature. An \elhso \emph{knowledge
base} (KB) is a tuple ${\K = \tuple{\T,\A}}$, where \T is an \elhso TBox and an
\A is an ABox such that each role $R$ occurring in axioms of types $10$ or $11$
in \T is simple.

Let $\ind$ be a fresh atomic concept and, for each role $R$, let $\SELF_R$ be a
fresh atomic concept uniquely associated with $R$. Table~\ref{table:Xi} shows
how to translate an \elhso TBox \T into a rule base $\Xi_{\T}$. Furthermore,
rule base $\close{\K}$ contains an atom $\ind(a)$ for each individual $a$
occurring in $\K$, a rule ${A(x) \rightarrow \top(x)}$ for each atomic concept
occurring in $\K$, and the following two rules for each role $R$ occurring in
$\K$.
\begin{align}
    \ind(x) \wedge R(x,x)	& \rightarrow \SELF_{R}(x) \\
	R(x,y) 					& \rightarrow \top(x) \wedge \top(y)
\end{align}
For $\K$ a KB, let ${\Xi_{\K} = \Xi_{\T} \cup \close{\K} \cup \A}$; then, \K is
\emph{unsatisfiable} iff ${\Xi_\K \models \exists y. \bot(y)}$. For $q$ a CQ
and $\pi$ a substitution, we write ${\K \models \pi(q)}$ iff $\K$ is
unsatisfiable or ${\Xi_{\K} \models \pi(q)}$. Our definition of the semantics
of \elhso is unconventional, but equivalent to the usual
one~\cite{Kroetzsch10:elreason}.

\section{Answering CQs over \elhso KBs}\label{sec:np_algorithm}

\begin{figure}
    \centering
    \includegraphics[scale=0.7]{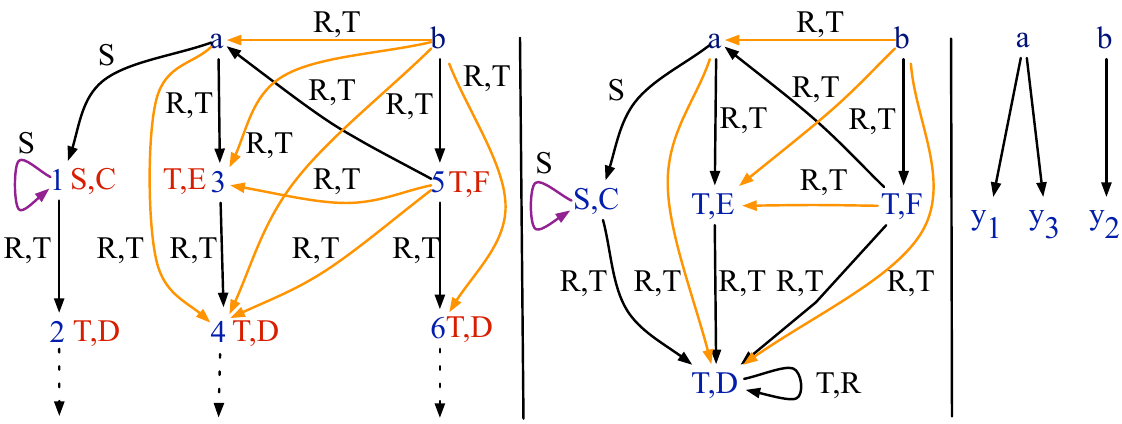}
    \setlength{\abovecaptionskip}{0pt}
    \setlength{\belowcaptionskip}{-5pt}
    \caption{The models of $\Xi_\K$ and $\dat_\K$, and the skeleton for $q$}\label{fig:models}
\end{figure}

In this section, we present an algorithm for answering CQs over \elhso KBs
running in \np. In the rest of this section, we fix ${\K = \tuple{\T,\A}}$ to
be an arbitrary \elhso KB.

Certain answers to a CQ over $\Xi_\K$ can be computed by evaluating the CQ over
a so-called \emph{canonical model} that can be homomorphically embedded into
all other models of $\Xi_\K$. It is well known \cite{KRH:elcq07} that such
models can be seen as a family of directed trees whose roots are the
individuals occurring in $\Xi_\K$, and that contain three kinds edges:
\emph{direct edges} point from parents to children or to the individuals in
$\Xi_\K$; \emph{transitive edges} introduce shortcuts between these trees; and
\emph{self-edges} introduce loops on tree nodes. We call non-root elements
\emph{auxiliary}. Moreover, each auxiliary element can be uniquely associated
with a rule of the form ${A_1(x) \rightarrow\exists z. R(x,z) \wedge A(z)}$ in
$\Xi_\K$ that was used to generate it, and we cal $R,A$ the element's
\emph{type}. Example \ref{example:kb} illustrates these observations.

\begin{example}\label{example:kb}
    Let $\K = \tuple{\T,\A}$ be an \elhso KB whose $\T$ contains the following
    axioms and ${\A = \setof{ A(a), B(b)}}$.
    \begin{align*}
        A   & \ISA \SOME{S}{C}      & E  & \ISA \SOME{T}{D} & G & \ISA \setof{a} \\
        C   & \ISA \SOME{S}{\SELF}  & B  & \ISA \SOME{T}{F} & D & \ISA \SOME{T}{D} \\
        C   & \ISA \SOME{T}{D}      & F  & \ISA \SOME{T}{D} & T & \ISA R \\
        A   & \ISA \SOME{T}{E}      & F  & \ISA \SOME{T}{G} & \span\TRANS{T}
    \end{align*}
    The left part of Figure \ref{fig:models} shows a canonical model $I$ of
    $\Xi_\K$. Each auxiliary element is represented as a number showing the
    element's type. Direct edges are black, transitive edges are orange, and
    self-edges are purple. Axiom ${D\ISA\SOME{T}{D}}$ makes $I$ is infinite;
    in the figure, we show the infinitely many successors of $2$, $4$, and $6$
    using a black dotted edge.
\end{example}

A canonical model $I$ of $\Xi_\K$ can be infinite, so a terminating CQ
answering algorithm for \elhso cannot materialise $I$ and evaluate CQs in it.
Instead, we first show how to translate $\K$ into a datalog program $\dat_\K$
that finitely captures the canonical model $I$ of $\Xi_\K$; next, we present a
CQ answering procedure that uses $\dat_\K$ to answer CQs over $\Xi_\K$.

\subsection{Datalog Translation}

\citeA{KRH:ELP-08} translate $\K$ into datalog for the purposes of ontology
classification, and \citeA{DBLP:conf/aaai/StefanoniMH13} use this translation
to answer CQs over \elho KBs. Let $o_{R,A}$ be an \emph{auxiliary individual}
not occurring in \indnames and uniquely associated with each role
$R\in\rolenames$ and each atomic concept ${A\in\conceptnames\cup \setof{\top}}$
occurring in $\K$; intuitively, $o_{R,A}$ represent all auxiliary terms of type
$R,A$ in a canonical model $I$ of $\Xi_\K$. We extend this translation by
uniquely associating with each role $R$ a \emph{direct predicate}
$\directedge{R}$ to represent the direct edges in $I$.

\begin{definition}\label{def:datalog-rewriting}
    For each axiom $\alpha \in \T$ not of type $7$, set $\dat_\T$ contains the
    translation of $\alpha$ into a rule as shown in Table~\ref{table:Xi};
    moreover, for each axiom ${A_1 \ISA \SOME{R}{A} \in \T}$, set $\dat_\T$
    contains rule ${A_1(x) \rightarrow R(x,o_{R,A}) \wedge
    \directedge{R}(x,o_{R,A}) \wedge A(o_{R,A})}$; finally, for each axiom
    ${S\ISA R \in \T}$, set $\dat_\T$ contains rule ${\directedge{S}(x,y)
    \rightarrow \directedge{R}(x,y)}$. Then, ${\dat_{\K} = \dat_{\T} \cup
    \close{\K} \cup \A}$ is the \emph{datalog program for} $\K$.
\end{definition}

\begin{example}
    The middle part of Figure \ref{fig:models} shows model $J$ of the datalog
    program $\dat_\K$ for the KB from Example \ref{example:kb}. For clarity,
    auxiliary individuals $o_{R,A}$ are shown as $R,A$. Note that auxiliary
    individual $o_{T,G}$ is `merged' in model $J$ with individual $a$ since
    ${\dat_\K \models o_{T,G}\approx a}$. We use the notation from
    Example~\ref{example:kb} to distinguish various kinds of edges.
\end{example}

The following proposition shows how to use $\dat_\K$ to test whether $\K$ is
unsatisfiable.

\begin{restatable}{proposition}{propsat}\label{prop:sat}
    $\K$ is unsatisfiable iff ${\dat_{\K} \models \exists y.\bot(y)}$.
\end{restatable}

\subsection{The CQ Answering Algorithm}

Program $\dat_\K$ can be seen as a strengthening of $\Xi_\K$: all existential
rules $A_1(x)\rightarrow\exists z. R(x,z)\wedge A(z)$ in $\Xi_\K$ are satisfied
in a model $J$ of $\dat_\K$ using a single auxiliary individual $o_{R,A}$.
Therefore, evaluating a CQ $q$ in $J$ produces a set of \emph{candidate
answers}, which provides us with an upper bound on the set of certain answers
to $q$ over $\Xi_\K$.

\begin{definition}
    A substitution $\tau$ is a \emph{candidate answer} to a CQ ${q =\exists
    \vec{y}.\psi(\vec{x}, \vec{y})}$ over $\dat_\K$ if ${\dom{\tau} =
    \vars{q}}$, each element of $\rng{\tau}$ is an individual occurring in
    $\dat_\K$, and ${\dat_\K \models \tau(q)}$. Such a candidate answer $\tau$
    is \emph{sound} if ${\Xi_\K \models \sproj{\tau}{\vec{x}}(q)}$.
\end{definition}

\citeA{DBLP:conf/aaai/StefanoniMH13} presented a filtering step that removes
unsound candidate answers; however, Example \ref{example:elho} shows that this
step can be incomplete when the query contains roles that are not simple.

\begin{example}\label{example:elho}
    Let $\K$ be as in Example \ref{example:kb} and let
    \begin{align*}
        q   &= \exists y. A(x_1)\wedge R(x_1, y) \wedge B(x_2) \wedge R(x_2, y) \wedge D(y).
    \end{align*}
    Moreover, let $\pi$ be the substitution such that $\pi(x_1) = a$ and
    ${\pi(x_2) =b}$, and let $\tau$ be such that $\pi\subseteq \tau$ and
    ${\tau(y)=o_{T,D}}$. Using models $I$ and $J$ from Figure \ref{fig:models},
    one can easily see that ${\Xi_\K \models \pi(q)}$ and ${\dat_\K \models
    \tau(q)}$. However, $q$ contains a `fork' ${R(x_1,y) \wedge
    R(x_2,y)}$, and $\tau$ maps $y$ to an auxiliary individual, so this answer
    is wrongly filtered as unsound.
\end{example}

Algorithm~\ref{algo:issound} specifies a procedure $\sound{q}{\dat_{\K}}{\tau}$
that checks whether a candidate answer is sound. We discuss the intuitions
using the KB from Example \ref{example:kb}, and the query $q$ and the candidate
answer $\tau$ from Example \ref{example:query}.

\begin{example}\label{example:query}
    Let $q$ and $\tau$ be as follows. Using Figure \ref{fig:models}, one can
    easily see that ${\dat_\K \models \tau(q)}$.
    \setlength{\abovedisplayskip}{0.4em}
    \setlength{\belowdisplayskip}{0.4em}
    \begin{displaymath}
    \begin{array}{@{}r@{\;}l@{\;}l@{}}
        q =     & \exists \vec{y}.  & S(x, y_1) \wedge S(y_1,y_1) \wedge R(x, y_3) \wedge D(y_3) \; \wedge \\
                &                   & R(y_2, y_3) \wedge F(y_2) \wedge T(y_2, x) \\[0.6ex]
        \tau =  & \multicolumn{2}{@{}l@{}}{\setof{x \mapsto a,\; y_1 \mapsto o_{S,C},\; y_2\mapsto o_{T,F},\; y_3\mapsto o_{T,D} }} \\
    \end{array}
    \end{displaymath}
\end{example}

We next show how $\sound{q}{\dat_{\K}}{\tau}$ decides that $\tau$ is
sound---that is, that a substitution $\pi$ mapping the variables in $q$ to
terms in $I$ exists such that ${\pi(x) = a}$ and ${\pi(q) \subseteq I}$.
Substitution $\tau$ already provides us with some constraints on $\pi$: it must
map variable $y_1$ to $1$ and variable $y_2$ to $5$, since these are the only
elements of $I$ of types $S,C$ and $T,F$, respectively. In contrast,
substitution $\pi$ can map variable $y_3$ to either one of $2$, $4$, and $6$.
Each such substitution $\pi$ is guaranteed to satisfy all unary atoms of $q$,
all binary atoms of $q$ that $\tau$ maps to direct edges pointing towards
(non-auxiliary) individuals from \indnames, and all binary atoms of $q$ that
contain a single variable and that $\tau$ maps to the self-edge in $J$. Atoms
$T(y_2,x)$ and $S(y_1,y_1)$ in $q$ satisfy these conditions, and so we call
them \emph{good} w.r.t.\ $\tau$. To show that $\tau$ is sound, we must
demonstrate that all other atoms of $q$ are satisfied.

Step $1$ of Algorithm~\ref{algo:issound} implements the `fork' and
`aux-acyclicity' checks \cite{DBLP:conf/aaai/StefanoniMH13}. To guarantee
completeness, we consider only those binary atoms in $q$ that contain simple
roles, and that $\tau$ maps onto direct edges in $J$ pointing towards auxiliary
elements. We call these atoms \emph{aux-simple} as they can be mapped onto the
direct edges in $I$ pointing towards auxiliary elements. In step~$1$ we compute
a new query $q_\sim$ by applying all constraints derived by the fork rule, and
in the rest of the algorithm we consider $q_\sim$ instead of $q$. In our
example, atom $S(x, y_1)$ is the only aux-simple atom, so $q$ does not contain
forks and ${q_\sim = q}$. When all binary atoms occurring in $q$ are good or
aux-simple, step $1$ guarantees that $\tau$ is sound. Query $q$ from
Example~\ref{example:query}, however, contains binary atoms that are neither
good nor aux-simple, so we proceed to step $3$.

Next, in step $3$ we guess a renaming $\sigma$ for the variables in $q_\sim$ to
take into account that distinct variables in $q_\sim$ that $\tau$ maps to the
same auxiliary individual can be mapped to the same auxiliary element of $I$,
and so in the rest of Algorithm~\ref{algo:issound} we consider $\sigma(q_\sim)$
instead of $q_\sim$. In our example, we guess $\sigma$ to be identity, so
${\sigma(q_\sim) = q_\sim = q}$.

In step $4$, we guess a \emph{skeleton} for $\sigma(q_\sim)$, which is a finite
structure that finitely describes the (possibly infinite) set of all
substitutions $\pi$ mapping the variables in $\sigma(q_\sim)$ to distinct
auxiliary elements of $I$. The right part of Figure \ref{fig:models} shows the
skeleton $\S$ for our example query. The vertices of $\S$ are the
(non-auxiliary) individuals from $\dat_\K$ and the variables from
$\sigma(q_\sim)$ that $\tau$ maps to auxiliary individuals, and they are
arranged into a forest rooted in \indnames. Such $\S$ represents those
substitutions $\pi$ that map variables $y_1$ and $y_3$ to auxiliary elements of
$I$ under individual $a$, and that map variable $y_2$ to an auxiliary element
of $I$ under individual $b$.

In steps $5$--$15$, our algorithm labels each edge $\tuple{v',v}\in\S$ with a
set of roles $L(v',v)$; after these steps, $\S$ represents those substitutions
$\pi$ that satisfy the following property (E):
\begin{quote}
    for each role ${P \in L(v',v)}$, a path from $\tau(v')$ to $\tau(v)$ in $J$
    exists that consists only of direct edges labelled by role $P$ pointing to
    auxiliary individuals.
\end{quote}
We next show how atoms of $\sigma(q_\sim)$ that are not good contribute to the
labelling of $\S$. Atom $S(x,y_1)$ is used in step~$6$ to label edge
$\tuple{a,y_1}$. For atom $R(x,y_3)$, in step $8$ we let $P=T$ and we label
edge $\tuple{a,y_3}$ with $P$. Using Figure \ref{fig:models} and the axioms in
Example \ref{example:kb}, one can easily check that the conditions in steps~$8$
and~$9$ are satisfied. For atom $R(y_2,y_3)$, variables $y_2$ and $y_3$ are not
reachable in $\S$, so we must split the path from $y_2$ to $y_3$. Thus, in
step~$8$ we let $P=T$, and in step~$13$ we let $a_t = a$; hence, atom
$R(y_2,y_3)$ is split into atoms $T(y_2,a)$ and $T(a,y_3)$. The former is used
in step $14$ to check that a direct path exists in $J$ connecting $o_{T,F}$
with $a$, and the latter is used to label edge $\tuple{a,y_2}$.

After the for-loop in steps $7$--$15$, skeleton $\S$ represents all
substitutions satisfying (E). In step $17$, function $\mathsf{exist}$ exploits
the direct predicates from $\dat_\K$ to find the required direct paths in $J$,
thus checking whether at least one such substitution exists (see Definition
\ref{def:exist}). Using Figure \ref{fig:models}, one can check that
substitution $\pi$ where $\pi(y_3)=4$ and that maps all other variables as
stated above is the only substitution satisfying the constraints imposed by
$\S$; hence, $\mathsf{isSound}$ returns \true, indicating that candidate answer
$\tau$ is sound.

\medskip

We now formalise the intuitions that we have just presented. Towards this goal,
in the rest of this section we fix a CQ $q'$ and a candidate answer $\tau'$ to
$q'$ over $\dat_\K$.

Due to equality rules, auxiliary individuals in $\dat_\K$ may be equal to
individuals from \indnames, thus not representing auxiliary elements of $I$.
Hence, set $\aux{\dat_\K}$ in Definition~\ref{def:dat-order} provides us with
all auxiliary individuals that are not equal to an individual from \indnames.
Moreover, to avoid dealing with equal individuals, we replace in query $q'$ all
terms that $\tau'$ does not map to individuals in $\aux{\dat_\K}$ with a single
canonical representative, and we do analogously for $\tau'$; this replacement
produces CQ $q$ and substitution $\tau$. Since $q$ and $\tau$ are obtained by
replacing equals by equals, we have $\dat_\K \models \tau(q)$. Our filtering
procedure uses $q$ and $\tau$ to check whether $\tau'$ is sound.

\begin{definition}\label{def:dat-order}
    Let $>$ be a total order on ground terms such that ${o_{R,A} > a}$ for all
    individuals $o_{R,A}$ and ${a \in \indnames}$ from $\dat_{\K}$. Set
    $\aux{\dat_{\K}}$ contains each individual $u$ from $\dat_{\K}$ for which
    no individual $a\in\indnames$ exists such that ${\dat_{\K} \models u
    \approx a}$. For each individual ${u}$ from $\dat_{\K}$, let ${u_\approx =
    u}$ if ${u \in \aux{\dat_{\K}}}$; otherwise, let $u_{\approx}$ be the
    smallest individual ${a \in \indnames}$ in the ordering $>$ such that
    ${\dat_{\K} \models u \approx a}$. Set $\setind{\dat_{\K}}$ contains
    $a_{\approx}$ for each individual $a\in\indnames$ occurring in $\dat_{\K}$.
    Then query $q$ is obtained from $q'$ by replacing each term
    $t\in\terms{q'}$ such that $\tau'(t)\not \in \aux{\dat_\K}$ with
    $\tau'(t)_{\approx}$; substitution $\tau$ is obtained by restricting
    $\tau'$ to only those variables occurring in $q$.
\end{definition}

Next, we define good and aux-simple atoms w.r.t.\ $\tau$.

\begin{definition}\label{def:atom-types}
    Let $R(s,t)$ be an atom where $\tau(s)$ and $\tau(t)$ are defined. Then,
    $R(s,t)$ is \emph{good} if ${\tau(t) \in\indnames}$, or ${s = t}$ and
    ${\dat_\K \models \SELF_{R}(\tau(s))}$. Furthermore, $R(s,t)$ is
    \emph{aux-simple} if ${s\neq t}$, $R$ is a simple role,
    ${\tau(t)\in\aux{\dat_{\K}}}$, and ${\tau(s) = \tau(t)}$ implies
    ${\dat_{\K}\not\models \SELF_{R}(\tau(s))}$.
\end{definition}

Note that, if $R(s,t)$ is not good, then $t$ is a variable and
$\tau(t)\in\aux{\dat_\K}$. Moreover, by the definition of $\dat_\K$, if atom
$R(s,t)$ is aux-simple, then ${\dat_\K\models
\directedge{R}(\tau(s),\tau(t))}$. The following definition introduces the
query $q_\sim$ obtained by applying the fork rule by
\citeA{DBLP:conf/aaai/StefanoniMH13} to only those atoms that are aux-simple.

\begin{definition}\label{def:fork}
    Relation ${\sim \; \subseteq \terms{q} \times \terms{q}}$ for $q$ and
    $\tau$ is the smallest reflexive, symmetric, and transitive relation closed
    under the $\mathsf{fork}$ rule.
    \begin{displaymath}
        \AxiomC{$s'\sim t'$}
        \LeftLabel{$\mathsf{(fork)}$}
        \RightLabel{\begin{tabular}{@{\;}l@{}}
                        \footnotesize{$R(s,s')$ and $P(t,t')$ are aux-simple} \\
                        \footnotesize{ atoms in $q$ w.r.t.\ $\tau$}
                    \end{tabular}}
        \UnaryInfC{$s \sim t$}
        \DisplayProof
    \end{displaymath}
    Query $q_\sim$ is obtained from query $q$ by replacing each term
    ${t\in\terms{q}}$ with an arbitrary, but fixed representative of the
    equivalence class of $\sim$ that contains $t$.
\end{definition}

To check whether $q_\sim$ is aux-acyclic, we next introduce the
\emph{connection graph} $\mathsf{cg}$ for $q$ and $\tau$ that contains a set
$E_s$ of edges $\tuple{v',v}$ for each aux-simple atom ${R(v',v) \in q_\sim}$.
In addition, $\mathsf{cg}$ also contains a set $E_t$ of edges $\tuple{v',v}$
that we later use to guess a skeleton for $\sigma(q_\sim)$ more efficiently. By
the definition of aux-simple atoms, we have $E_s\subseteq E_t$.

\begin{definition}\label{def:connection-graph}
    The \emph{connection graph for $q$ and $\tau$} is a triple
    ${\mathsf{cg}=\tuple{V,E_s, E_t}}$ where ${E_s, E_t \subseteq V \times
    V}$ are smallest sets satisfying the following conditions.
    \begin{itemize}
        \item ${V = \setind{\dat_{\K}} \cup \setof{ z \in \vars{q_\sim} \mid \tau(z) \in \aux{\dat_\K}}}$.

        \item Set $E_s$ contains $\tuple{v',v}$ for all ${v',v \in V}$ for
        which a role $R$ exist such that $R(v',v)$ is an aux-simple atom in
        $q_{\sim}$.

        \item Set $E_t$ contains $\tuple{v',v}$ for all ${v',v \in V}$ such
        that individuals ${\setof{u_1, \ldots, u_n} \subseteq \aux{\dat_\K}}$
        and roles $R_1,\ldots, R_n$ exist with ${n > 0}$, ${u_n = \tau(v)}$,
        and ${\dat_{\K}\models \directedge{R_i}(u_{i-1}, u_i)}$ for each ${i
        \in \interval{1}{n}}$ and ${u_0 = \tau(v')}$.
    \end{itemize}
\end{definition}

Function $\dsound{q}{\dat_{\K}}{\tau}$ from Definition \ref{def:dsound} ensures
that $\tau$ satisfies the constraints in $\sim$, and that $q_\sim$ does not
contain cycles consisting only of aux-simple atoms.

\begin{definition}\label{def:dsound}
    Function $\dsound{q}{\dat_{\K}}{\tau}$ returns $\true$ if and only if the two following conditions hold.
    \begin{enumerate}
        \item\label{spur:cond1} For all $s,t\in\terms{q}$, if ${s \sim t}$,
        then ${\tau(s) = \tau(t)}$.

        \item\label{spur:cond2} $\tuple{V,E_s}$ is a directed acyclic graph.
    \end{enumerate}
\end{definition}

We next define the notions of a variable renaming for $q$ and $\tau$, and of a
skeleton for $q$ and $\sigma$.

\begin{definition}\label{def:var-renaming}
    A substitution $\sigma$ with ${\dom{\sigma} = V \cap
    \vars{q}}$ and ${\rng{\sigma} \subseteq \dom{\sigma}}$ is a
    \emph{variable renaming for $q$ and $\tau$} if
    \begin{enumerate}
        \item for each ${v \in \dom{\sigma}}$, we have $\tau(v) =
        \tau(\sigma(v))$,

        \item for each $v\in\rng{\sigma}$, we have $\sigma(v) =v$, and

        \item directed graph $\tuple{\sigma(V), \sigma(E_s)}$ is a forest.
    \end{enumerate}
\end{definition}

\begin{definition}\label{def:skeleton}
    A \emph{skeleton} for $q$ and a variable renaming $\sigma$ is a directed
    graph ${\S = \tuple{\V, \E}}$ where $\V = \sigma(V)$, and $\E$ satisfies
    $\sigma(E_s) \subseteq \E \subseteq \sigma(E_t)$ and it is a forest whose
    roots are the individuals occurring in $\V$.
\end{definition}

Finally, we present function $\mathsf{exist}$ that checks whether one can
satisfy the constraints imposed by the roles $L(v',v)$ labelling a skeleton
edge ${\tuple{v',v} \in \E}$.

\begin{definition}\label{def:exist}
    Given individuals $u'$ and $u$, and a set of roles ${L}$, function
    ${\exist(u',u,L)}$ returns \true if and only if individuals ${\setof{u_1,
    \ldots, u_n} \subseteq \aux{\dat_\K}}$ with ${n > 0}$ and ${u_n = u}$ exist
    where
    \begin{itemize}
        \item if ${S \in L}$ exists such that ${\TRANS{S} \not\in \T}$, then $n
        = 1$; and

        \item $u_0 =u'$, and ${\dat_{\K}\models \directedge{R}(u_{i-1}, u_i)}$
        for each ${R \in L}$ and each ${i \in \interval{1}{n}}$.
    \end{itemize}
\end{definition}

\begin{algorithm}[t]
    \small
    \DontPrintSemicolon
    \lIf{$\dsound{q}{\dat_{\K}}{\tau} = \false$}{\KwRet \false}
    \textbf{return} \true \textbf{if} each ${R(s,t) \in q_\sim}$ is good or aux-simple\;
    \textbf{guess} a variable renaming $\sigma$ for $q$ and $\tau$\;
    \textbf{guess} a skeleton $\S =\tuple{\V,\E}$ for $q$, $\sigma$, and $\tau$\;
    \textbf{for} $\tuple{v',v}\in \E$, let $L(v',v) = \emptyset$\;
    \textbf{for} \emph{aux-simple atom $R(s,t) \in \sigma(q_{\sim})$}, add $R$ to $L(s,t)$\;
    \For{neither good nor aux-simple $R(s,t) \in \sigma(q_{\sim})$}{
        \textbf{guess}  role $P$ s.t.\ ${\dat_{\K}\models {P}(\tau(s),\tau(t))}$ and $P\subrole R$\;
        \lIf{$\tuple{s,t}\not\in \E \text{ and } \TRANS{P}\not\in\T$}{\KwRet \false}
        \If{$s$ reaches $t$ in $\E$}{
            let $v_0, \ldots, v_n$ be the path from $s$ to $t$ in $\E$\;
        }
        \Else{
            let $a_t$ be the root reaching $t$ in $\E$ via $v_0, \ldots, v_n$\;
            \lIf{$\dat_{\K} \not\models P(\tau(s),a_t)$}{\KwRet \false}
        }
        \textbf{for} $i \in \interval{1}{n}$, add $P$ to $L(v_{i-1},v_i)$\;
    }
    \For{$\tuple{v',v}\in \E$}{
        \lIf{$\exist(\tau(v'),\tau(v), L(v',v)) = \false$}{\KwRet \false}
    }
    \KwRet \true\;
    \caption{\footnotesize $\sound{q}{\dat_{\K}}{\tau}$}\label{algo:issound}
\end{algorithm}

Candidate answer $\tau'$ for $q'$ over $\dat_\K$ is sound, if the
nondeterministic procedure $\sound{q}{\dat_{\K}}{\tau}$ from
Algorithm~\ref{algo:issound} returns \true, as shown by Theorem
\ref{th:correctness}.

\begin{theorem}\label{th:correctness}
    Let $\pi'$ be a substitution. Then ${\Xi_\K \models \pi'(q')}$ iff $\K$ is
    unsatisfiable, or a candidate answer $\tau'$ to $q'$ over $\dat_\K$ exists
    such that $\sproj{\tau'}{\vec{x}} = \pi'$ and the following conditions hold:
    \begin{enumerate}
        \item for each $x\in\vec{x}$, we have $\tau'(x)\in\indnames$, and

        \item a nondeterministic computation exists such that function
        ${\sound{q}{\dat_{\K}}{\tau}}$ returns $\true$.
    \end{enumerate}
\end{theorem}

The following results show that our function $\mathsf{isSound}$
runs in nondeterministic polynomial time.
\begin{restatable}{theorem}{ucomplexity}\label{th:complexity-ubound}
    Function ${\sound{q}{\dat_{\K}}{\tau}}$ can be implemented so that
    \begin{enumerate}
        \item it runs in nondeterministic polynomial time,

        \item if each binary
        atom in $q$ is either good or aux-simple w.r.t.\ $\tau$, it runs in polynomial time, and

        \item if the TBox \T and the query $q$ are fixed, it runs in polynomial
        time in the size of the ABox $\A$.
    \end{enumerate}
\end{restatable}

Each rule in $\dat_\K$ contains a fixed number of variables, so we can compute
all consequences of $\dat_\K$ using polynomial time. Thus, we can compute CQ
$q$ and substitution $\tau$ in polynomial time, and by Proposition
\ref{prop:sat}, we can also check whether $\K$ is unsatisfiable using
polynomial time; hence, by Theorem~\ref{th:complexity-ubound}, we can check
whether a certain answer to $q'$ over $\Xi_\K$ exists using nondeterministic
polynomial time in combined complexity (i.e., when the ABox, the TBox, and the
query are all part of the input), and in polynomial time in data complexity
(i.e., when the TBox and the query are fixed, and only the ABox is part of the
input).

The filtering procedure by \citeA{DBLP:conf/aaai/StefanoniMH13} is polynomial,
whereas the one presented in this paper introduces a source of intractability.
In Theorem \ref{th:complexity-lbound} we show that checking whether a candidate
answer is sound is an \np-hard problem; hence, this complexity increase is
unavoidable. We prove our claim by reducing the \np-hard problem of checking
satisfiability of a 3CNF formula $\varphi$~\cite{Garey:1979:CIG:578533}.
Towards this goal, we define an \elhso KB $\K_\varphi$ and a Boolean CQ
$q_\varphi$ such that $\varphi$ is satisfiable if and only if
${\Xi_{\K_\varphi} \models q_\varphi}$. Furthermore, we define a substitution
$\tau_\varphi$, and we finally show that $\tau_\varphi$ is a unique candidate
answer to $q_\varphi$ over $\dat_{\K_\varphi}$.

\begin{restatable}{theorem}{lbound}\label{th:complexity-lbound}
    Checking whether a candidate answer is sound is \np-hard.
\end{restatable}

\subsection{Preliminary Evaluation}

\begin{table*}[t!]
    \scriptsize
    \centering
    \caption{Evaluation results}\label{tab:evaluation}
    \renewcommand*{\arraystretch}{0.8}
    \setlength{\tabcolsep}{1.5pt}
    \begin{tabular}{|c|r|c|c|c|cl|}
    \hline
        &           &  Inds.\   & Unary     & Binary    & Total     & Ratio \\
        &           &           & atoms     & atoms     & atoms     & \\
    \hline
    U5  & before    &  100,848  & 169,079   & 296,941   & 466,020   & \\
        & after     &  100,873  & 511,115   & 1,343,848 & 1,854,963 & 3.98 \\
    \hline
    U10 & before    &  202,387  & 339,746   & 598,695   & 938,441   & \\
        & after     &  202,412  & 1,026,001 & 2,714,214 & 3,740,215 & 3.98 \\
    \hline
    U20 & before    &  426,144  & 714,692   & 1,259,936 & 1,974,628 & \\
        & after     &  426,169  & 2,157,172 & 5,720,670 & 7,877,842 & 3.99 \\
    \hline
    \end{tabular}
    \qquad
    \renewcommand*{\arraystretch}{0.8}
    \begin{tabular}{|r|cccc|cccc|cccc|cccc|cccc|}
        \hline
	        & \multicolumn{4}{c|}{$q_3^l$}          & \multicolumn{4}{c|}{$q_1^t$}      & \multicolumn{4}{c|}{$q_2^t$}  & \multicolumn{4}{c|}{$q_3^t$}  & \multicolumn{4}{c|}{$q_4^t$} \\
	        & C     & U		& F        		& N     & C     & U     & F     & N         & C     & U		& F     & N		& C     & U     & F		& N     & C		& U     & F		& N \\
		\hline
	    U5  & 10    & 0     & 0.06      	& 0		& 73K	& 12    & 1.71	& 7.55      & 3K    & 0     & 0.01  & 0     & 157K	& 66    & 1.07  & 8.6	& 30K   & 63    & 2.44  & 10.9 \\
		\hline
	    U10 & 22    & 0     & 0.06      	& 0     & 149K  & 12    & 1.68  & 7.54      & 6K    & 0     & 0.01  & 0     & 603K  & 81    & 1.20  & 9.6   & 61K   & 63    & 2.44  & 10.9 \\
		\hline
	    U20 & 43  	& 0     & 0.07      	& 0     & 313K  & 12	& 1.66  & 7.55		& 12K	& 0     & 0.01	& 0     & 2.6M  & 90	& 1.28  & 10.3  & 129K  & 63	& 2.44  & 10.9 \\
		\hline
        \multicolumn{21}{c}{(C)  total number of candidate answers (U) percentage of unsound answers  (F) average filtering time in ms}\\[-0.2em]
        \multicolumn{21}{c}{ (N) average number of nondeterministic choices required for each candidate answer}\\
    \end{tabular}
\end{table*}

We implemented our algorithm in a prototypical system, and we conducted a
preliminary evaluation with the goal of showing that the number of consequences
of $\dat_\K$ is reasonably small, and that the nondeterminism of the filtering
procedure is manageable. Our prototype uses the
RDFox~\cite{DBLP:conf/aaai/MotikNPHO14} system to materialise the consequences
of $\dat_\K$. We ran our tests on a MacBook Pro with 4GB of RAM and a 2.4Ghz
Intel Core 2 Duo processor.

We tested our system using the version of the LSTW
benchmark~\cite{DBLP:conf/semweb/LutzSTW13} by
\citeA{DBLP:conf/aaai/StefanoniMH13}. The TBox of the latter is in \elho, and
we extended it to \elhso by making the role \emph{subOganizationOf} transitive
and by adding an axiom of type 5 and an axiom of type 7. We used the data
generator provided by LSTW to generate KBs U5, U10, and U20 of 5, 10, and 20
universities, respectively. Finally, only query $q_3^l$ from the LSTW benchmark
uses transitive roles, so we have manually created four additional queries. Our
system, the test data, and the queries are all available
online.\footnote{\url{http://www.cs.ox.ac.uk/isg/tools/EOLO/}} We evaluated the
practicality of our approach using the following two experiments.

First, we compared the size of the materialised consequences of $\dat_\K$ with
that of the input data. As the left-hand side of Table~\ref{tab:evaluation}
shows, the ratio between the two is four, which, we believe, is acceptable in
most practical scenarios.

Second, we measured the `practical hardness' of our filtering step on our test
queries. As the right-hand side of Table~\ref{tab:evaluation} shows, soundness
of a candidate answer can typically be tested in as few as several
milliseconds, and the test involves a manageable number of nondeterministic
choices. Queries $q_3^t$ and $q_4^t$ were designed to obtain a lot of candidate
answers with auxiliary individuals, so they retrieve many unsound answers.
However, apart from query $q_3^t$, the percentage of the candidate answers that
turned out to be unsound does not change with the increase in the size of the
ABox. Therefore, while some queries may be challenging, we believe that our
algorithm can be practicable in many cases.

\section{Acyclic and Arborescent Queries}\label{sec:acyclic}

In this section, we prove that answering a simple class of tree-shaped acyclic
CQs---which we call \emph{arborescent}---over $\mathcal{ELHO}$ KBs is
tractable, whereas answering acyclic queries is \np-hard. In addition, we show
that extending ${\mathcal{EL}}$ with transitive or reflexive roles makes
answering arborescent queries \np-hard. This is in contrast with the recent
result by \citeA{DBLP:conf/ijcai/BienvenuOSX13}, who show that answering
acyclic CQs over ${\mathcal{ELH}}$ KBs is tractable. We start by introducing
acyclic and arborescent queries.

\begin{definition}
    For $q$ a Boolean CQ, ${\dgraph{q}= \tuple{\vars{q}, E}}$ is a directed
    graph where $\tuple{x,y} \in E$ for each $R(x,y)\in q$. Query $q$ is
    \emph{acyclic} if the graph obtained from $\dgraph{q}$ by removing the
    orientation of edges is acyclic; $q$ is \emph{arborescent} if $q$ contains
    no individuals and ${\dgraph{q}}$ is a rooted tree with all edges pointing
    towards the root.
\end{definition}

Definition~\ref{def:entails} and Theorem~\ref{th:arborescent} show how to
answer arborescent CQs over \elho KBs in polynomial time. Intuitively, we apply
the fork rule (cf.\ Definition~\ref{def:fork}) bottom-up, starting with the
leaves of $q$ and spread constraints upwards.

\begin{definition}\label{def:entails}
    Let $\K$ be an \elho KB, let $\dat_\K$ be the datalog program for $\K$, let
    $\setind{\dat_\K}$ and $\aux{\dat_\K}$ be as specified in Definition
    \ref{def:dat-order}, and let $q$ be an arborescent query rooted in ${r \in
    \vars{q}}$. For each ${y \in \vars{q}}$ with $y\neq r$, and each ${V
    \subseteq \vars{q}}$, sets \roles{y} and $\pred{V}$ are defined as follows.
    \begin{displaymath}
	\begin{array}{@{}c@{}}
        \roles{y} = \setof{R\in\rolenames \mid  R(y,x)\in q \text{ with $x$ the parent of $y$ in } \dgraph{q}}\\[1ex]
        \pred{V} =  \setof{y\in\vars{q} \mid \exists x\in V \text{ with $x$ the parent of $y$ in } \dgraph{q}}
	\end{array}
    \end{displaymath}
    Set ${\mathsf{RT}}$ is the smallest set satisfying the following conditions.
    \begin{itemize}
        \item $\setof{r} \in \mathsf{RT}$ and the level of $\setof{r}$ is $0$.

        \item For each set $V \in \mathsf{RT}$ with level $n$, we have
        ${\pred{V} \in \mathsf{RT}}$ and the level of $\pred{V}$ is $n+1$.

        \item For each set $V \in \mathsf{RT}$ with level $n$ and each
        $y\in\pred{V}$, we have ${\setof{y} \in \mathsf{RT}}$ and the level of
        $\setof{y}$ is $n+1$.
    \end{itemize}
    For each $V \in \mathsf{RT}$, set $\mathsf{c}_V$ contains each ${u \in
    \aux{\dat_\K}\cup\setind{\dat_\K}}$ such that ${\dat_\K \models B(u)}$ for
    each unary atom ${B(x)\in q}$ with ${x\in V}$. By reverse-induction on the
    level of the sets in $\mathsf{RT}$, each $V \in \mathsf{RT}$ is associated
    with a set ${\mathsf{A}_V \subseteq \setind{\dat_{\K}}\cup\aux{\dat_{\K}}}$.
    \begin{itemize}
        \item For each set $V \in \mathsf{RT}$ of maximal level, let
        ${\mathsf{A}_V = \mathsf{c}_V}$.

        \item For $V\in \mathsf{RT}$ a set of level ${n}$ where $\mathsf{A}_V$
        is undefined but $\mathsf{A}_{W}$ has been defined for each ${W\in
        \mathsf{RT}}$ of level $n+1$, let ${\mathsf{A}_V= \mathsf{c}_V \cap
        (\mathsf{i}_V \cup \mathsf{a}_V)}$, where $\mathsf{i}_V$ and
        $\mathsf{a}_V$ are as follows.
    \end{itemize}
	\begin{small}
    \begin{displaymath}
	\begin{array}{@{}c@{}}
        \mathsf{i}_V = \{u\in\setind{\dat_{\K}} \mid \forall y \in \pred{V}\exists u'\in \mathsf{A}_{\setof{y}}. \dat_{\K} \models \bigwedge\limits_{R\in\roles{y}} R(u', u)\}\\[2ex]
        \mathsf{a}_V = \{u\in\aux{\dat_{\K}} \mid \exists u'\in \mathsf{A}_{\pred{V}}\forall y \in \pred{V}. \dat_{\K} \models \bigwedge\limits_{R\in\roles{y}}  \directedge{R}(u', u)\}
	\end{array}
    \end{displaymath}
	\end{small}%
    Function $\mathsf{entails}(\dat_\K, q)$ returns \true if and only if
    ${\mathsf{A}_{\setof{r}}}$ is nonempty.
\end{definition}

\begin{theorem}\label{th:arborescent}
    For $\K$ a satisfiable \elho KB and $q$ an arborescent query, function
    $\mathsf{entails}(\dat_\K, q)$ returns \true if and only if ${\Xi_\K
    \models q}$. Furthermore, function $\mathsf{entails}(\dat_\K, q)$ runs in
    time polynomial in the input size.
\end{theorem}

Finally, we show that (unless $\ptime = \np$), answering arbitrary acyclic
queries over \elho KBs is harder than answering arborescent queries, and we
show that adding transitive or reflexive roles to the DL $\el$ makes answering
arborescent queries intractable.

\begin{theorem}\label{th:lowerbound-acyclic}
    For $\K = \tuple{\T, \A}$ a KB and $q$ a Boolean CQ, checking $\K \models
    q$ is \np-hard in each of the following cases.
    \begin{enumerate}
        \item The query $q$ is acyclic and the TBox $\T$ is in $\elho$.

        \item The query $q$ is arborescent and the TBox $\T$ consists only of
        axioms of type $1$ and $7$, and of one axiom of type $8$.

        \item The query $q$ is arborescent and the TBox $\T$ consists only of
        axioms of  type $1$ and $7$, and of one axiom of type $9$.
    \end{enumerate}
\end{theorem}

\section{Outlook}

In future, we shall adapt our filtering procedure to detect unsound answers
already during query evaluation. Moreover, we shall extend
Algorithm~\ref{algo:issound} to handle complex role inclusions, thus obtaining
a practicable approach for OWL 2 EL.

\section*{Acknowledgements}
This work was supported by Alcatel-Lucent; the EU FP7 project OPTIQUE; and the
EPSRC projects MASI$^3$, Score!, and DBOnto.

\bibliographystyle{aaai}
\bibliography{references}

\ifdraft{
\clearpage
\appendix
\onecolumn
\setcounter{secnumdepth}{2}

\section{Skolem Chase and Universal Interpretations}\label{app:chase}

In this section, we present a special variant of \emph{Skolem
chase~}\cite{DBLP:conf/pods/Marnette09}. Our chase uses
\emph{merging}~\cite{DBLP:books/aw/AbiteboulHV95} and
\emph{pruning}~\cite{msh09hypertableau} to deal with equality rules and
guarantee that on rule base $\Xi_\K$ the so-called universal interpretation it
produces satisfies the structural properties described in Section
\ref{sec:np_algorithm}. We next formally present our chase variant.

Let $\Sigma$ be a rule base and assume, w.l.o.g., that each variable occurring
in $\Sigma$ is quantified over exactly once in $\Sigma$; furthermore, let $>$
be an arbitrary total order on ground terms. Next, we first define some
auxiliary notions, after which we define our Skolem chase variant and universal
interpretations of $\Sigma$.

The \emph{skolemisation of an existential rule} ${\varphi(\vec{x}, \vec{y})
\rightarrow \exists \vec{z}. \psi(\vec{x},\vec{z})}$ is the formula
${\psi_{\mathsf{sk}}}$ obtained from $\psi$ by substituting ${f_z(\vec{x})}$
for each ${z \in \vec{z}}$ with $f_z$ a fresh function symbol of arity
$|\vec{x}|$. A \emph{chase instance} is a set ${I = I_{\Sigma}\cup I_\eq}$
where $I_{\Sigma}$ is a set of ground atoms over the predicates in $\Sigma$,
and $I_\eq$ is a set of assertions of the form $w \rightsquigarrow w'$ with
$\rightsquigarrow$ a fresh predicate. For each term $w$, let ${\norm{w}{I} =
w}$, if no term $w'$ exists such that ${w \rightsquigarrow w' \in I}$;
otherwise, let ${\norm{w}{I} = w'}$ where $w'$ is the smallest term in the
ordering $>$ for which terms ${w_0, \ldots, w_n}$ with ${w_0 = w}$ and ${w_n =
w'}$ exist such that ${w_{i-1} \rightsquigarrow w_{i}\in I}$ for each ${i \in
\interval{1}{n}}$. For a formula $\alpha$, $\norm{\alpha}{I}$ is the result of
uniformly substituting each term $w$ occurring in $\alpha$ with $\norm{w}{I}$.

An \emph{existential rule ${\varphi(\vec{x}, \vec{y}) \rightarrow \exists
\vec{z}. \psi(\vec{x}, \vec{z})}$ is applicable} to a chase instance ${I =
I_\Sigma \cup I_\eq}$, if a substitution $\sigma$ exists such that
${\sigma(\varphi) \subseteq I}$ and ${\norm{\sigma(\psi_{\mathsf{sk}})}{I} \not
\subseteq I}$; the \emph{result} of applying such a rule to ${I}$ is obtained
by adding ${\norm{\sigma(\psi_{\mathsf{sk}})}{I}}$ to $I_{\Sigma}$. An
\emph{equality rule ${\varphi(\vec{x}) \rightarrow s\approx t}$ is applicable}
to a chase instance ${I = I_{\Sigma}\cup I_\eq}$, if a substitution $\sigma$
exists such that ${\sigma(\varphi) \subseteq I}$ and ${\norm{\sigma(s)}{I} \neq
\norm{\sigma(t)}{I}}$; for ${\setof{w, w'} = \setof{\norm{\sigma(s)}{I},
\norm{\sigma(t)}{I}}}$ such that ${w > w'}$, the \emph{result} of applying such
a rule to ${I}$ is the chase instance ${I_{\Sigma}' \cup I_\eq'}$ where
${I_\eq'}$ is the result of adding ${w \rightsquigarrow w'}$ to ${I_{\eq}}$,
and $I_{\Sigma}'$ is obtained by removing all atoms occurring in $I_{\Sigma}$
that contain a term $w_2$ with $w$ a proper subterm of $w_2$, and by replacing
each occurrence of $w$ in the resulting instance with $w'$.

A \emph{chase for $\Sigma$ w.r.t.\ $>$} is a sequence of chase instances ${I_0,
I_1,\ldots}$ where $I_0$ contains each ground atom in $\Sigma$, and, for each
${i \geq 1}$, chase instance ${I_{i+1}}$ is the result of an (arbitrarily
chosen) rule in $\Sigma$ applicable to ${I_i}$, and ${I_{i+1} = I_i}$, if no
rule is applicable to ${I_i}$. This sequence must be \emph{fair}---that is, if
a rule ${\varphi \rightarrow \psi}$ in ${\Sigma}$ is applicable to some ${I_i}$
under a specific substitution $\sigma$ and $\sigma(\varphi) \subseteq I_j$, for
each $j\geq i$, then $k \geq i$ exists such that ${I_{k+1}}$ is the result of
${\varphi \rightarrow \psi}$ on ${I_k}$ w.r.t.\ substitution $\sigma$. Set ${I
= \bigcup_{i \in \nat} \bigcap_{j\geq i} I_j}$ is a \emph{universal
interpretation of $\Sigma$ (w.r.t.\ $>$)} and its \emph{domain} is the set
$\domain{I}$ containing $\norm{w}{I}$ for each term $w$ that occurs in $I$.

Please note that for each term $w$ that occurs in $I$ in at least one atom over
a predicate from $\Sigma$, we have $\norm{w}{I} =w$ and $w\in\domain{I}$. Also,
owing to the fairness of the chase sequence, no rule in $\Sigma$ is applicable
to $I$. Furthermore, if $\Sigma$ does not contain equality rules, then $I$ is
independent from the order in which rules are applied, and so it is \emph{the
universal interpretation} of $\Sigma$. Otherwise, if $\Sigma$ contains equality
rules, then $I$ is homomorphically equivalent w.r.t.\ the predicates occurring
in $\Sigma$ to any other universal interpretation $I'$ of $\Sigma$, although
$I$ and $I'$ may disagree on the assertions over $\rightsquigarrow$. Finally,
it is well known \cite{DBLP:conf/pods/Marnette09} that $I$ can be
homomorphically embedded into any model of $\Sigma$; so $I$ can be used to
answer arbitrary CQs over $\Sigma$.
\begin{fact}
    For each CQ $q$ and each
    substitution $\pi$, ${\Sigma \models \pi(q)}$ if and only if a substitution
    $\pi_*$ with ${\dom{\pi_*} = \vars{q}}$ exists such that
    ${\pi \subseteq \pi_*}$ and ${\norm{\pi_*(q)}{I} \subseteq I}$.
\end{fact}

\subsection{Universal Interpretations of $\Xi_\K$ and $\dat_\K$}

Let $\K$ be an \elhso knowledge base, and let $\Xi_{\K}$ and $\dat_{\K}$ be the
rule base and the datalog program associated with $\K$, respectively. In the
rest of this appendix, we shall make the two following assumptions. First, we
associate to each rule ${A_1(x) \rightarrow \exists z. R(x,z)\wedge A(z)}$ in
$\Xi_{\K}$ a fresh unary function symbol ${f_{R,A}^{A_1}}$, and we assume that
the skolemisation of such a rule is given by ${R(x, f_{R,A}^{A_1}(x))\wedge
A(f_{R,A}^{A_1}(x))}$. Second, we assume that the chase for $\Xi_{\K}$ and for
$\dat_{\K}$, respectively, is w.r.t.\ the total order $>$ specified in
Definition \ref{def:dat-order}, and that ${f_{R, A}^{A_1}(w) > a}$ for each
ground term $w$ and each individual $a$.

\section{Proof of Proposition \ref{prop:sat}}\label{sec:proofsat}

We fix an \elhso knowledge base $\K$. Let $\Xi_{\K}$ and $\dat_{\K}$ be the rule
base and the datalog program associated with $\K$, respectively; moreover, let
$I$ and $J$ be universal interpretations of $\Xi_{\K}$ and $\dat_{\K}$,
respectively.

We next define a function $\delta$ that maps each term $w$ occurring in $I$ to
a term $\delta(w)$ as follows:
\[
\delta(w) = \begin{cases}
    w           & \text{ if } w \in \indnames,\\
    o_{R, A}    & \text{ if $w$ is of the form $ w=f_{R,A}^{A_1}(w')$.}
\end{cases}
\]
The following two Lemmas show that $\delta$ establishes a tight connection
between $I$ and $J$. The proofs of these results is given in  Appendix \ref{sec:connection}.

\begin{restatable}{lemma}{homomorphism}\label{lemma:homomorphism}
    Mapping $\delta$ satisfies the following properties for all terms $w_1$ and
    $w_2$ occurring in $I$, each individual $a \in\indnames$, each role ${R \in
    \rolenames}$, and each concept ${C \in \conceptnames \cup
    \setof{\top, \bot}}$.
    \begin{enumerate}[{H}1.]
        \item \label{hom:cond1} ${C(w_1) \in \lpmodel}$ implies that
        ${C(\delta(w_1)) \in \datalogmodel}$.

        \item \label{hom:cond2} ${R(w_1, w_2) \in \lpmodel}$ implies that
        ${R(\delta(w_1), \delta(w_2)) \in \datalogmodel}$.

        \item \label{hom:cond3} $R(w_1, w_2) \in I$ and $w$ is of the form $f_{P,A}^{A_1}(w_1)$
        imply that ${\directedge{R}(\delta(w_1), \delta(w_2)) \in J}$.

        \item \label{hom:cond4} ${\norm{w_1}{I} = a}$ implies that ${\norm{\delta(w_1)}{J} =a}$.
    \end{enumerate}
\end{restatable}

\begin{restatable}{lemma}{datembed} \label{lemma:dat-embed}
    Mapping $\delta$ satisfies the following properties for all terms $w_{1}$
    and $w_2$
    occurring in $I$, each individual $a\in\indnames$, each role ${R \in
    \rolenames}$, and each concept ${C \in \conceptnames \cup
    \setof{\top, \bot}}$.
    \begin{enumerate}[D1.]
        \item \label{dat-cond1} ${C(\delta(w_1)) \in \datalogmodel}$
        implies that ${C(w_1) \in \lpmodel}$.

        \item \label{dat-cond2} ${\SELF_{R}(\delta(w_1)) \in \datalogmodel}$
        implies that ${R(w_1,w_1) \in \lpmodel}$.

        \item \label{dat-cond3} ${R(\delta(w_1),\delta(w_2)) \in
        \datalogmodel}$ and ${\delta(w_2) \in \indnames}$ imply that
        ${R(w_1,w_2) \in \lpmodel}$.

        \item \label{dat-cond4} ${R(\delta(w_1),\delta(w_2)) \in
        \datalogmodel}$ and $\delta(w_2)$ is of the form $o_{P,A}$
        imply that
        \begin{itemize}
            \item a term $w_1' \in \domain{I}$ exists such that
            ${R(w_1', w_2) \in \lpmodel}$, and

            \item a term $w_2' \in \domain{I}$ exists such that $\delta(w_2') =
            o_{P,A}$ and ${R(w_1, w_2') \in \lpmodel}$.
        \end{itemize}

        \item \label{dat-cond5} ${\directedge{R}(\delta(w_1),\delta(w_2)) \in
        \datalogmodel}$ and $\delta(w_2)$ is of the form $o_{P,A}$ imply that
        \begin{itemize}
            \item a term ${w_3' \in \domain{I}}$ of the form
            $f_{P,A}^{A_1}(w_1)$ exists such that ${P(w_1, w_3') \in I}$, and

            \item $P\subrole R$.
        \end{itemize}

        \item \label{dat-cond6} ${\norm{\delta(w_1)}{J} = \delta(a)}$ implies that
        ${\norm{w_1}{I} = a}$.

        \item \label{dat-cond7} For each individual $u \in \domain{J}$, term ${w \in
        \domain{\lpmodel}}$ exists such that ${\delta(w) = u}$.
    \end{enumerate}
\end{restatable}

We are now ready to show that $\dat_{\K}$ can be used to check the
satisfiability of $\K$.

\propsat*
\begin{proof}
    By the definition of \elhso semantics,  \K is unsatisfiable if and only if $\Xi_{\K}
    \models \exists x. \bot(x)$. By Lemmas \ref{lemma:homomorphism} and
    \ref{lemma:dat-embed}, for each term $w \in \domain{I}$, we have ${\bot(w)
    \in I}$ if and only if ${\bot(\delta(w)) \in J}$. Consequently, $\K$ is
    unsatisfiable if and only if ${\dat_{\K} \models \exists x. \bot(x)}$.
\end{proof}

\section{Proofs of Lemmas \ref{lemma:homomorphism} and \ref{lemma:dat-embed}}\label{sec:connection}

We prove Lemmas \ref{lemma:homomorphism} and \ref{lemma:dat-embed} in various
stages. To start, we next show that mapping $\delta$ satisfies a slightly
relaxed version of properties {H1}--{H4} from Lemma \ref{lemma:homomorphism}.

\begin{lemma}\label{lemma:aux:homomorphism}
    Mapping $\delta$ satisfies the following four properties for all terms
    $w_1$ and $w_2$ occurring in $I$, each individual $a\in\indnames$, each
    role ${R \in \rolenames}$, and each concept ${C \in \conceptnames \cup
    \setof{\top, \bot}}$.
    \begin{enumerate}[(i)]
        \item \label{aux:hom:cond1} ${C(w_1) \in \lpmodel}$ implies that
        ${\norm{C(\delta(w_1))}{J} \in \datalogmodel}$.

        \item \label{aux:hom:cond2} ${R(w_1, w_2) \in \lpmodel}$ implies that
        ${\norm{R(\delta(w_1), \delta(w_2))}{J} \in \datalogmodel}$.

        \item \label{aux:hom:cond3} $R(w_1, w_2)\in \lpmodel$ and $w$ is of
        the form $f_{P,A}^{A_1}(w_1)$ imply that
        ${\norm{\directedge{R}(\delta(w_1), \delta(w_2))}{J} \in J}$.

        \item \label{aux:hom:cond4} ${w_1 \rightsquigarrow a \in \lpmodel}$ implies
        that ${\norm{\delta(w_1)}{J} = \norm{a}{J}}$.
    \end{enumerate}
\end{lemma}
\begin{proof}

    Let ${I_0, I_1, \ldots}$ be the chase for $\Xi_{\K}$ w.r.t.\ $>$ used to
    construct $I$. We show by induction on this sequence that each $I_n$
    satisfies the properties.

    \smallskip

    \basecase Consider chase instance $I_0$. By the definition, $I_0$ contains
    only ground atoms constructed using the predicates in ${\conceptnames \cup
    \setof{\top, \bot, \ind} \cup \rolenames}$ and the individuals in
    \indnames. Therefore, for each term $w$ occurring in $I_0$, we have
    $\norm{w}{I_0} = w$ and $w \in\indnames$, and so $\delta(w) = w$. Consider
    an arbitrary atom ${\phi \in I_0}$. By the definition of $\dat_{\K}$,
    $\phi$ is an atom in $\dat_{\K}$. Since $J$ satisfies all ground atoms in
    $\dat_{\K}$, we have ${\norm{\phi}{J} \in J}$, so properties
    (\ref{aux:hom:cond1})--(\ref{aux:hom:cond4}) hold.

    \medskip

    \indstep Consider an arbitrary ${n \in \nat}$ and assume that $I_n$
    satisfies properties (\ref{aux:hom:cond1})--(\ref{aux:hom:cond4}). By
    considering each rule in $\Xi_{\K}$, we assume that the rule is applicable
    to $I_n$, and we show that the properties hold for all fresh atoms in the
    resulting instance.

    \smallskip

    (Datalog Rule) Consider a datalog rule ${\varphi(\vec{x},\vec{y})
    \rightarrow \psi(\vec{x})}$ in $\Xi_{\K}$, and assume that a substitution
    $\sigma$ with ${\dom{\sigma} =\vec{x} \cup \vec{y}}$ exists such that
    ${\sigma(\varphi) \subseteq I_n}$. Let $\sigma'$ be the substitution such
    that ${\sigma'(x) = \delta(\sigma(x))}$ for each variable ${x \in \vec{x}
    \cup \vec{y}}$. By the inductive hypothesis, we have
    ${\norm{\sigma'(\psi)}{J} \subseteq J}$. Since ${\varphi(\vec{x}, \vec{y})
    \rightarrow \psi(\vec{x}) \in \dat_\K}$ and no rule is applicable to $J$,
    we have ${\norm{\sigma'(\psi)}{J} \subseteq J}$.

    \smallskip

    (Existential Rule) Consider ${A_1(x) \rightarrow \exists z. R(x, z)\wedge
    A(z)}$ in $\Xi_{\K}$, assume that ${A_1(w_1) \in I_n}$, and let ${w_2 =
    \norm{f_{R,A}^{A_1}(w_1)}{I_n}}$. By the inductive hypothesis,
    ${\norm{A_1(\delta(w_1))}{J} \in J}$ holds. Program $\dat_{\K}$ contains
    ${A_1(x) \rightarrow R(x, o_{R,A}) \wedge \directedge{R}(x, o_{R,A})\wedge
    A(o_{R,A})}$ and no rule is applicable to $J$, so ${\norm{R(\delta(w_1),
    o_{R,A})\wedge \directedge{R}(\delta(w_1), o_{R,A}) \wedge A(o_{R,A})}{J}
    \subseteq J}$ holds. We distinguish two cases.
    \begin{itemize}
        \item $w_2 = f_{R,A}^{A_1}(w_1)$. Then, $\delta(w_2) = o_{R,A}$, and so
        ${\norm{o_{R,A}}{J} = \norm{\delta(w_2)}{J}}$.

        \item $w_2 \not= f_{R,A}^{A_1}(w_1)$. By the form of equality rules in
        $\Xi_{\K}$ and due to ${w_2 = \norm{f_{R,A}^{A_1}(w_1)}{I_n}}$, we have
        $w_2 \in \indnames$. Then, terms ${u_0,\ldots, u_n}$ with ${u_0 =
        f_{R,A}^{A_1}(w_1)}$ and ${u_n = w_2}$ exist in $I_n$ such that
        $u_{i-1}\rightsquigarrow u_i\in I_n$ for each $i\in\interval{1}{n}$. By
        the inductive hypothesis, we have
        ${\norm{\delta(f_{R,A}^{A_1}(w_1))}{J} = \norm{\delta(w_2)}{J}}$,
         and so ${\norm{o_{R,A}}{J} = \norm{\delta(w_2)}{J}}$.
    \end{itemize}
    As stated above, we have ${\norm{R(\delta(w_1), o_{R,A})\wedge
    \directedge{R}(\delta(w_1), o_{R,A}) \wedge A(o_{R,A})}{J} \subseteq J}$,
    so properties (\ref{aux:hom:cond1})--(\ref{aux:hom:cond3}) are satisfied.

    \smallskip

    (Equality Rule) Consider a rule ${A(x) \rightarrow x \approx a}$ in
    $\Xi_{\K}$ and assume that $A(w_1) \in I_n$. As ${w_1 \in \domain{I_n}}$,
    we have $\norm{w_1}{I_n} = w_1$. Then let terms $w$ and $w'$ be such that
    ${\setof{w, w'} = \setof{w_1, \norm{a}{I_n}}}$ and $w > w'$. By the
    inductive hypothesis, we either have $\norm{A(\delta(w))}{J} \in J$ and
    $\norm{\delta(w')}{J} = \norm{a}{J}$, or $\norm{A(\delta(w'))}{J} \in J$
    and $\norm{\delta(w)}{J} = \norm{a}{J}$. In either cases, since ${A(x)
    \rightarrow x \approx a \in \dat_\K}$ and no rule is applicable to $J$, we
    have ${\norm{\delta(w)}{J} = \norm{\delta(w')}{J}}$, and property
    (\ref{aux:hom:cond4}) is satisfied. We next consider the atoms in $I_n$
    that get replaced by the application of this rule.
    \begin{itemize}
        \item $C(w) \in I_n$. By the inductive hypothesis, we have
        $\norm{C(\delta(w))}{J} \in J$; thus ${\norm{C(\delta(w'))}{J}
        \in J}$.

        \item $R(w, w) \in I_n$. By the inductive hypothesis, we
        have $\norm{R(\delta(w), \delta(w))}{J} \in J$; thus
        $\norm{R(\delta(w'), \delta(w'))}{J} \in J$.

        \item $R(w, w_2) \in I_n$. By the inductive hypothesis, we have
        $\norm{R(\delta(w), \delta(w_2))}{J} \in J$; thus
        $\norm{R(\delta(w'), \delta(w_2))}{J} \in J$.

        \item $R(w_2, w) \in I_n$. By the inductive hypothesis, we have
        $\norm{R(\delta(w_2), \delta(w))}{J} \in J$; thus,
        $\norm{R(\delta(w_2), \delta(w'))}{J} \in J$.        \qedhere

    \end{itemize}
\end{proof}

We next show that $J$ satisfies a slightly relaxed version of properties
{D1}--{D7} from Lemma \ref{lemma:dat-embed}.

\begin{lemma}\label{lemma:aux:dat-embed}
    Mapping $\delta$ satisfies the following properties for all terms ${w_1}$
    and ${w_2}$ occurring in $I$, each individual $a\in\indnames$, each role
    ${R \in \rolenames}$, and each concept ${C \in \conceptnames \cup
    \setof{\top, \bot}}$.
    \begin{enumerate}[(a)]
        \item \label{aux:dat-cond1} ${C(\delta(w_1)) \in \datalogmodel}$
        implies that ${\norm{C(w_1)}{I} \in \lpmodel}$.

        \item \label{aux:dat-cond3} ${\SELF_{R}(\delta(w_1)) \in \datalogmodel}$
        implies that ${\norm{R(w_1,w_1)}{I} \in \lpmodel}$.

        \item \label{aux:dat-cond2} ${R(\delta(w_1),\delta(w_2)) \in
        \datalogmodel}$ and ${\delta(w_2) \in \indnames}$ imply that
        ${\norm{R(w_1,w_2)}{I} \in \lpmodel}$.

        \item \label{aux:dat-cond4} ${R(\delta(w_1),\delta(w_2)) \in
        \datalogmodel}$ and $\delta(w_2)$ is of the form $o_{P,A}$
        imply that
        \begin{itemize}
            \item a term $w_1'$ from $I$  exists such that
            ${\norm{R(w_1', w_2)}{I} \in \lpmodel}$, and

            \item a term $w_2'$ from $I$ exists such that $\delta(w_2') =
            o_{P,A}$ and ${\norm{R(w_1, w_2')}{I} \in \lpmodel}$.
        \end{itemize}

        \item \label{aux:dat-cond5} ${\directedge{R}(\delta(w_1),\delta(w_2)) \in
        \datalogmodel}$ and $\delta(w_2)$ is of the form $o_{P,A}$ imply that
        \begin{itemize}
            \item a term ${w_3'}$ from $I$ of the form
            $f_{P,A}^{A_1}(\norm{w_1}{I})$ exists such that ${\norm{P(w_1,
            w_3')}{I} \in I}$, and

            \item $P\subrole R$.
        \end{itemize}
        \item \label{aux:dat-cond6} ${\delta(w_1) \rightsquigarrow  a\in J}$ implies that
        ${\norm{w_1}{I} = \norm{a}{I}}$.

        \item \label{aux:dat-cond7} For each individual $u$ occurring in $J$, a term
        ${w}$ occurring in $I$ exists such that ${\delta(w) = u}$.
    \end{enumerate}
\end{lemma}
\begin{proof}
    Let ${J_0, J_1, \ldots}$ be the chase for $\dat_{\K}$ used to
    construct $J$. We prove by induction on this sequence that each $J_n$ satisfies the properties.

    \smallskip

    \basecase Consider $J_0$. By the definition, $J_0$ does not contain
    assertions over $\rightsquigarrow$. Moreover, $\dat_{\K}$ and $\Xi_{\K}$
    contain the same ground atoms, all of which are constructed using the
    individuals from \indnames and the predicates in ${\conceptnames \cup
    \setof{\top, \bot, \ind} \cup \rolenames}$. Finally, $I$ satisfies all the
    ground atoms in $\Xi_{\K}$, so properties (\ref{aux:dat-cond1})--(\ref{aux:dat-cond7})
    are satisfied.

    \medskip

    \indstep Consider an arbitrary ${n \in \nat}$ and assume that $J_n$
    satisfies properties (\ref{aux:dat-cond1})--(\ref{aux:dat-cond7}). By
    considering each rule in $\dat_{\K}$, we assume that the rule is applicable
    to $J_n$, and we show that the properties hold for all fresh atoms in the
    resulting instance.

    \smallskip

    Consider a rule of the form $\SELF_{P}(x) \rightarrow \SELF_{R}(x)$ in
    $\dat_{\K}$ and assume that ${\SELF_{P}(\delta(w_1)) \in J_n}$. By the
    inductive hypothesis, we have ${\norm{\SELF_{P}(w_1)\wedge P(w_1, w_1)}{I}
    \subseteq I}$. By the definition of program $\dat_{\K}$ and $\Xi_{\K}$, we
    have ${P \ISA R \in \T}$, and therefore rules ${\SELF_{P}(x) \rightarrow
    \SELF_{R}(x)}$ and ${P(x,y) \rightarrow R(x,y)}$ are contained in
    $\Xi_{\K}$. Since no rule is applicable to $I$, we have
    ${\norm{\SELF_{R}(w_1) \wedge R(w_1, w_1)}{I} \subseteq I}$.

    \smallskip

    Consider a rule of the form $\varphi(x) \rightarrow B(x)$ in $\dat_{\K}$
    with $\varphi$ a conjunction of unary atoms over variable $x$ and $B$ a
    concept. Let $\sigma$ and $\sigma'$ be substitutions such that ${\sigma(x)
    = \delta(w_1)}$ and $\sigma'(x) = w_1$. Assume that ${\sigma(\varphi)
    \subseteq J_n}$. By the inductive hypothesis, we have
    ${\norm{\sigma'(\varphi)}{I} \subseteq I}$. Since no rule is applicable to
    $I$, we have ${\norm{B(w_1)}{I} \in I}$.

    \smallskip

    Consider a rule of the form $A(x) \rightarrow x \approx a$ and assume that
    ${A(\delta(w_1)) \in J_n}$. As $\delta(w_1)\in\domain{J_n}$, we have
    $\norm{\delta(w_1)}{J_n} = \delta(w_1)$. Let $w$ and $w'$ be terms such
    that ${\setof{\delta(w), \delta(w')} = \setof{\delta(w_1), \norm{a}{J_n}}}$
    and $\delta(w) > \delta(w')$. By the inductive hypothesis, we either have
    $\norm{A(w)}{I} \in I$ and $\norm{w'}{I} = \norm{a}{I}$, or
    $\norm{A(w')}{I} \in I$ and $\norm{w}{I} = \norm{a}{I}$. In either cases,
    since no rule is applicable to $I$, we have $\norm{w}{I} = \norm{w'}{I}$.
    We next consider the atoms in $J_n$ that get replaced by the application of
    this rule.
    \begin{itemize}
        \item ${\SELF_{R}(\delta(w)) }\in J_n$. By the inductive hypothesis, we
        have ${\norm{\SELF_{R}(w) \wedge R(w, w)}{I} \subseteq I}$; so
        ${\norm{\SELF_{R}(w') \wedge R(w', w')}{I} \subseteq I}$.

        \item $C(\delta(w)) \in J_n$. By the inductive hypothesis, we have
        ${\norm{C(w)}{I} \in I}$, and so ${\norm{C(w')}{I} \in I}$.

        \item $R(\delta(w), \delta(w)) \in J_n$.  We distinguish two cases.
        \begin{itemize}
            \item ${\delta(w) \in \indnames}$. By the inductive
            hypothesis, we have ${\norm{R(w, w)}{I} \in I}$, and so
            ${\norm{R(w', w')}{I} \in I}$.

            \item ${\delta(w)}$ is of the form $o_{P,A}$. By the inductive
            hypothesis, a term $w_2'$ exists such that $\delta(w_2') = o_{P,A}$
            and ${\norm{R(w, w_2')}{I} \in I}$. Since $\delta(w_2') =
            \delta(w)$, we have ${\norm{w_2'}{I} = \norm{w'}{I}}$,
            and so ${\norm{R(w', w')}{I} \in I}$.
        \end{itemize}

        \item $R(\delta(w_2), \delta(w)) \in J_n$.  We distinguish two cases.
        \begin{itemize}
            \item ${\delta(w) \in \indnames}$. By the inductive
            hypothesis, we have ${\norm{R(w_2, w)}{I} \in I}$, and so
            ${\norm{R(w_2, w')}{I} \in I}$.

            \item ${\delta(w)}$ is of the form $o_{P,A}$. By the inductive
            hypothesis, a term $w_2'$ exists such that $\delta(w_2') = o_{P,A}$
            and ${\norm{R(w_2, w_2')}{I} \in I}$. Since $\delta(w_2') =
            \delta(w)$, we have ${\norm{w_2'}{I} = \norm{w'}{I}}$, and so
            ${\norm{R(w_2, w')}{I} \in I}$.
        \end{itemize}

        \item $R(\delta(w), \delta(w_2)) \in J_n$. We distinguish two cases.
        \begin{itemize}
            \item ${\delta(w_2) \in \indnames}$. By the inductive hypothesis,
            ${\norm{R(w, w_2)}{I} \in I}$, thus ${\norm{R(w', w_2)}{I} \in I}$.

            \item ${\delta(w_2)}$ is of the form $o_{P,A}$. By the inductive
            hypothesis, a term ${w_2'}$ with ${\delta(w_2') = o_{P,A}}$ exists
            such that ${\norm{R(w, w_2')}{I} \in I}$, and so ${\norm{R(w',
            w_2')}{I} \in I}$.
        \end{itemize}

        \item $\directedge{R}(\delta(w), \delta(w_2)) \in J_n$ and
        $\delta(w_2)$ is of the form $o_{P,A}$. Property (\ref{aux:dat-cond5})
        is satisfied by the inductive hypothesis.

    \end{itemize}

    Consider a rule of the form $R(x, y) \wedge A_1(y) \rightarrow A(x)$ in
    $\dat_{\K}$ and assume that ${\setof{R(\delta(w_1),\delta(w_2)),
    A_1(\delta(w_2))} \subseteq J_n}$. We distinguish two cases.
    \begin{itemize}
        \item $\delta(w_2) \in \indnames$. By the inductive
        hypothesis, we have ${\norm{R(w_1, w_2)}{I}\subseteq I}$.

        \item $\delta(w_2) =o_{P,A}$. By the inductive
        hypothesis, a term $w_2'$ exists such that ${\delta(w_2') = o_{P,A}}$
        and ${\norm{R(w_1, w_2')\wedge A(w_2')}{I} \in I}$.
    \end{itemize}
    In either cases, since no rule is applicable to $I$, we have
    ${\norm{A(w_1)}{I} \in I}$.

    \smallskip

    Consider a rule $T(x, y) \rightarrow R(x, y)$ in $\dat_{\K}$ and assume
    that ${T(\delta(w_1), \delta(w_2)) \in J_n}$. We distinguish two
    cases.
    \begin{itemize}
        \item $\delta(w_2) \in \indnames$. By the inductive hypothesis,
        ${\norm{T(w_1, w_2)}{I}\in I}$. As no rule is applicable to $I$,
        we have ${\norm{R(w_1, w_2)}{I} \in I}$.

        \item $\delta(w_2)$ is of the form $o_{P,A}$. By the inductive
        hypothesis, we have that terms $w_1'$ and $w_2'$ exist such that
        ${\delta(w_2') = o_{P,A}}$ and ${\norm{T(w_1, w_2') \wedge T(w_1',
        w_2)}{I}\subseteq I}$. Since no rule is applicable to $I$, we have
        ${\norm{R(w_1, w_2')\wedge R(w_1', w_2)}{I} \subseteq I}$.
    \end{itemize}

    Consider a rule $\directedge{T}(x, y) \rightarrow \directedge{R}(x, y)$ in
    $\dat_{\K}$ and assume that ${\directedge{T}(\delta(w_1), \delta(w_2)) \in
    J_n}$ and $\delta(w_2)$ is of the form $o_{P,A}$. By the inductive
    hypothesis, a term ${w_3}$ of the form $f_{P,A}^{A_1}(\norm{w_1}{I})$
    exists such that ${\norm{P(w_1, w_3')}{I} \in I}$ and ${P \subrole T}$. By
    the definition of $\dat_{\K}$, we have $T\ISA R \in\T$. Consequently, we
    have $P\subrole R \in \T$ and property (\ref{aux:dat-cond5}) holds.

    \smallskip

    Consider a rule $R(x, y) \rightarrow A(y)$ in $\dat_{\K}$, and assume that
    ${R(\delta(w_1), \delta(w_2)) \in J_n}$.
    \begin{itemize}
        \item $\delta(w_2) \in \indnames$. By the inductive
        hypothesis, we have ${\norm{R(w_1, w_2)}{I}\subseteq I}$.

        \item $\delta(w_2)$ is of the form $o_{P,A}$. By the inductive
        hypothesis, a term $w_1'$ exists such that ${\norm{R(w_1', w_2)}{I} \in
        I}$.
    \end{itemize}
    In either cases, since no rule is applicable to $I$, we have
    ${\norm{A(w_2)}{I} \in I}$.

    \smallskip

    Consider a rule $A_1(x) \rightarrow R(x, o_{R,A}) \wedge \directedge{R}(x,
    o_{R,A})\wedge A(o_{R,A})$ in $\dat_{\K}$, and assume that ${A(\delta(w_1))
    \in J_n}$. Let $w_2$ be a term such that ${\delta(w_2) =
    \norm{o_{R,A}}{J_n}}$. By the inductive hypothesis, we have
    ${\norm{A(w_1)}{I}} \in I$. By the definition of $\dat_{\K}$, rule base
    $\Xi_{\K}$ contains ${A_1(x) \rightarrow \exists z. R(x, z) \wedge A(z)}$.
    Let $w_3' =f_{R,A}^{A_1}(\norm{w_1}{I})$. Since no rule is applicable to
    $I$, we have ${\norm{R(w_1,w_3')\wedge A(w_3')}{I} \subseteq I}$. We
    distinguish two cases.
    \begin{itemize}
        \item $\delta(w_2) \neq o_{R,A}$. By the form of equality rules
        occurring in $\dat_\K$, we have $\delta(w_2)\in\indnames$. Then
        individuals $u_0,\ldots, u_n$ with ${u_0 = o_{R,A}}$ and ${u_n =
        \delta(w_2)}$ exist such that ${u_{i-1}\rightsquigarrow u_i \in J_n}$
        for each ${i \in \interval{1}{n}}$. Moreover, given that ${\delta(w_3')
        = o_{R,A}}$, by the inductive hypothesis, we have ${\norm{w_3'}{I} =
        \norm{w_2}{I}}$. Hence, we have ${\norm{R(w_1, w_2)\wedge A(w_2)}{I}
        \subseteq I}$. By the inductive hypothesis, property
        (\ref{aux:dat-cond7}) is satisfied.

        \item $\delta(w_2) = o_{R,A}$. Then $w_2$ is of the form
        ${f_{R,A}^{A_2}(w')}$. Because such term can only be introduced in $I$
        by the application of a rule of the form $A_2(x) \rightarrow \exists z.
        R(x,z)\wedge A(z)$, a term $w_1'$ must exist such that ${\norm{R(w_1',
        w_2)\wedge A(w_2)}{I} \subseteq I}$. As stated above, we also have
        ${\norm{R(w_1, w_3')}{I} \in I}$. By the reflexivity of $\subrole$, we
        also have $R\subrole R$, so properties (\ref{aux:dat-cond4}) and
        (\ref{aux:dat-cond5}) are satisfied. As $\delta(w_3') = o_{R,A}$,
        property (\ref{aux:dat-cond7}) is also satisfied.
    \end{itemize}

    Consider a rule $R(x, y) \wedge R(y, z) \rightarrow R(x,z)$, and assume
    ${\setof{R(\delta(w_1), \delta(w_2)), R(\delta(w_3), \delta(w_4))}\subseteq
    J_n}$ and ${\delta(w_2) = \delta(w_3)}$. It follows that ${R(\delta(w_2),
    \delta(w_4)) \in J_n}$. We distinguish four cases.
    \begin{itemize}
        \item $\delta(w_2)\in \indnames$ and $\delta(w_4)\in \indnames$. By the
        inductive hypothesis, we have ${\norm{R(w_1, w_2)\wedge R(w_2, w_4)}{I}
        \in I}$. Since no rule is applicable to $I$, we have ${\norm{R(w_1,
        w_4)}{I} \in I}$.

        \item $\delta(w_2)$ is of the form $o_{P,A}$ and $\delta(w_4)\in
        \indnames$. By the inductive hypothesis, a term $w_2'$ exists such that
        $\delta(w_2') = o_{P,A}$ and $\norm{R(w_1, w_2')}{I} \in I$. Due to
        $\delta(w_2') = \delta(w_2)$, we have ${\norm{R(\delta(w_2'),
        \delta(w_4))}{J_n} \in J_n}$. By the inductive hypothesis, we have
        ${\norm{R(w_2', w_4)}{I} \in I}$. Since no rule is applicable to $I$,
        we have ${\norm{R(w_1, w_4)}{I} \in I}$.

        \item $\delta(w_2)\in \indnames$ and $\delta(w_4)$ is of the form
        $o_{R,B}$. By the inductive hypothesis, we have ${\norm{R(w_1, w_2)}{I}
        \in I}$ and terms $w_2'$ and $w_4'$ exist such that $\delta(w_4') =
        o_{R,B}$ and ${\norm{R(w_2', w_4) \wedge R(w_2, w_4')}{I} \subseteq
        I}$. As no rule is applicable to $I$, we have ${\norm{R(w_1, w_4')}{I}
        \in I}$.

        \item $\delta(w_2)$ is of the form $o_{P,A}$ and $\delta(w_4)$ is of
        the form $o_{R,B}$. By the inductive hypothesis, a term $w_2'$ exists
        such that $\delta(w_2') = o_{P,A}$ and $\norm{R(w_1, w_2')}{I} \in I$.
        Due to $\delta(w_2') = \delta(w_2)$, we have ${\norm{R(\delta(w_2'),
        \delta(w_4))}{J_n} \in J_n}$. By the inductive hypothesis, terms $u_2'$
        and $w_4'$ exist such that $\delta(w_4') = o_{R,B}$ and ${\norm{R(u_2',
        w_4) \wedge R(w_2', w_4')}{I} \subseteq I}$. As no rule is applicable
        to $I$, we have ${\norm{R(w_1, w_4')}{I} \in I}$.
    \end{itemize}

    Consider a rule $A(x) \rightarrow R(x, x) \wedge \SELF_{R}(x)$, assume
    ${A(\delta(w_1)) \in J_n}$, and let $w_2$ be a term such that ${\delta(w_2)
    = \delta(w_1)}$. By the inductive hypothesis, we have $\norm{A(w_1) \wedge
    A(w_2)}{I} \subseteq I$. Since no rule is applicable to $I$, we have
    ${\norm{R(w_1, w_1)\wedge \SELF_R(w_1)}{I} \subseteq I}$ and ${\norm{R(w_2,
    w_2)\wedge \SELF_R(w_2)}{I} \subseteq I}$, so property
    (\ref{aux:dat-cond1}) holds. We distinguish two cases.
    \begin{itemize}
        \item $\delta(w_2) \in \indnames$. Then individuals $u_0,\ldots, u_n$
        with $u_0 = \delta(w_1)$ and $u_n =\delta(w_2)$ exist such
        that ${u_{i-1}\rightsquigarrow u_i \in J_n}$ for each ${i \in
        \interval{1}{n}}$. By the inductive hypothesis, we have, $\norm{w_1}{I}
        = \norm{w_2}{I}$; thus ${\norm{R(w_1, w_2)}{I} \in I}$.

        \item $\delta(w_2)$ is of the form $o_{P,A}$. Thus, $\delta(w_1) =
        o_{P,A}$. As stated above, we have ${\norm{R(w_1, w_1)\wedge R(w_2,
        w_2)}{I} \subseteq I}$.
    \end{itemize}

    Consider a rule $R(x,x) \wedge \ind(x) \rightarrow \SELF_R(x)$. Assume that
    ${\setof{R(\delta(w_1), \delta(w_2)), \ind(\delta(w_2))}\subseteq J_n}$ and
    ${\delta(w_2) = \delta(w_1)}$. By the definition of \close{\K} and since no
    rule in $\dat_\T$ derives atoms over $\ind$, we have that $\delta(w_2) \in
    \indnames$. By the inductive hypothesis, we then have ${\norm{R(w_1,
    w_2)\wedge \ind(w_2)}{I} \subseteq I}$. Since $\delta$ is the identity on
    \indnames and $\delta(w_1) = \delta(w_2)$, we have $w_1 = w_2$, and so
    ${\norm{R(w_1, w_1) \wedge R(w_2, w_2)}{I} \subseteq I}$. As no rule is
    applicable to $I$, we have ${\norm{\SELF_R(w_1) \wedge \SELF_R(w_2)}{I}
    \subseteq I}$.
\end{proof}

Lemmas \ref{lemma:homomorphism} and \ref{lemma:dat-embed} follow immediately
from Lemmas \ref{lemma:aux:homomorphism} and \ref{lemma:aux:dat-embed}, and the
following result.
\begin{lemma}\label{lemma:equality}
    For each term $w$ occurring in $I$ and each individual $a\in\indnames$, the following two properties hold.
    \begin{enumerate}[E1.]
        \item  $\norm{w}{I} = a$ if and only if $\norm{\delta(w)}{J} = a$.
        \item  $\norm{w}{I} = w$ if and only if $\norm{\delta(w)}{J} = \delta(w)$.
    \end{enumerate}
\end{lemma}
\begin{proof}
    By the definition of $\Xi_{\K}$ and $\dat_{\K}$, each equality rule
    occurring in these rule bases is of the form $A(x) \rightarrow x\approx a$
    with ${a \in \indnames}$. Consequently, for all terms $u$ and $u'$ such
    that $u\rightsquigarrow u' \in I \cup J$, we have $u' \in \indnames$. Let
    $w$ be an arbitrary term occurring in $I$, and let $a\in\indnames$ be an
    arbitrary individual.

    We first prove E1. $(\Rightarrow)\;$ Assume that $\norm{w}{I} = a$. Then
    terms $w_0,\ldots, w_n$ with $w_0 = w$ and $w_n =a$ exist such that for
    each $i \in \interval{1}{n}$ we have $w_i\in\indnames$ and
    $w_{i-1}\rightsquigarrow w_i\in I$. By Lemma \ref{lemma:aux:homomorphism},
    we have $\norm{\delta(w_0)}{J} = \norm{\delta(w_n)}{J}$; that is,
    $\norm{\delta(w)}{J} = \norm{a}{J}$. We show that $\norm{a}{J} = a$. Assume
    the opposite; hence, an individual $b$ exists such that $a \rightsquigarrow
    b \in J$ and $a >b$. By Lemma \ref{lemma:aux:dat-embed}, we have
    $\norm{a}{I} = \norm{b}{I}$. Since $a> b$, we have $\norm{w}{I} = b$, which
    is a contradiction. $(\Leftarrow)\;$ Assume that $\norm{\delta(w)}{I} = a$.
    Then individuals $u_0,\ldots, u_n$ with $u_0 = \delta(w)$ and $u_n =a$
    exist such that for each $i \in \interval{1}{n}$ we have $u_i\in\indnames$
    and $u_{i-1}\rightsquigarrow u_i\in I$. By Lemma \ref{lemma:aux:dat-embed},
    we have $\norm{w}{I} = \norm{a}{I}$. We show that $\norm{a}{I} = a$. Assume
    the opposite; hence, an individual $b$ exists such that $a \rightsquigarrow
    b \in I$ and $a >b$. By Lemma \ref{lemma:aux:homomorphism}, we have
    $\norm{a}{J} = \norm{b}{J}$. Since $a> b$, we have $\norm{\delta(w)}{I} =
    b$, which is a contradiction.

    \smallskip

    Next, we prove property E2 by contraposition. $(\Rightarrow)\;$ Assume that
    $\norm{\delta(w)}{J} \neq \delta(w)$; hence, an individual $b\in\indnames$
    exists such that $\delta(w) \neq b$ and $\norm{\delta(w)}{J} =b$. By the
    definition of $\delta$, we have $w \neq b$. By property E1, we have
    $\norm{w}{I} = b$, as required. $(\Leftarrow)\;$ Assume that $\norm{w}{I}
    \neq w$; hence, an individual $b\in\indnames$ exists such that $w \neq b$
    and $\norm{w}{I} =b$. By the definition of $\delta$, we have $\delta(w)
    \neq b$. By property E1, we have $\norm{\delta(w)}{I} = b$, as required.
\end{proof}

\section{Proof of Theorem \ref{th:correctness}}

Let KB $\K =\tuple{\T,\A}$ be a satisfiable \elhso KB, let $\Xi_{\K}$, and
$\dat_{\K}$ be the rule base and the datalog program associated with $\K$,
respectively; moreover, let $I$ and $J$ be universal interpretations of
$\Xi_{\K}$ and $\dat_{\K}$, respectively. To prove Theorem
\ref{th:correctness}, we first show that our function is sound, after which we
show that it is also complete.

\subsection{Soundness}

Let $q' = \exists \vec{y}.\, \psi(\vec{x}, \vec{y})$ be a CQ, let $\tau'$ be a
candidate answer to $q'$ over $\dat_\K$, and let $\pi' =
\sproj{\tau'}{\vec{x}}$. Assume that the two following conditions hold:
\begin{enumerate}
    \item for each $x\in\vec{x}$, we have $\tau'(x)\in\indnames$, and

    \item a nondeterministic computation exists such that function
    ${\sound{q}{\dat_{\K}}{\tau}}$ returns $\true$.
\end{enumerate}
By the definition of candidate answer, we have ${\dom{\tau'} = \vars{q'}}$,
each element of $\rng{\tau'}$ is an individual occurring in $\dat_\K$, and
${\dat_\K \models \tau'(q')}$. Since $\sproj{\tau'}{\vec{x}} \subseteq
\indnames$, we have that ${\pi'(x) \in \indnames}$ for each $x\in\vec{x}$. In
the rest of this proof we show that ${\Xi_\K \models \pi'(q')}$.

\smallskip

Let CQ $q$ and substitution $\tau$ be as specified in Definition
\ref{def:dat-order}; and let relation $\sim$, CQ $q_{\sim}$ and the connection graph
${\mathsf{cg} = \tuple{V, E_s, E_t}}$ be as determined by $\mathsf{isSound}$.
By the construction of $q$ and $\tau$, we have $\dat_{\K} \models \tau(q)$ and
$\tau(q) \subseteq J$. We next construct a substitution $\pi$ with $\dom{\pi} =
\vars{q}$ such that $\pi(q) \subseteq I$ and the following property holds.
\begin{enumerate}[(1)]
  \item \label{eq:compatible} For each term $t\in\terms{q}$, we have $\delta(\pi(t)) = \tau(t)$.
\end{enumerate}
Later, we will show that property \eqref{eq:compatible} and $\pi(q)\subseteq I$
imply that $\Xi_\K \models \pi'(q')$, thus proving the soundness claim.

To construct substitution $\pi$, we proceed in two steps: we first show how to
construct $\pi$ in case our algorithm returns \true in step $2$; after which we
show how to construct $\pi$ in case our algorithm returns \true in step $18$.

\subsubsection{Case 1: $\mathsf{isSound}$  returns \true in step $2$ }

Assume that ${\mathsf{isSound}(q, \dat_{\K}, \tau)}$ returns \true in step $2$.
By condition \ref{spur:cond2} in Definition \ref{def:dsound}, directed graph $\tuple{V,
E_s}$ is acyclic; we next show that $\tuple{V,E_s}$ is a forest, after which we will
show how to construct substitution $\pi$ by structural induction on this forest.
\begin{lemma}
    Directed graph $\tuple{V, E_s}$ is a forest.
\end{lemma}
\begin{proof}
    Since directed graph $\tuple{V, E_s}$ is acyclic, we are left to show that
    for each $v\in V$, there exists at most one vertex $v'$ such that
    $\tuple{v',v}\in E_s$. Assume the opposite; hence, vertices $v_1$, $v_2$,
    and $v$ exist in $V$ such that $v_1 \neq v_2$ and $\setof{\tuple{v_1, v},
    \tuple{v_2, v}}\subseteq E_s$. Then, roles $R$ and $P$ exist such that
    $R(v_1, v)$ and $P(v_2, v)$ are aux-simple atoms in $q_{\sim}$ and
    ${\tau(v) \in \aux{\dat_\K}}$. By the definition of $\sim$, we have $v_1
    \sim v_2$; and, by the construction of $q_{\sim}$, we have $v_1 = v_2$,
    which is a contradiction.
\end{proof}

We next construct substitution $\pi$ with $\dom{\pi} = \vars{q}$ that will
satisfy \eqref{eq:compatible} and the two following properties:
\begin{enumerate}
    \item[(2)] for all terms $s,t \in \terms{q}$ such that $s\sim t$, we have $\pi(s) =
    \pi(t)$, and

    \item[(3)] for each $\tuple{v',v} \in E_s$ with $\tau(v)$ of the form
    $o_{P,A}$, we have $P(\pi(v'), \pi(v)) \in I$.
\end{enumerate}
We define $\pi$ by structural induction on the forest $\tuple{V,E_s}$; later we
show that $\pi(q)\subseteq I$.

\smallskip

\basecase Consider a root $v\in V$. Fix an arbitrary term $w\in\domain{I}$ such
that $\delta(w) = \tau(v)$. For each term $s\in \terms{q}$ with $s \sim v$, let
$\pi(s) = w$. By condition \ref{spur:cond1} in Definition \ref{def:dsound}, we
have $\tau(s) = \tau(v)$. Thus property \eqref{eq:compatible} and {(2)} hold.

\smallskip

\indstep Consider an arbitrary $\tuple{v',v} \in E_s$ with $\pi(v')$ defined
and $\pi(v)$ undefined. By the definition of $E_s$, a role $R\in \rolenames$
exists such that $R(v',v)$ is an aux-simple atom in $q_{\sim}$. Hence, we have
${\directedge{R}(\tau(v'), \tau(v)) \in J}$ and $\tau(v)$ is of the form
$o_{P,A}$. Since by property \eqref{eq:compatible} we have ${\delta(\pi(v')) =
\tau(v')}$, and due to property D\ref{dat-cond5} in Lemma
\ref{lemma:dat-embed}, a term ${w \in \domain{I}}$ exists such that
${P(\pi(v'), w) \in I}$. Then, for each term $s \in\terms{q}$ with $s \sim v$,
let $\pi(s) = w$. Properties (2) and (3) immediately hold. By condition $2$ in
Definition \ref{def:dsound}, we have $\tau(s) = \tau(v)$, and so property
\eqref{eq:compatible} holds.

\begin{lemma}\label{lemma:issound2}
      Substitution $\pi$ satisfies $\pi(q)\subseteq I$.
\end{lemma}
\begin{proof}
    Recall that $\tau(q)\subseteq J$ and ${\mathsf{isSound}(q, \dat_{\K},
    \tau)}$ returns \true in step $2$; we next show that $\pi(q) \subseteq I$.

    \smallskip

    Consider an atom $A(s)$ in $q$. By assumption, we have ${A(\tau(s)) \in
    J}$. By Lemma \ref{lemma:dat-embed}, for each $w \in\domain{I}$ with
    $\delta(w) = \tau(s)$, we have $A(w) \in I$. By property
    \eqref{eq:compatible} in the definition of $\pi$, we have $A(\pi(s)) \in I$.

    \smallskip

    Consider an atom $R(s', t')$ in $q$. By the definition of $q_{\sim}$, an
    atom $R(s, t)$ occurs in $q_{\sim}$ such that $s' \sim s$ and $t' \sim t$.
    By condition \ref{spur:cond1} in the definition of $\mathsf{isDSound}$, we
    have ${\tau(s')= \tau(s)}$ and ${ \tau(t') =\tau(t)}$. By assumption, we
    have ${R(\tau(s'), \tau(t')) \in J}$ and so ${R(\tau(s), \tau(t)) \in J}$.
    By property (2) in the definition of $\pi$, it suffices to show that
    $R(\pi(s), \pi(t)) \in I$. Given that our algorithm returns \true in step
    $2$, exactly one of the following holds.

    \smallskip $R(s,t)$ is such that $\tau(t) \in \indnames$. By Lemma
    \ref{lemma:dat-embed}, for all terms $w',w \in\domain{I}$ with ${\delta(w')
    = \tau(s)}$ and ${\delta(w) = \tau(t)}$, we have $R(w', w) \in I$. By
    property \eqref{eq:compatible} in the definition of $\pi$, we have
    $R(\pi(s), \pi(t)) \in I$.

    \smallskip $R(s,t)$ is such that $s=t$, $\tau(t)\in\aux{\dat_{\K}}$, and
    $\SELF_{R}(\tau(t)) \in J$. By Lemma \ref{lemma:dat-embed}, for each term
    $w \in\domain{I}$ with $\delta(w) = \tau(t)$, we have $R(w, w) \in I$. By
    property \eqref{eq:compatible} in the definition of $\pi$, we have
    $R(\pi(t), \pi(t)) \in I$.

    \smallskip $R(s,t)$ is aux-simple. It follows that $\tau(t)$ is of the form
    $o_{P,A}$ and ${\directedge{R}(\tau(s), \tau(t)) \in J}$. By property
    D\ref{dat-cond5} of Lemma \ref{lemma:dat-embed}, we have $P\subrole R$. By
    the definition of $E_s$, we have $\tuple{s, t} \in E_s$. By property (3) in
    the definition of $\pi$, we have ${P(\pi(s), \pi(t)) \in I}$. Since no rule
    is applicable to $I$ and $P\subrole R$, we have ${R(\pi(s), \pi(t)) \in
    I}$.
\end{proof}

\subsubsection{Case 2: $\mathsf{isSound}$  returns \true in step $18$ }

We are left to show that, if our function returns \true in step $18$, then a
substitution $\pi$ with $\dom{\pi} = \vars{q}$ exists such that $\pi(q)
\subseteq I$ and property \eqref{eq:compatible} is satisfied.

\smallskip

Assume that ${\mathsf{isSound}(q, \dat_{\K}, \tau)}$ returns \true in step
$18$. Let variable renaming $\sigma$, skeleton $\S = \tuple{\V, \E}$, and
function $L$ be as determined by $\mathsf{isSound}$. By the definition of a
skeleton for $q$ and $\sigma$, graph $\S$ is a forest rooted in
$\V\cap\setind{\dat_{\K}}$. We next define substitution $\pi$ that will satisfy
property \eqref{eq:compatible} as well as the following properties:
\begin{enumerate}
    \item[(2)] for all terms $s,t\in\terms{q}$ such that $s\sim t$ or $\sigma(s) = t$,
    we have $\pi(s) = \pi(t)$,

    \item[(3)] for each $\tuple{v',v}\in\E$ and each role $P\in L(v',v)$, we have $P(\pi(v'),\pi(v)) \in I$.
\end{enumerate}
We
define $\pi$ by structural induction on the forest $\S = \tuple{\V, \E}$;
later we show that $\pi(q) \subseteq I$.

\basecase Consider a root $v\in \V\cap\setind{\dat_{\K}}$. Given that each
element in $\rng{\sigma}$ is a variable, no term $s\in\terms{q}$ exists such
that $\sigma(s) = v$. Then, for each term $s\in \terms{q}$ with $s \sim v$, let
$\pi(s) =v$. By condition \ref{spur:cond1} in Definition \ref{def:dsound}, we
have $\tau(s) = \tau(v)$. Thus, properties \eqref{eq:compatible} and (2) are
satisfied.

\smallskip

\indstep Consider ${\tuple{v',v}\in \E}$ such that $\pi(v')$ has been defined,
but $\pi(v)$ has not; and let ${u_0 = \tau(v')}$ and ${w_0 = \pi(v')}$. By the
definition of $\mathsf{exist}$, individuals $\setof{u_1, \ldots, u_n}\subseteq
\aux{\dat_\K}$ with $n>0$ and ${u_n = \tau(v)}$ exist such that for each ${i
\in \interval{1}{n}}$, we have $u_i$ is of the form $o_{T_i, A_i}$, and
${\directedge{P_j}(u_{i-1}, u_i) \in J}$ for each $P_j \in L(v',v)$. Then, for
each ${i \in \interval{1}{n}}$, by property D\ref{dat-cond5} of Lemma
\ref{lemma:dat-embed}, a term $w_i$ of the form ${f_{T_i, A_i}^{B_i}(w_{i-1})}$
exists in $\domain{I}$ such that $T_i(w_{i-1}, w_i) \in I$; moreover, $T_i
\subrole P_j$ for each role $P_j \in L(v',v)$. Since no rule is applicable to
$I$ and $T_i \subrole P_j$, we have $P_j(w_{i-1}, w_i)\in I$. For each term $s
\in \terms{q}$ such that $s\sim v$ or $\sigma(s) = v$, let $\pi(s) = w_n$.
Property (2) is clearly satisfied. Property \eqref{eq:compatible} is
also satisfied, since $\sigma(s) = v$ implies that $\tau(s) = \tau(v)$, by
construction of $\sigma$, and $s\sim v$ implies that ${\tau(s) = \tau(v)}$, by
condition $1$ in Definition \ref{def:dsound}. For property (3) we
distinguish two cases.
\begin{itemize}
    \item A role $P \in L(v', v)$ exists such that $\TRANS{P} \not \in \T$. By
    the definition of function $\mathsf{exist}$, we then have $n=1$.
    Consequently, $\pi(v') = w_0$ and $\pi(v) = w_1$; thus, $P_j(\pi(v'),
    \pi(v)) \in I$ for each $P_j \in L(v',v)$.

    \item For each $P_j \in L(v',v)$ we have $\TRANS{P_j} \in \T$. Since no
    rule is applicable and $P_j(w_{i-1}, w_i)\in I$ for each
    $i\in\interval{1}{n}$, we have ${P_j(w_0, w_n) \in I}$; that is,
    ${P_j(\pi(v'), \pi(v)) \in I}$.
\end{itemize}

\begin{lemma}\label{lemma:issound}
    Substitution $\pi$ satisfies $\pi(q)\subseteq I$.
\end{lemma}
\begin{proof}
    We show that $\pi(q) \subseteq I$ by considering the various atoms occurring in $q$.

    \smallskip

    Consider an atom $A(s)$ in $q$. By assumption, we have ${A(\tau(s)) \in
    J}$. By Lemma \ref{lemma:dat-embed}, for each $w \in\domain{I}$ with
    $\delta(w) = \tau(s)$, we have $A(w) \in I$. By property
    \eqref{eq:compatible} in the definition of $\pi$, we have $A(\pi(s)) \in I$.

    \smallskip

    Consider an atom $R(s', t')$ in $q$. By assumption, we have ${R(\tau(s'),
    \tau(t')) \in J}$. By the definition of $q_{\sim}$, terms $s''$ and $t''$
    occur in $q_{\sim}$ such that $s'\sim s''$, $t'\sim t''$, and $R(s'', t'')$
    is an atom in $q_{\sim}$. By condition \ref{spur:cond1} in the definition
    of $\mathsf{isDSound}$, we have ${\tau(s') = \tau(s'')}$ and ${\tau(t') =
    \tau(t'')}$. Therefore, ${R(\tau(s''), \tau(t'')) \in J}$. By the
    definition of $\sigma(q_{\sim})$, terms $s$ and $t$ occur in
    $\sigma(q_\sim)$ such that $\sigma(s'') = s$, $\sigma(t'') = t$, and $R(s,
    t)$ is an atom in $\sigma(q_{\sim})$. By the definition of variable
    renaming, we have $\tau(t'') = \tau(t)$ and $\tau(s'') = \tau(s)$. Thus
    ${R(\tau(s), \tau(t)) \in J}$. By property (2) in the definition of $\pi$,
    it suffices to show that $R(\pi(s), \pi(t)) \in I$. Towards this goal, we
    consider four distinct cases.

    \smallskip

    $R(s,t)$ is such that $\tau(t) \in\indnames$. By Lemmas
    \ref{lemma:dat-embed}, for all terms $w',w \in\domain{I}$ with $\delta(w')
    = \tau(s)$ and $\delta(w) = \tau(t)$, we have $R(w', w) \in I$. By property
    \eqref{eq:compatible} in the definition of $\pi$, we have $R(\pi(s),
    \pi(t)) \in I$.

    \smallskip

    $R(s,t)$ is such that $s=t$, $\tau(t)\in\aux{\dat_\K}$, and
    $\SELF_{R}(\tau(t)) \in J$. By Lemma \ref{lemma:dat-embed}, for each term
    $w \in\domain{I}$ with $\delta(w) = \tau(t)$ , we have $\norm{R(w, w)}{I}
    \in I$. By property \eqref{eq:compatible} in the definition of $\pi$, we
    have $R(\pi(t), \pi(t)) \in I$.

    \smallskip

    $R(s,t)$ is aux-simple. By the definition of $E_s$ and $\E$, we have
    $\tuple{s, t} \in \sigma(E_s)\cap\E$, and $R\in L(s,t)$ by step $6$. By
    property (3) in the definition of $\pi$, we have ${R(\pi(s), \pi(t)) \in
    I}$.

    \smallskip

    $R(s,t)$ is neither good nor aux-simple. Let $P$ and $v_0,\ldots, v_n$
    be as determined in steps $8$--$15$ when Algorithm \ref{algo:issound}
    considers atom $R(s,t)$. Then by step $8$ we have ${P \subrole R}$.
    Furthermore, for each $i \in \interval{1}{n}$, we have ${P \in L(v_{i-1},
    v_i)}$ by step $15$; but then, by property (3) we have
    $P(\pi(v_{i-1}),\pi(v_i)) \in I$. Next, we distinguish two cases.
    \begin{itemize}
        \item $s$ reaches $t$ in \E. If $\tuple{s,t}\in \E$, then $n =1$ and,
        by property {(3)} in the definition of $\pi$, we have $P(\pi(s),
        \pi(t)) \in I$. Otherwise, we have $\TRANS{P} \in \T$ and, since no
        rule is applicable to $I$, we have ${P(\pi(s), \pi(t)) \in I}$.

        \item $s$ does not reach $t$ in $\E$. Then, we have $\TRANS{P} \in
        \T$, ${v_0 = a_t}$, and ${P(\tau(s), v_0) \in J}$. As $v_0 \in
        \indnames$ and ${\pi(v_0) = v_0}$, by Lemma \ref{lemma:dat-embed}, we
        have ${P(\pi(s), \pi(v_0)) \in I}$. Due to ${\TRANS{P}\in\T}$ and no
        rule is applicable to $I$, we have ${\norm{P(\pi(s), \pi(t))}{I} \in
        I}$.
    \end{itemize}
    Since $P\subrole R$ and no rule is applicable to $I$ we have
    ${R(\pi(s), \pi(t)) \in I}$.
\end{proof}

Finally, we prove the soundness claim.
\begin{lemma}
    Substitution $\pi'$ satisfies $\Xi_\K \models \pi'(q')$.
\end{lemma}
\begin{proof}

    To prove the lemma, we show that a substitution $\pi'_*$ with $\dom{\pi'_*}
    = \vars{q'}$ exists such that $\pi' \subseteq \pi'_*$ and
    ${\norm{\pi_*'(q')}{I}\subseteq I}$. By the construction of $q$ and $\tau$,
    we have $\dat_{\K} \models \tau(q)$ and $\tau(q) \subseteq J$. Let $\pi$ be
    the substitution specified just above Lemma \ref{lemma:issound2}, if
    $\mathsf{isSound}(q,\dat_\K, \tau)$ returns \true in step $2$, otherwise,
    let $\pi$ be the substitution specified just above Lemma
    \ref{lemma:issound}. By Lemmas \ref{lemma:issound2} and
    \ref{lemma:issound}, we have $\pi(q) \subseteq I$; furthermore, $\pi$
    satisfies property \eqref{eq:compatible}. Let $\gamma$ be the mapping from
    $\terms{q'}$ to $\terms{q}$ such that $q$ is obtained by replacing each
    $t\in\terms{q'}$ with $\gamma(t)$.

    Let $\pi_*'$ be the substitution such that for each $z\in\vars{q'}$, we
    have $\pi_*'(z) = \pi(z)$, if $\gamma(z) =z$; otherwise, we define
    $\pi_*'(z)$ as an arbitrary term $w$ occurring in $I$ such that $\delta(w)
    = \tau'(z)$. Since $\pi$ satisfies property \eqref{eq:compatible} and by
    construction $\pi_*'$, for each term $t\in\terms{q'}$, we have
    $\delta(\pi'_*(t)) = \tau'(t)$. Since $\tau'(x)\in\indnames$ for each
    $x\in\vec{x}$ and given that $\delta$ is the identity on \indnames, we have
    $\pi'\subseteq \pi'_*$. To prove that $\norm{\pi_*'(q')}{I}\subseteq I$, we
    show that for each term $t\in\terms{q'}$, we have $\norm{\pi_*'(t)}{I} =
    \pi(\gamma(t))$. The property clearly holds for each term $t\in\terms{q'}$
    such that $\gamma(t) = t$. Then consider an arbitrary term $t\in\terms{q'}$
    such that $\gamma(t)\neq t$. By the definition of $\gamma$, we have
    $\gamma(t) = \norm{\tau'(t)}{J}$ and $\gamma(t)\in\indnames$. By Lemma
    \ref{lemma:equality}, for each term $w$ occurring in $I$ with $\delta(w) =
    \tau'(t)$, we have $\norm{w}{I} = \gamma(t)$. By the definition of
    $\pi_*'$, we then have $\norm{\pi_*'(t)}{I} = \gamma(t) = \pi(\gamma(t))$.
    \end{proof}

\subsection{Completeness}

To prove the completeness claim, we start by establishing two properties of the
universal interpretation $I$. To this end, we start with a couple of
definitions.

Let $\prec$ be the smallest irreflexive and transitive relation on the set of
terms occurring in $I$ such that $w \prec f_{R,A}^{A_1}(w)$ for each term
occurring in $I$, each role $R$, and all concepts ${A, A_1 \in \setof{\top}
\cup \conceptnames}$. Furthermore, an atom $R(w',w) \in I$ is a
\emph{self-loop} if $w' = w$ and $\SELF_{R}(w) \in I$.

\begin{lemma}\label{lemma:univ}
    Interpretation $I$ satisfies the following properties for each role
    ${R \in \rolenames}$ and all terms ${w',w \in \domain{I}}$ with $w
    \not \in \indnames$.
    \begin{enumerate}
        \item ${R(w', w) \in I}$ is not a self-loop and $R$ is simple imply
        that $w$ is of the form $f_{T,A}^{A_1}(w')$.

        \item ${R(w', w) \in I}$ is not a self-loop implies that a role
        $P\in\rolenames$ and terms $w_0,\ldots,w_m$ from $\domain{I}$ with $w_m
        = w$ exist where
        \begin{enumerate}[{2}a.]
            \item if $m > 1$ or $w' \not \prec w$, then $\TRANS{P} \in
            \T$,
            \item $P \subrole R$,

            \item $w_0 = w'$, if $w'\prec w$; otherwise, $w_0$ is the unique
            individual $a\in\indnames$ with $a\prec w$, and ${P(w',
            w_0) \in I}$, and

            \item for each $i \in \interval{1}{m}$, $w_i$ is of the form
            $f_{S_i,A_i}^{B_i}(w_{i-1})$ and $P(w_{i-1}, w_{i}) \in I$.
        \end{enumerate}
    \end{enumerate}
\end{lemma}
\begin{proof}
    Let $I_0, I_1, \ldots$ be the chase sequence used to construct $I$. Please
    observe that, by virtue of the pruning step in the application of equality
    rules, for each term $w$ occurring in $I$ of the form $w = f_{T,A}^{A_1}(w')$
    we either have $\norm{w}{I}= a$ for some individual $a\in\indnames$, or
    $\norm{w'}{I} = w'$ and $\norm{w}{I} = w$. Then to prove the lemma, we show
    by induction on this sequence that each $I_n$ satisfies the following
    properties for each role ${R \in \rolenames}$ and all terms $w'$ and $w$
    occurring in $I$ such that $w\not\in\indnames$.
    \begin{enumerate}[A.]
        \item ${R(w', w) \in I_n}$ is not a self-loop and $R$ is simple imply
        that $w$ is of the form $f_{T,A}^{A_1}(w')$.

        \item ${R(w', w) \in I_n}$ is not a self-loop implies that a role $P
        \in \rolenames$ and terms $w_0,\ldots,w_m$ from $I$ with ${w_m =
        w}$ exist where
        \begin{enumerate}[(i)]
            \item if $m > 1$ or $w' \not \prec w$, then $\TRANS{P} \in
            \T$,
            \item $P \subrole R$,

            \item $w_0 = w'$, if $w'\prec w$; otherwise, $w_0$ is the unique
            individual $a\in\indnames$ with $a\prec w$, and ${\norm{P(w',
            w_0)}{I} \in I}$, and

            \item for each $i \in \interval{1}{m}$, $w_i$ is of the form
            $f_{S_i,A_i}^{B_1}(w_{i-1})$ and $\norm{P(w_{i-1}, w_{i})}{I} \in I$.
        \end{enumerate}
    \end{enumerate}

    \basecase Consider $I_0$. All atoms in $I_0$ are over the individuals in
    $\indnames$, so properties A and B hold vacuously.

    \smallskip

    \indstep Consider an arbitrary ${n \in \nat}$ and assume that $I_n$
    satisfies properties A and B. By considering each rule in $\Xi_{\K}$ that
    derives binary atoms, we assume that the rule is applicable to $I_n$, and
    we show that the properties hold for all fresh atoms in the resulting
    instance.

    \smallskip

    Consider a rule $A_1(x) \rightarrow \exists z. R(x, z) \wedge A(z)$, and
    assume that ${A_1(w') \in I_n}$. Furthermore, let $w = f_{R,A}^{A_1}(w')$,
    and assume that $w'$ and $w$ occur in $I$. Clearly, if $\norm{w}{I_n} = a$
    for some individual $a\in\indnames$, then the properties hold vacuously;
    hence we consider the case in which $\norm{w}{I_n}=w$. Please note that $w'
    \prec w$. Since $w'$ occurs in $I$, we have ${\norm{A_1(w')}{I} \in I}$. As
    no rule is applicable to $I$, we have ${\norm{R(w', w)}{I} \in I}$.
    Property A holds, and property B is satisfied for role $R$ and terms $w_0 =
    w'$ and $w_1 = w$.

    \smallskip

    Consider a rule $A(x) \rightarrow x \approx a$, and assume that ${A(u) \in
    I_n}$. Let $u'$ and $w'$ be terms occurring in $I$ such that ${\setof{u',
    w'} = \setof{\norm{u}{I_n}, \norm{a}{I_n}}}$ and $u' > w'$. By the
    definition of $>$, we have $w'\in\indnames$. Since $u$ occurs in $I$, we
    have $\norm{A(u)}{I} \in I$. Since no rule is applicable to $I$, we have
    $\norm{u'}{I} = \norm{w'}{I}$. We next show that the properties are preserved
    for all atoms in $I_n$ that get replaced by the application of this rule.
    Please observe that the properties hold vacuously for each atom $R(w,
    u')\in I_n$ and each atom $R(u',b)\in I_n$ with $b\in\indnames$. Then
    consider an atom $R(u', w) \in I_n$ with $w\neq u'$ and
    $w\not\in\indnames$. Then $R(u', w) \in I_n$ is not a self-loop. Since this
    atom is replaced, not removed, by the application of the rule, we must have
    that $u' \not \prec w$. Let $c\in\indnames$ be the unique individual such that
    $c\prec w$. By the inductive hypothesis, a role ${P \in \rolenames}$ and
    terms ${w_0,\ldots, w_m}$ with $w_0 = c$ and $w_m = w$ exist satisfying
    properties {\it (i)}--{\it (iv)}. Then $\norm{P(u', w_0)}{I} \in I$,
    $\TRANS{P} \in \T$, and $P\subrole R$. Role $R$ is not simple, thus
    property A holds. We next distinguish two cases.
    \begin{itemize}
        \item $w' \prec w$. Since $w'\in\indnames$, we have $w' = w_0$. Then
        role $P$ and terms $w_0,\ldots, w_m$ satisfy properties {\it (i)}--{\it
        (iv)}.

        \item $w'\not \prec w$. As stated above, we have $\TRANS{P} \in \T$ and
        $\norm{P(u',w_0)}{I} \in I$. Thus, $\norm{P(w',w_0)}{I} \in I$, and
        role $P$ and terms $w_0,\ldots, w_m$ satisfy properties {\it (i)}--{\it
        (iv)}.
    \end{itemize}

    Consider a rule $T(x, y) \rightarrow R(x, y)$, and assume that ${T(w', w)
    \in I_n}$ is not a self-loop. For property A, assume that $R$ is a simple
    role. Hence, $T$ is a simple role as well, and property A follows from the
    inductive hypothesis. For property B, by the inductive hypothesis, a role
    $P$ and terms $w_0, \ldots, w_m$ with $w_m = w$ exist satisfying {\it
    (i)}--{\it (iv)}. By the definition of $\Xi_{\K}$, we have $T\ISA R\in \T$.
    Since $P\subrole T$, we then have $P\subrole R$, and properties {\it
    (i)}--{\it (iv)} are satisfied.

    \smallskip

    Consider a rule $R(x, y) \wedge R(y, z) \rightarrow R(x,z)$, and assume
    that ${\setof{R(w', w''), R(w'', w)}\subseteq I_n}$. Property A holds
    vacuously, as $R$ is not simple. For property B, we distinguish four cases.
    \begin{itemize}

        \item $w' \prec w''$ and $w'' \prec w$. By the inductive hypothesis, a
        role $P_1$ and terms $w_0^1,\ldots, w_{m_1}^1$ with $w_0^1 = w'$ and
        $w_{m_1}^1 = w''$ exist satisfying properties {\it (i)}--{\it (iv)};
        moreover, a role $P_2$ and terms $w_0^2,\ldots, w_{m_2}^2$ with $w_0^2
        = w''$ and $w_{m_2}^2 = w$ exist satisfying properties {\it (i)}--{\it
        (iv)}. Then, for $j = 1,2$, we have $P_j \subrole R$. Furthermore, for
        each $i \in\interval{1}{m_j}$, we have $\norm{P_j(w_{i-1}^j,
        w_{i}^j)}{I} \in I$. As no rule is applicable to $I$, we have
        $\norm{R(w_{i-1}^j, w_{i}^j)}{I} \in I$. Then, property B is
        satisfied for role $R$ and terms $w_0^1, \ldots, w_{m_1}^1,
        w_1^2,\ldots w_{m_2}^2$.

        \item $w' \prec w''$ and $w'' \not \prec w$. By the inductive
        hypothesis, a role $P_1$ and terms $w_0^1,\ldots, w_{m_1}^1$ with
        $w_0^1 = w'$ and $w_{m_1}^1 = w''$ exist satisfying properties {\it
        (i)}--{\it (iv)}. Let ${b \in \indnames}$ be the unique individual such
        that $b \prec w$. Since ${w \not \in \indnames}$, by the inductive
        hypothesis, a role $P_2$ and terms $w_0^2,\ldots, w_{m_2}^2$ with
        $w_0^2 = b$ and $w_{m_2}^2 = w$ exist satisfying properties {\it
        (i)}--{\it (iv)}. By property {\it (iii)}, we have
        ${\norm{P_2(w_{m_1}^1, w_0^2)}{I} \in I}$. Then, for $j = 1,2$, we have
        $P_j \subrole R$. Moreover, for each $i \in\interval{1}{m_j}$, we have
        $\norm{P_j(w_{i-1}^j, w_{i}^j)}{I} \in I$. As no rule is applicable to
        $I$, we have $\norm{R(w_{i-1}^j, w_{i}^j)}{I} \in I$ and
        $\norm{R(w_{m_1}^1, w_0^2)}{I} \in I$. Since $\TRANS{R}\in\T$ and no
        rule is applicable to $I$, we have $\norm{R(w_0^1, w_{0}^2)}{I} \in I$.
        Property B is satisfied for role $R$ and terms ${w_0^2,\ldots,
        w_{m_2}^2}$.

        \item $w' \not \prec w''$ and $w'' \prec w$. By
        the inductive hypothesis, a role $P_2$ and terms $w_0^2,\ldots,
        w_{m_2}^2$ with $w_0^2 = w''$ and $w_{m_2}^2 = w$ exist satisfying
        properties {\it (i)}--{\it (iv)}. Then, we have ${P_2 \subrole R}$.
        Furthermore, for each $i \in\interval{1}{m_2}$, we have
        $\norm{P_2(w_{i-1}^2, w_{i}^2)}{I} \in I$. As no rule is applicable to
        $I$, we have $\norm{R(w_{i-1}^2, w_{i}^2)}{I} \in I$. We distinguish
        two cases.
        \begin{itemize}
            \item $w'' \in \indnames$. It follows that $w_0^2 = w''$. Then
            property B is satisfied for role $R$ and terms $w_0^2,\ldots,
            w_{m_2}^2$.

            \item $w'' \not \in \indnames$. Let $a \in\indnames$ be the unique
            individual such that $a\prec w''$. By the inductive hypothesis, a
            role $P_1$ and terms $w_0^1,\ldots, w_{m_1}^1$ with $w_0^1 = a$ and
            $w_{m_1}^1 = w''$ exist satisfying properties {\it (i)}--{\it (iv)}.
            Then, we have ${P_1 \subrole R}$ and ${\norm{P_1(w', a)}{I} \in I}$.
            Furthermore, for each ${i \in \interval{1}{m_1}}$, we have
            ${\norm{P_1(w_{i-1}^1, w_{i}^1)}{I} \in I}$. As no rule is applicable
            to $I$, we have $\norm{R(w', a)}{I} \in I$ and $\norm{R(w_{i-1}^1,
            w_{i}^1)}{I} \in I$. Then the property is satisfied for role $R$
            and terms $w_0^2,\ldots, w_{m_2}^2$.
        \end{itemize}
        \item $w' \not \prec w''$ and $w'' \not \prec w$. By assumption, we
        have $w\not \in\indnames$. Let $b\in\indnames$ be the unique individual such
        that $b\prec w$. By the inductive hypothesis, a
        role $P_2$ and terms $w_0^2,\ldots, w_{m_2}^2$ with $w_0^2 = b$ and
        $w_{m_2}^2 = w$ exist satisfying properties {\it (i)}--{\it (iv)}. Then,
        we have ${P_2 \subrole R}$ and $\norm{P_2(w'', w_{0}^2)}{I} \in I$.
        Furthermore, for each $i \in\interval{1}{m_2}$, we have
        $\norm{P_2(w_{i-1}^2, w_{i}^2)}{I} \in I$. As no rule is applicable to
        $I$, we have $\norm{R(w'', w_{0}^2)}{I} \in I$ and $\norm{R(w_{i-1}^2,
        w_{i}^2)}{I} \in I$. We distinguish two cases.
        \begin{itemize}
            \item $w'' \in \indnames$. As stated above, we have $\norm{R(w',
            w'')}{I} \in I$ and $\norm{R(w'', w_{0}^2)}{I} \in I$. Since
            $\TRANS{R} \in I$ and no rule is applicable to $I$, we have
            ${\norm{R(w', w_{0}^2)}{I} \in I}$. Then property B is satisfied
            for role $R$ and terms $w_0^2,\ldots, w_{m_2}^2$.

            \item $w'' \not \in \indnames$. Let $a \in\indnames$ be the unique
            individual such that $a\prec w''$. By the inductive hypothesis, a
            role $P_1$ and terms $w_0^1,\ldots, w_{m_1}^1$ with $w_0^1 = a$ and
            $w_{m_1}^1 = w''$ exist satisfying properties {\it (i)}--{\it (iv)}.
            Then, we have ${P_1 \subrole R}$ and ${\norm{P_1(w', w_{0}^1)}{I}
            \in I}$. Furthermore, for each $i \in\interval{1}{m_1}$, we have
            ${\norm{P_1(w_{i-1}^1, w_{i}^1)}{I} \in I}$. As no rule is applicable
            to $I$, we have ${\norm{R(w', w_{0}^1)}{I} \in I}$ and
            $\norm{R(w_{i-1}^1, w_{i}^1)}{I} \in I$. Since $\TRANS{R}\in\T$
            and no rule is applicable to $I$, we have $\norm{R(w', w_{0}^2)}{I}
            \in I$. Property B is satisfied for role $R$ and terms
            $w_0^2,\ldots, w_{m_2}^2$.  \qedhere
        \end{itemize}
    \end{itemize}
\end{proof}

Let $q' = \exists \vec{y}. \psi(\vec{x}, \vec{y})$ be a CQ and let $\pi'$ be a
substitution such that $\dom{\pi'} = \vec{x}$ and each element in $\rng{\pi'}$
is an individual from \indnames. Assume that $\Xi_{\K} \models \pi'(q')$; we
next show that a substitution $\tau'$ with $\dom{\tau'} = \vars{q'}$ exists
such that ${\sproj{\tau'}{\vec{x}} = \pi'}$, each element in ${\rng{\tau'}}$ is
an individual occurring in $\dat_\K$, and all of the following
conditions hold:
\begin{enumerate}
    \item for each $x\in\vec{x}$, we have $\tau'(x)\in\indnames$,

    \item ${\dat_{\K} \models \tau'(q')}$, and

    \item a nondeterministic computation exists such that function
    ${\sound{q}{\dat_{\K}}{\tau}}$ returns $\true$.
\end{enumerate}

Since $\Xi_{\K} \models \pi'(q)$, a substitution $\pi_{*}'$ with
$\dom{\pi_{*}'} = \vars{q'}$ exists such that $\pi'\subseteq \pi_*'$ and
$\norm{\pi_{*}'(q')}{I} \subseteq I$. Let substitution $\tau'$ be such that,
for each variable $z \in \vars{q'}$, we have ${\tau'(z) =
\delta(\pi_{*}'(z))}$. Since $\delta$ is the identity on \indnames and
$\pi_{*}'(x) \in \indnames$ for each $x\in\vec{x}$, we have $\pi' \subseteq
\tau'$ and $\tau'(x)\in \indnames$. By Lemmas \ref{lemma:homomorphism} and
\ref{lemma:equality}, we have ${\norm{\tau'(q')}{J} \subseteq J}$, and so
$\dat_\K \models \tau'(q')$. Hence, in the following, we show that property (3)
holds.

Let $q$ and $\tau$ be as specified in Definition \ref{def:dat-order}. Then, we
have $\dat_\K\models \tau(q)$ and $\tau(q)\subseteq J$. Let $\pi_*$ be the
restriction of $\pi_*'$ to only those variables that occur in $q$. By the
definition of $\tau$ and $\pi_*$, for each variable $z\in\vars{q}$, we have
$\delta(\pi_*(z)) = \tau(z)$. We next show that $\pi_*(q)\subseteq I$.

\begin{lemma}
    Substitution $\pi_*$ satisfies $\pi_*(q)\subseteq I$
\end{lemma}
\begin{proof}
    Let $\pi_*'$ and $\pi_*$ be as specified above; furthermore, let $\gamma$
    be the mapping from $\terms{q'}$ to $\terms{q}$ such that $q$ is obtained
    by replacing each term $t\in\terms{q'}$ with $\gamma(t)$. To prove that
    $\pi_*(q)\subseteq I$, we show that for each term $t\in\terms{q'}$, we have
    ${\pi_*(\gamma(t)) = \norm{\pi_*'(t)}{I}}$. We distinguish two cases.
    \begin{itemize}
        \item Consider an arbitrary variable $z\in\vars{q'}$ such that
        $\tau'(z)\in\aux{\dat_\K}$. Then $\gamma(z) = z$. By the definition of
        $\aux{\dat_\K}$, for each individual $u$ such that $\dat_{\K}\models
        \tau'(z) \approx u$, we have $u\not\in\indnames$ and $\tau'(z) = u$;
        therefore, $\norm{\tau'(z)}{J} = \tau'(z)$. But then, by the
        construction of $\tau'$ and by Lemma \ref{lemma:equality}, we have
        $\norm{\pi_{*}'(z)}{I} = \pi_{*}(z)$.

        \item Consider an arbitrary term $t\in\terms{q'}$ such that
        $\tau'(t)\not \in \aux{\dat_\K}$. By the definition of $q$, we have
        $\gamma(t) = \norm{\tau'(t)}{J}$ and $\gamma(t)\in\indnames$. By Lemma
        \ref{lemma:equality}, for each term $w$ occurring in $I$ with
        $\delta(w) = \tau'(t)$, we have $\norm{w}{I} = \gamma(t)$. By the
        definition of $\tau'$, we then have $\pi(\gamma(t)) = \gamma(t) =
        \norm{\pi_*'(t)}{I}$.\qedhere
    \end{itemize}
\end{proof}

Let relation $\sim$ and query $q_{\sim}$ be as specified in Definition
\ref{def:fork}; moreover, let connection graph ${\mathsf{cg} = \tuple{V, E_s,
E_t}}$ be as specified in Definition \ref{def:connection-graph}. We next show
that function $\mathsf{isDSound}(q, \dat_{\K}, \tau)$ returns \true.

\begin{lemma}\label{lemma:complete-dsound}
    Function $\mathsf{isDSound}(q, \dat_{\K}, \tau)$ returns \true.
\end{lemma}
\begin{proof}
    We next prove that $\mathsf{isDSound}(q, \dat_{\K}, \tau)$ return \true by
    showing that the two conditions of Definition \ref{def:dsound} are
    satisfied.

    \smallskip

    (Condition \ref{spur:cond1}) We prove that, for each ${s \sim t}$, we have
    ${\tau(s) = \tau(t)}$ and ${\pi_*(s) = \pi_*(t)}$. We proceed by induction
    on the number of steps required to derive ${s\sim t}$. For the base case,
    the empty relation $\sim$ clearly satisfies the two properties. For the
    inductive step, consider an arbitrary relation $\sim$ obtained in $n$ steps
    that satisfies these constraints; we show that the same holds for all
    constraints derivable from $\sim$. We focus on the $(\mathsf{fork})$ rule,
    as the derivation of ${s \sim t}$ due to reflexivity, symmetry, or
    transitivity clearly preserves the required properties. Let $s_1'$, $s_2$,
    $s_2'$, and $s_2$ be arbitrary terms in \terms{q}, and let $R_1$ and $R_2$
    be arbitrary roles such that ${s_1' \sim s_2'}$ is obtained in $n$ steps,
    atoms $R_1(s_1,s_1')$ and $R_2(s_2,s_2')$ occur in $q$ and are aux-simple,
    and ${\tau(s_2')\in\aux{\dat_{\K}}}$. By the definition of $\tau$, we have
    $\pi(s_2')\not\in\indnames$. Moreover, by the definition of aux-simple
    atom, for $i=1,2$ we have ${s_i\neq s_i'}$, role $R_i$ is simple, and
    ${\tau(s_i) = \tau(s_i')}$ implies that ${\SELF_{R_i}(\tau(s_i')) \not\in
    J}$. By the inductive hypothesis, we have ${\tau(s_1') = \tau(s_2')}$ and
    ${\pi_*(s_1') = \pi_*(s_2')}$. By Lemma \ref{lemma:homomorphism}, atoms
    ${\setof{R_1(\pi_*(s_1), \pi_*(s_1')), R_2(\pi_*(s_2), \pi_*(s_2'))}
    \subseteq \lpmodel}$ are not self-loops. Moreover, since $R_1$ and $R_2$
    are simple roles, by property 1 in Lemma \ref{lemma:univ}, we have
    ${\pi_*(s_1) = \pi_*(s_2)}$. Therefore, ${\pi_*'(s_1) = \pi_*'(s_2)}$. By
    the construction of $\tau'$, we have ${\tau'(s_1) = \tau'(s_2)}$, and so
    ${\tau(s_1) = \tau(s_2)}$.

    \smallskip

    (Condition \ref{spur:cond2}) We show that $\tuple{V,E_s}$ is a DAG. Assume
    the opposite; hence, vertices ${v_0, \ldots, v_m \in V}$ with $v_m = v_0$
    exist such that $m > 0$ and ${\tuple{v_{i-1},v_{i}} \in E_s}$ for each
    $i\in \interval{1}{m}$. By the definition of $E_s$, we have ${E_s\subseteq
    V\times (V\cap \vars{q})}$. Thus, for each $i \in\interval{0}{m}$ we have
    ${v_i \not\in \setind{\dat_\K}}$, and so $\tau(v_i)\in\aux{\dat_\K}$. Consider an
    arbitrary ${i \in \interval{1}{m}}$ and the corresponding edge
    ${\tuple{v_{i-1},v_{i}} \in E_s}$. By the definition of $E_s$, a role $R_i$
    exists such that $R_i(v_{i-1},v_{i})$ is an aux-simple atom in $q_{\sim}$.
    By the definition of aux-simple atom, we have ${v_{i-1}\neq v_i}$, role
    $R_i$ is simple, and ${\tau(v_{i-1}) = \tau(v_i)}$ implies that
    $\SELF_{R_i}(\tau(v_i)) \not\in J$. By Lemma \ref{lemma:homomorphism} and
    by the construction of $\tau$, atom ${R_i(\pi_*(v_{i-1}),\pi_*(v_i))\in
    \lpmodel}$ is not a self-loop. Since $\tau(v_i)\in\aux{\dat_{\K}}$, we have
    ${\pi_*(v_i)\not\in\indnames}$. Then, by property 1 in Lemma
    \ref{lemma:univ}, we have ${\pi_*(v_{i-1}) \prec \pi_*(v_i)}$. Thus,
    $\pi_*(v_0) \prec \pi_*(v_m)$ which contradicts ${v_m = v_0}$.
\end{proof}

We are now ready to prove the completeness claim.

\begin{lemma}
    A nondeterministic computation exists such that $\mathsf{isSound}(q, \dat_{\K}, \tau)$ returns \true.
\end{lemma}

\begin{proof}
Recall that $\pi_*(q)\subseteq I$ and that $\tau(q)\subseteq J$. By Lemma
\ref{lemma:complete-dsound}, the condition in step $1$ in Algorithm
\ref{algo:issound} is not satisfied. If each binary atom $R(s, t)$ occurring in
$q_{\sim}$ is either good or aux-simple, then our algorithm returns \true in
step $2$; hence, in the rest of this proof, we assume that this is not the case.

\smallskip

For the variable renaming $\sigma$ in step $3$, for each variable $z \in V$,
let $\sigma(z)$ be an arbitrary, but fixed, variable $z'\in V$ such that
$\pi_*(z) = \pi_*(z')$. It is straightforward to see that
${\pi_*(\sigma(q_{\sim})) \subseteq I}$.

\smallskip

For the skeleton $\S = \tuple{\V, \E}$ in step $3$, let $\V = \sigma(V)$ and
let $\E$ be the smallest set containing ${\tuple{v',v}\in \E}$ for all
${v',v\in\V}$ such that ${\pi_*(v') \prec \pi_*(v)}$ and no vertex ${v''\in
\V}$ exists such that ${\pi_*(v') \prec \pi_*(v'')\prec \pi_*(v)}$. By the
definition of $\prec$, graph $\tuple{\V,\E}$ is a forest rooted in
$\V\cap\setind{\dat_{\K}}$.

We next show that for each $\tuple{v', v}\in\sigma(E_s)$, we have
$\tuple{v',v}\in\E$. Consider an arbitrary edge $\tuple{v',v} \in\sigma(E_s)$.
Then vertices $u'$ and $u$ exist in $V$ such that $\tuple{u',u}\in E_s$,
$\pi_*(u') = \pi_*(v')$, and $\pi_*(u) = \pi_*(v)$. By the definition of $E_s$,
role $R$ exist such that $R(u', u)$ is an aux-simple atom in $q_{\sim}$. By the
definition of aux-simple atom, we have $\tau(u)\in\aux{\dat_\K}$, ${u' \neq
u}$, role $R$ is simple, and ${\tau(u') = \tau(u)}$ implies that
$\SELF_{R}(\tau(u)) \not\in J$. By Lemma \ref{lemma:homomorphism} and by the
construction of $\tau$, atom ${R(\pi_*(u),\pi_*(u))\in \lpmodel}$ is not a
self-loop. Since $\tau(u)\in\aux{\dat_{\K}}$, we have ${\pi_*(u) \not \in
\indnames}$. Then, by property 1 in Lemma \ref{lemma:univ}, we have $\pi_*(u)$
is of the form $f_{T,A}^B(\pi_*(u'))$. Thus, $\pi(u')\prec \pi(u)$ and no
vertex $v''\in\V$ exists such that $\pi(u')\prec \pi(v'') \prec \pi(u)$. Since
$\pi_*(u) = \pi_*(v)$ and $\pi_*(u') = \pi_*(v')$, we have $\tuple{v',v}\in\E$,
as required. Therefore, $\sigma(E_s)\subseteq \E$. Please note that since
$\tuple{\V, \E}$ is a forest, so is $\tuple{\sigma(V), \sigma(E_s)}$, as
required by condition $3$ in Definition \ref{def:var-renaming}.

It remains to show that each edge $\tuple{v',v}\in \E$ occurs in $\sigma(E_t)$.
Consider an arbitrary edge $\tuple{v',v} \in \E$. By the definition of $\E$, we
have $\pi_*(v')\prec \pi_*(v)$, and so $\pi_*(v)$ is a functional term. Then
let $w_0,\ldots,w_k$ be terms such that $w_0 = \pi_*(v')$, $w_k = \pi_*(v)$,
and $w_i$ is of the form ${f_{T_i,A_i}^{B_i}(w_{i-1})}$ for each $i \in
\interval{1}{k}$. Note that these terms are uniquely defined by the edge, and
that, by the construction of $I$, for each $i \in \interval{1}{k}$, we have
$T_i(w_{i-1}, w_i) \in I$. By Lemma \ref{lemma:homomorphism}, we have
$\directedge{T_i}(\delta(w_{i-1}), \delta(w_i)) \in J$. By the definition of
$\tau$, we have ${\delta(w_0) = \tau(v')}$ and ${\delta(w_k) = \tau(v)}$. Thus,
$\tuple{v', v} \in \sigma(E_t)$, as required by Definition \ref{def:skeleton}.

Finally, consider an edge $\tuple{v',v}\in\E$ and let $w_0,\ldots,w_k$ be terms
uniquely associated with this edge. Then  a role $R$ is \emph{compatible} with $\tuple{v,v'}$ if
$\directedge{R}(\delta(w_{i-1}), \delta(w_i)) \in J$ for each
$i\in\interval{1}{k}$, and $\TRANS{R}\not\in\T$ implies that $k =1$. In the
rest of this proof we will show the following property.
\begin{itemize}
    \item[($\lozenge$)] Each role $R \in L(v',v)$ is compatible with the edge
    $\tuple{v',v}$.
\end{itemize}
By the definition of function $\mathsf{exist}$ (Definition \ref{def:exist}) and
the above definition of compatibility, property $(\lozenge)$ implies that the
condition in step $17$ is not satisfied for edge $\tuple{v',v} \in \E$.

\smallskip

For the loop in step $6$; let $R(s,t)$ be an arbitrary aux-simple atom in
$\sigma(q_{\sim})$. By the definition of aux-simple atom, we have
$\tau(t)\in\aux{\dat_\K}$, role $R$ is simple, and $s \neq t$. By the
definition of $\sigma$, we have $\pi_*(s) \neq \pi_*(t)$. Thus,
$R(\pi_*(s),\pi_*(t))\in I$ is not a self-loop. Since $R$ is simple, by
property $1$ in Lemma \ref{lemma:univ}, $\pi_*(t)$ is of the form
$f_{T,A}^{A_1}(\pi_*(s))$. By Lemma \ref{lemma:homomorphism}, we have
$\directedge{R}(\delta(\pi_*(s)), \delta(\pi_*(t))) \in I$. Let $w_0,\ldots,
w_k$ be the terms associated with $\tuple{s,t}\in\E$. By the form of
$\pi_*(t)$ and since $w_0 = \pi_*(s)$ and $w_k = \pi_*(t)$, we have $k=1$.
Thus, property ($\lozenge$) is satisfied.

\smallskip

For the loop in steps $7$--$15$; let $R(s,t)$ be an arbitrary atom in
$\sigma(q_{\sim})$ that is neither good nor aux-simple. We next determine the
nondeterministic choices that preserve $(\lozenge)$ in step $15$, and that
satisfy the conditions in steps $8$, $9$, and $14$. By assumption, we have
$R(\pi_*(s),\pi_*(t)) \in I$. Since $R(s,t)$ is not good, we have
$\tau(t)\in\aux{\dat_{\K}}$, and either $s\neq t$ or $\SELF_{R}(\tau(t))\not
\in J$. By Lemma \ref{lemma:homomorphism} and by the definition of $\tau$, atom
$R(\pi_*(s),\pi_*(t)) \in I$ is not a self-loop.
Hence, a role $P$ and terms ${w_0, \ldots, w_m}$ with $w_m = \pi_*(t)$ exist in
$\domain{I}$ that satisfy property $2$ in Lemma \ref{lemma:univ}. Since
$P\subrole R$ and by properties {\it (2c)} and {\it (2d)}, we have $P(\pi_*(s),
\pi_*(t))\in I$. By Lemma \ref{lemma:homomorphism} and by the construction of
$\tau$, we have $P(\tau(s), \tau(t))\in J$; consequently, the conditions in
step $8$ are satisfied. Assume that $\TRANS{P} \not \in \T$. By property $2$
in Lemma \ref{lemma:univ}, we have $m =1$ and $\pi_*(v)\prec \pi_*(t)$. Thus,
$\pi_*(t)$ is of the form $f_{T,A}^B(\pi_*(s))$. By the definition of skeleton
$\S$, we have $\tuple{s,t}\in \E$; thus, the condition in step $9$ is not
satisfied. Next, we consider two cases.
\begin{itemize}
    \item $s$ reaches $t$ in \E. It follows that $\pi_*(s)\prec \pi_*(t)$. Let
    $v_0,\ldots, v_n$ be the unique path connecting $s$ to $t$ in \S. Since
    ${\pi_*(s)\prec \pi_*(t)}$, we have $w_0 = \pi_*(v_0)$ and $w_m = \pi_*(v_n)$. Thus,
    for each $i \in \interval{1}{n}$, a unique index $\ell_i$ exists such that
    ${\pi_*(v_i) = w_{\ell_i}}$. By property {\it (2d)} in Lemma \ref{lemma:univ}
    and by Lemma \ref{lemma:homomorphism}, role $P$ is compatible with
    $\tuple{v_{i-1}, v_i} \in \E$.

    \item $s$ does not reach $t$ in \E. Hence, $\pi_*(s) \not \prec \pi_*(t)$.
    Let $a_t$ be the root of $t$ in $\E$, and let $v_0,\ldots, v_n$ be the
    unique path connecting $a_t$ to $t$ in $\E$. Since $\pi_*(s)\not \prec
    \pi_*(t)$, we have $w_0 =a_t$, $w_0 \prec w_m$, and $w_m = \pi_*(v_n)$;
    moreover, $\TRANS{P}\in\T$, and $P(\pi_*(s), a_t) \in I$. Thus, for each $i
    \in \interval{1}{n}$, a unique index $\ell_i$ exists such that ${\pi_*(v_i)
    = w_{\ell_i}}$. By Lemma \ref{lemma:homomorphism}, we have $P(\tau(s), a_t)
    \in J$, so condition in step $17$ is not satisfied. Then, by property {\it
    (2d)} in Lemma \ref{lemma:univ} and by Lemma \ref{lemma:homomorphism}, role
    $P$ is compatible with ${\tuple{v_{i-1}, v_i} \in \E}$.\qedhere
\end{itemize}
\end{proof}

\section{Proof of Theorem \ref{th:complexity-ubound}}

Let $\K = \tuple{\T,\A}$ be a satisfiable \elhso KB and let $\dat_\K$ be the
datalog program for $\K$; furthermore, let $q = \exists \vec{y}. \psi(\vec{x},
\vec{y})$ be a CQ, and let $\tau$ be a candidate answer for $q$ and $\dat_\K$.
\ucomplexity*
\begin{proof}

    Please note that since the number of variables occurring in each rule in
    $\dat_{\K}$ is fixed, we can compute the set of all consequences of
    $\dat_\K$ in polynomial time~\cite{DBLP:journals/csur/DantsinEGV01}. Thus
    all the following operations can be implemented to run in polynomial time:
    \begin{itemize}

        \item computing sets $\aux{\dat_\K}$ and $\setind{\dat_{\K}}$,

        \item computing the connection graph $\mathsf{cg}$ for $q$ and $\tau$,

        \item given two individuals $u'$ and $u$ and a set of role $L$,
        checking whether $\mathsf{exist}(u',u,L)$ returns \true, and

        \item checking whether a binary atom is good or aux-simple w.r.t.\
        $\tau$.
    \end{itemize}

    Next, we argue that we can compute relation $\sim$ in polynomial time. As
    stated above, we can evaluate in polynomial time the precondition of the
    $(\mathsf{fork})$ rule. In addition, the size of relation $\sim$ is bounded
    by $|\terms{q}|^2$, the rules used to compute it are monotonic, and each
    inference can be applied in polynomial time, so  the claim follows.

    Also, we can check in linear time whether a directed graph is a acyclic by
    searching for a topological ordering of its
    vertices~\cite{DBLP:books/daglib/0023376}. Therefore, condition
    \ref{spur:cond2} in Definition \ref{def:dsound} can be checked in
    polynomial time. Therefore, steps $1$ and $2$ in Algorithm
    \ref{algo:issound} run in time polynomial in the input size, and property 2
    holds.

    Since steps $3$--$15$ in Algorithm \ref{algo:issound} can all clearly be
    implemented to run in nondeterministic polynomial time in the input size,
    property 1 also holds.

    For property 3, assume that the TBox $\T$ and the query $q$ are fixed. Then
    the set of all consequences of $\dat_\K$ can be computed in time polynomial
    in $\A$. Given that the number of variables occurring in $q$ is fixed, the
    number of guessing steps required in steps $3$ and $4$ is fixed; also, the
    number of alternatives for these steps is linear in the size of $\A$. Thus,
    steps $3$ and $4$ require polynomial time. Moreover, the maximum number of
    iterations of the for-loop in steps $7$–-$15$ is fixed. The number of
    alternatives for the guessing steps in line $8$ is fixed as well.
    Therefore, steps $7$–-$15$ require time polynomial in the size of $\A$. All
    other steps can clearly be implemented in time polynomial in the size of
    $\A$, thus $\mathsf{isSound}$ runs in time polynomial in the size of $\A$.
    \end{proof}

\section{Proof of Theorem \ref{th:complexity-lbound}}

\lbound*
\begin{proof}

    The proof is by reduction from the \np-hard problem of checking the
    satisfiability of a 3CNF formula~\cite{Garey:1979:CIG:578533}. Let $\varphi =
    \bigwedge_{j=1}^{m} C_j$ be a 3CNF formula over variables
    $\setof{v_1,\ldots, v_n}$, where each $C_j$ is a set of three literals $C_j
    = \setof{l_{j,1}, l_{j,2}, l_{j, 3}}$. A sequence $\nu = l_{j_1,
    k_1},\ldots, l_{j_\ell, k_\ell}$ of literals from $\varphi$ is
    \emph{consistent} if $\setof{v_i, \neg v_i}\not \subseteq \nu$ for each
    $i\in\interval{1}{n}$. Such a $\nu$ is a \emph{truth assignment} for $\varphi$
    if for each $j\in\interval{1}{m}$, a literal $l\in \nu$ exists such that $l
    = l_{j,k}$ for some $k\in\interval{1}{3}$. Then $\varphi$ is satisfiable if
    and only if there exists a consistent truth assignment for $\varphi$.

    Let $\varphi$ be a 3CNF formula. We next define an \elhso KB $\K_\varphi$
    and a Boolean CQ $q_\varphi$ such that $\K_\varphi$ does not contain axioms
    of type $2$, and $\K_\varphi \models q_\varphi$ if and only if there exists
    a consistent truth assignment for $\varphi$. Let $\dat_\varphi$ be the
    translation of $\K_\varphi$ into datalog. By Theorem \ref{th:correctness}
    and since $\K_\varphi$ does not contain axioms of type $2$, we have
    $\K_\varphi \models q_\varphi$ if and only if a candidate answer $\tau$ for
    $q_\varphi$ and $\dat_{\varphi}$ exists such that
    $\mathsf{isSound}(q_\varphi, \dat_{\varphi}, \tau)$ returns \true. Then, to
    prove the theorem, we show that there exists a unique candidate answer
    $\tau_\varphi$ for $q_\varphi$ and $\dat_{\varphi}$. Thus, $\K_\varphi
    \models q_\varphi$ if and only if $\mathsf{isSound}(q_\varphi,
    \dat_{\varphi}, \tau_\varphi)$ returns \true, hence,
    $\mathsf{isSound}(q_\varphi, \dat_{\varphi}, \tau_\varphi)$ returns \true
    if and only if $\varphi$ is satisfiable.

    For convenience, in the following we will specify $\K_\varphi$ using its
    equivalent formalisation as a rule base $\Xi_{\varphi}$. Moreover, rule
    base $\Xi_{\varphi}$ will not contain equality rules; consequently
    $\Xi_{\varphi}$ has a unique universal interpretation $I$. We will present
    our construction of $\Xi_{\varphi}$ in stages, and for each we will
    describe how it affects the universal interpretation $I$.

    Our encoding of $\Xi_{\varphi}$ uses a fresh individual $a$, fresh
    concepts $A$ and $G$, fresh roles $R$ and $T$, a fresh concept $L_{j,k}$
    and a fresh role $S_{j,k}$ uniquely associated to each literal $l_{j,k}$, a
    fresh concept $C_j$ uniquely associated to each clause $C_j$, and fresh
    roles $P_i$, $N_i$, and $T_i$ uniquely associated to each variable $v_i$.

    Before presenting $\Xi_{\varphi}$, we need a couple of definitions. Given
    two terms $w'$ and $w$ from $I$, and a word $\rho = T_1 \cdots T_\ell$ over
    \rolenames, we write $\rho(w', w)\in I$, if terms $w_0,\ldots,w_\ell$ with
    $w_0 = w'$ and $w_m =w$ exist such that $T_{i}(w_{i-1}, w_i)\in I$ for each
    $i\in\interval{1}{\ell}$. In the following, we uniquely associate to each
    sequence $\nu = l_{j_1, k_1},\ldots, l_{j_\ell, k_\ell}$ of literals from
    $\varphi$ the word ${\rho_\nu = R \cdot S_{j_1, k_1} \cdot R \cdot S_{j_2,
    k_2}\cdots R\cdot S_{j_\ell, k_\ell} \cdot R}$.

    We next present $\Xi_{\varphi}$ which consists of four parts.

    The first part of rule base $\Xi_{\varphi}$ contains atom \eqref{eq:hard:1} and
    rules \eqref{eq:hard:2}--\eqref{eq:hard:5}. Then for each sequence $\nu$ of
    literals from $\varphi$, a term $w_\nu$ exists in $I$ such that $\rho_\nu(a,
    w_\nu) \in I$ and $G(w_\nu)\in I$.
    \begin{align}
        \label{eq:hard:1}                        & A(a)                                                                      & \\
        \label{eq:hard:2} A(x)\rightarrow        & \exists z. R(x,z) \wedge C_j(z)                           & \forall j\in\interval{1}{m} \\
        \label{eq:hard:3} A(x)\rightarrow        & \exists z. R(x,z) \wedge G(z)                             &  \\
        \label{eq:hard:4} C_j(x)\rightarrow      & \exists z. S_{j,k}(x,z) \wedge L_{j,k}(z)                 & \forall j\in\interval{1}{m}\; \forall k\in\interval{1}{3} \\
        \label{eq:hard:5} L_{j,k}(x)\rightarrow  & A(x)                                                      & \forall j\in\interval{1}{m}\; \forall k\in\interval{1}{3}
    \end{align}

    The second part of rule base $\Xi_{\varphi}$ contains rules
    \eqref{eq:hard:6} and \eqref{eq:hard:7}. Consider an arbitrary literal
    $l_{j,k}$ and arbitrary terms $w'$ and $w$ in $I$ such that $S_{j,k}(w',
    w)\in I$. Then, for each variable $v_i$, we have $P_i(w', w)\in I$ if and
    only if sequence $v_i, l_{j,k}$ is consistent; and $N_i(w', w)\in I$ if and
    only if  sequence $\neg v_i, l_{j,k}$ is consistent.
    \begin{align}
        \label{eq:hard:6} S_{j,k}(x,y)   & \rightarrow P_i(x,y)                                             & \forall j\in\interval{1}{m},\forall k\in\interval{1}{3}, \forall i\in\interval{1}{n} \text{ with } l_{j,k} = v_i \text{ or $l_{j,k}$ is not over } v_i\\
        \label{eq:hard:7} S_{j,k}(x,y)   & \rightarrow N_i(x,y)                                             & \forall j\in\interval{1}{m},\forall k\in\interval{1}{3}, \forall i\in\interval{1}{n} \text{ with } l_{j,k} = \neg v_i \text{ or $l_{j,k}$ is not over } v_i
    \end{align}

    The third part of rule base $\Xi_{\varphi}$ contains rules
    \eqref{eq:hard:8}--\eqref{eq:hard:11}. Then for each sequence $\nu$ of
    literals from $\varphi$ with ${\rho_\nu(a, w_\nu) \in I}$ and each
    $i\in\interval{1}{n}$, we have ${P_i(a, w_\nu)\in I}$ if and only if ${\nu,
    v_i}$ is consistent; and ${N_i(a, w_\nu)\in I}$ if and only if ${\nu, \neg
    v_i}$ is consistent.
    \begin{align}
        \label{eq:hard:8} R(x,y)      & \rightarrow P_i(x,y)                                                 & \forall i\in\interval{1}{n}\\
        \label{eq:hard:9} R(x,y)      & \rightarrow N_i(x,y)                                                 & \forall i\in\interval{1}{n}\\
        \label{eq:hard:10} P_i(x,y)\wedge P_i(y,z)      & \rightarrow P_i(x,z)                                                 & \forall i\in\interval{1}{n}\\
        \label{eq:hard:11} N_i(x,y)\wedge N_i(y,z)      & \rightarrow N_i(x,y)                                                 & \forall i\in\interval{1}{n}
    \end{align}

    The fourth part of rule base $\Xi_{\varphi}$ contains rules
    \eqref{eq:hard:12}--\eqref{eq:hard:14}. Then for each sequence $\nu$ of literals
    from $\varphi$ with $\rho_\nu(a, w_\nu) \in I$ and each $i\in\interval{1}{n}$,
    we have $T_i(a, w_\nu)\in I$ if and only if $\setof{v_i, \neg
    v_i}\not\subseteq \nu$; furthermore, for each clause $C_j$, a term $w_{j}$
    exists such that $T(w_{j}, w_\nu) \in I$ if and only if an index
    $k\in\interval{1}{3}$ exists such that $l_{j,k} \in \nu$.
    \begin{align}
        \label{eq:hard:12} P_i(x,y)   & \rightarrow T_i(x,y)                                               & \forall i\in\interval{1}{n} \\
        \label{eq:hard:13} N_i(x,y)   & \rightarrow T_i(x,y)                                               & \forall i\in\interval{1}{n}\\
        \label{eq:hard:14} L_i(x,y)   & \rightarrow T(x,y)                                                  & \forall i\in\interval{1}{n}
    \end{align}
    Query $q_\varphi$ is given in \eqref{eq:hard:15}. Then $\K_\varphi \models
    q_\varphi$ if and only if a sequence $\nu$ of literals from $\varphi$
    exists such that $T_i(a, w_\rho) \in I$ for each $i\in\interval{1}{n}$,
    and, for each $j\in\interval{1}{m}$, a term $w_j$ exists such that
    $C_j(w_j) \in I$ and $T(w_j, w_{\rho})\in I$; hence, $\K_\varphi \models
    q_\varphi$ if and only if there exists a consistent truth assignment $\nu$
    for $\varphi$.
    \begin{align}
        \label{eq:hard:15} q_\varphi = \exists y \exists z_1,\ldots \exists z_{m}.\; G(y) \wedge \bigwedge_{i=1}^n T_i(a,y)  \wedge \bigwedge_{j=1}^{m} C_j(z_j) \wedge T(z_{j}, y)
    \end{align}

    Program $\dat_{\varphi}$ contains all the atoms and the rules in $\Xi_{\varphi}$
    apart from  rules \eqref{eq:hard:2}--\eqref{eq:hard:4} which
    are replaced by rules \eqref{eq:hard:16}--\eqref{eq:hard:18}.
    \begin{align}
        \label{eq:hard:16} A(x)        & \rightarrow  R(x,o_{R,C_j}) \wedge C_j(o_{R,C_j})                                              & \forall j\in\interval{1}{m} \\
        \label{eq:hard:17} A(x)        & \rightarrow  R(x,o_{R,G}) \wedge G(o_{R,G})                                                  &\\
        \label{eq:hard:18} C_j(x)      & \rightarrow S_{j,k}(x,o_{S_{j,k},  L_{j,k}}) \wedge L_{j,k}(o_{S_{j,k},  L_{j,k}})             & \forall j\in\interval{1}{m}, \forall k\in\interval{1}{3}
    \end{align}

    We next define substitution $\tau_\varphi$, and we show that
    $\dat_{\varphi} \models \tau_\varphi(q_\varphi)$. Substitution $\tau_\varphi$ is given in
    \eqref{eq:hard:19}.
    \begin{align}
        \label{eq:hard:19} \tau(y) = o_{R, G} \text{ and } \tau(z_j) = o_{R,C_j} \qquad \forall j\in\interval{1}{m}
    \end{align}
    Consider an arbitrary $i\in\interval{1}{n}$; we show that ${\dat_\varphi
    \models T_i(a, \tau(y))}$. By atom \eqref{eq:hard:1}, and by rule
    \eqref{eq:hard:17} we have ${\dat_\varphi \models R(a, \tau(y))}$. But
    then, by rules \eqref{eq:hard:8}, \eqref{eq:hard:9}, \eqref{eq:hard:12}
    \eqref{eq:hard:13}, we immediately have ${\dat_\varphi \models T_i(a,
    \tau(y))}$, as required. Next, consider an arbitrary $j\in\interval{1}{m}$;
    we show that ${\dat_\varphi \models T(\tau(z_j), \tau(y))}$. Please note
    that by atom \eqref{eq:hard:1} and by rules \eqref{eq:hard:16}, we have
    ${\dat_\varphi \models C_j(\tau(z_j))}$. Consider an arbitrary
    literal $l_{j,k}$ in $C_j$. By rules \eqref{eq:hard:18} and
    \eqref{eq:hard:5}, there exists an individual $u$ such that $\dat_\varphi
    \models S_{j,k}(\tau(z_j),u)$ and $\dat_\varphi \models A(u)$. Then, by
    rule \eqref{eq:hard:17} we have ${\dat_\varphi \models R(u, \tau(y))}$. By
    rules \eqref{eq:hard:8}--\eqref{eq:hard:14}, we have ${\dat_\varphi \models
    T(\tau(z_j), \tau(y))}$, as required.

    We are left to show that $\tau_\varphi$ is unique. Consider an arbitrary
    candidate answer $\xi$ for $q_\varphi$ and $\dat_\varphi$. Since $G(y)$ is
    an atom in $q$ and only rule \eqref{eq:hard:17} derives assertions over
    $G$, we must have $\xi(y) = o_{R,G}$. Consider an arbitrary
    $j\in\interval{1}{m}$. Due to atom $C_j(z_j)$ in $q$, and only rule
    \eqref{eq:hard:16} derives assertions over concept $C_j$, we must have
    $\xi(y) = o_{R,C_j}$. Thus, $\xi = \tau_\varphi$, as required.
\end{proof}

\section{Proof of Theorem \ref{th:arborescent}}

Let $\K$ be a satisfiable $\elho$ KB, and let $\Xi_\K$ and $\dat_\K$ be the
rule base and the datalog program associated with $\K$, respectively;
furthermore, let $\setind{\dat_\K}$ and $\aux{\dat_\K}$ be as specified in
Definition \ref{def:dat-order}, and let $q$ be an arborescent query. We next
show that function $\mathsf{entails}(\dat_\K, q)$ returns \true if and only if
$\Xi_\K\models q$, and that $\mathsf{entails}(\dat_\K, q)$ runs in time
polynomial in the input size.

To this end, we start with a couple of definitions. Let $I$ and $J$ be
universal interpretations for $\Xi_\K$ and $\dat_\K$, respectively. Moreover,
let $\delta$ be the mapping as specified at the beginning of Section
\ref{sec:proofsat}. For each set $V\in\mathsf{RT}$, let $q_V$ be the
arborescent query obtained from $q$ by replacing each variable $y\in V$ with a
fresh variable $y_V$, and by removing each variable $z$, and all the atoms
involving $z$, in the resulting query such that $z \neq y_V$ and $z$ is not a
descendant of $y_V$ in $\dgraph{q}$.

\begin{lemma}
    Function $\mathsf{entails}(\dat_\K, q)$ returns \true if and only if $\Xi_\K\models q$.
    Furthermore, $\mathsf{entails}(\dat_\K, q)$ runs in time polynomial in the input size.
\end{lemma}
\begin{proof}

    Since $q$ does not contain individuals, we have $\Xi_\K\models q$ if and
    only if a substitution $\pi$ with ${\dom{\pi} = \vars{q}}$ exists such that
    ${\pi(q) \subseteq I}$. In the following, we show that the latter is the
    case if and only if function $\mathsf{entails}(\dat_\K, q)$ returns \true .

    Let $\mathsf{RT}$ be as specified in Definition \ref{def:entails}. Let $M$
    be the largest $n\in\nat$ for which a set $V\in \mathsf{RT}$ exists whose
    level is $n$. By the definition, $M$ is the length of the longest path in
    $\dgraph{q}$, and so it is linearly bounded by the size of $q$.

    Next, we argue that we can compute set $\mathsf{RT}$ in polynomial time.
    This follows from the fact that the rules used to compute it are monotonic,
    each rule can be applied at most $M$ times, and each rule application
    introduces at most $|q|$ sets in $\mathsf{RT}$, and the size of each set is
    linearly bounded by $|q|$.

    By Lemmas \ref{lemma:homomorphism} and \ref{lemma:dat-embed}, and since
    each rule in $\dat_\K$ contains a fixed number of variables, each set
    $V\in\mathsf{RT}$ satisfies the two following properties for each
    $u\in\setind{\dat_\K}\cup \aux{\dat_\K}$ and each term $w\in\domain{I}$
    with $\delta(w) =u$.
    \begin{enumerate}[\it A.]
        \item $u \in \mathsf{c}_V$ if and only if $B(w) \in I$ for
        each concept $B$ such that $B(y)\in q$ for some $y\in V$.
        \item $\mathsf{c}_V$ can be computed in time polynomial in the size of
        $\dat_\K$ and $q$.
    \end{enumerate}

    To prove the lemma, we next show that each set $V\in \mathsf{RT}$ satisfies the
    following properties for each term $u\in\setind{\dat_\K}\cup\aux{\dat_\K}$.
    \begin{enumerate}
        \item $u \in \mathsf{A}_V$ if and only if a substitution $\pi$ with
        $\dom{\pi}=\vars{q_V}$ exists such that $\pi(q_V)\subseteq I$ and
        $\delta(\pi(y_V)) = u$.
        \item $\mathsf{A}_V$ can be computed in time polynomial in the size
        of $\dat_\K$ and $q$.
    \end{enumerate}

    The proof goes by reverse-induction on the level of sets in $\mathsf{RT}$.

    \smallskip

    \basecase Consider an arbitrary set $V\in\mathsf{RT}$ of level $M$. By the
    definition of $\mathsf{A}_V$, we have $\mathsf{A}_V = \mathsf{c}_V$.
    Furthermore, an atom $B(y_V)$ is in $q_V$ if and only if a variable $y\in
    V$ exists such that $B(y)$ in $q$. Properties $1$ and $2$ follow
    immediately from A and B.

    \smallskip

    \indstep Consider an arbitrary $n\in\nat$. Let $V\in\mathsf{RT}$ be an
    arbitrary set of level $n$, and assume that properties $1$ and $2$ hold for
    each set ${W\in \mathsf{RT}}$ of level $n+1$; we show that the properties
    are satisfied by $V$.

    For property $1$, let $u$ be an arbitrary term in $\setind{\dat_\K} \cup
    \aux{\dat_\K}$. By the definition of $\mathsf{A}_V$ we have ${\mathsf{A}_V=
    \mathsf{c}_V \cap (\mathsf{i}_V \cup \mathsf{a}_V)}$. By the definition of
    $q_V$ and from property A, it suffices to show that $u\in (\mathsf{i}_V
    \cup \mathsf{a}_V)$ if and only if a substitution $\pi$ exists such that
    $\pi(q_V)\subseteq I$ and $\delta(\pi(y_V)) = u$. We consider the two
    directions of $1$ separately.

    \smallskip
    $(\Rightarrow)$ Assume that $u \in (\mathsf{i}_V \cup \mathsf{a}_V)$.
    We distinguish two cases.
    \begin{itemize}
        \item $u \in\mathsf{i}_V$. Thus, $u\in\setind{\dat_\K}$ and
        $u\in\indnames$. Consider an arbitrary variable $y\in\pred{V}$ and an
        arbitrary role ${R\in\roles{y}}$. Then, an individual
        ${u'\in\setind{\dat_\K} \cup \aux{\dat_\K}}$ exists such that ${u'\in
        \mathsf{A}_{\setof{y}}}$ and ${R(u', u)\in J}$. By the inductive
        hypothesis, a substitution $\pi_{\setof{y}}$ exists such that
        ${\pi_{\setof{y}}(q_{\setof{y}})\subseteq I}$ and $\delta(w') = u'$,
        where ${w' = \pi_{\setof{y}}(y_{\setof{y}})}$. By Lemma
        \ref{lemma:dat-embed}, we have ${R(w', u) \in I}$. Then let $\pi$ be
        the substitution such that ${\pi(y_V) = u}$ and
        ${\pi_{\setof{y}}\subseteq \pi}$ for each ${y \in \pred{V}}$. Then $\pi(q_V)\subseteq I$, as required.

        \item $u \in\mathsf{a}_V$. Thus, $u\in\aux{\dat_\K}$ and $u$ is of the
        form $o_{P,A}$. Consider an arbitrary role $R$ such that
        $R\in\roles{y}$ for some $y\in \pred{V}$. Then, an individual
        $u'\in\setind{\dat_\K} \cup \aux{\dat_\K}$ exists such that $u'\in
        \mathsf{A}_{\pred{V}}$ and $\directedge{R}(u', u)\in J$. By the
        inductive hypothesis, a substitution $\pi_{\pred{V}}$ exists such that
        $\pi_{\pred{V}}(q_{\pred{V}})\subseteq I$ and $\delta(w') =u'$, where
        $w'=\pi_{\pred{V}}(y_{\pred{V}})$. By Lemma \ref{lemma:dat-embed}, a
        term $w\in\domain{I}$ exists such that $\delta(w) = o_{P,A}$ and $P(w',
        w) \in I$ and $P\subrole R$. Since no rule is applicable to $I$, we
        have $R(w', w) \in I$. Then let $\pi$ be the substitution such that
        $\pi(y_V) = w$, $\pi_{\pred{V}}\subseteq \pi$, and $\pi(y) = w'$ for
        each $y\in\pred{V}$. Then $\pi(q_V)\subseteq I$, as required.
    \end{itemize}

    \smallskip
    $(\Leftarrow)$ Assume that a substitution $\pi$ exists such that
    $\pi(q_V)\subseteq I$ and $\delta(\pi(y_V)) = u$. We distinguish two cases.
    \begin{itemize}
        \item $\pi(y_V)\in\indnames$. By Lemma \ref{lemma:equality}, we have $u
        \in\setind{\dat_\K}$ and $u\in\indnames$. Consider an arbitrary
        variable $y\in\pred{V}$ and an arbitrary role $R\in\roles{y}$. By the
        definition of $q_V$, atom $R(y, y_V)$ occurs in $q_V$. Since $R(\pi(y),
        \pi(y_V)) \in I$, by Lemma \ref{lemma:homomorphism}, we have
        $R(\delta(\pi(y)), \delta(\pi(y_V)))\in J$ and
        ${\delta(\pi(y))\in\setind{\dat_\K}\cup \aux{\dat_\K}}$. Due to the
        construction of $q_V$, $q_{\setof{y}}$ is a subquery of $q_V$; thus,
        $\pi(q_{\setof{y}}) \subseteq I$. Hence, by the inductive hypothesis,
        we have $\delta(\pi(y))\in \mathsf{A}_{\setof{y}}$. Due to $\delta(\pi(y_V)) =u$
        and $u\in\setind{\dat_\K}$, we have $u\in \mathsf{i}_V$, as required.

        \item $\pi(y_V)\not\in\indnames$. By Lemma \ref{lemma:equality}, we
        have $u \in\aux{\dat_\K}$. Consider an arbitrary variable $y\in
        \pred{V}$ and an arbitrary role $R\in\roles{y}$. By the definition of
        $q_V$, atom $R(y, y_V)$ occurs in $q_V$. By assumption, we have
        $R(\pi(y), \pi(y_V))\in I$. Since $\K$ is an \elho knowledge base, role
        $R$ is simple and $\SELF_R(\pi(y_V))\not \in I$. By Lemma
        \ref{lemma:univ}, term $\pi(y_V)$ is of the form $f_{P,A}^B(\pi(y))$.
        Therefore, for each variable $z\in\pred{V}$, we have $\pi(z)=\pi(y)$.
        Furthermore, by Lemma \ref{lemma:homomorphism}, we have
        $\directedge{R}(\delta(\pi(y)), \delta(\pi(y_V)))\in J$ and
        ${\delta(\pi(y))\in\setind{\dat_\K}\cup \aux{\dat_\K}}$. Let
        substitution $\pi_{\pred{V}}$ for $q_{\pred{V}}$ be obtained from $\pi$
        by setting $\pi_{\pred{V}}(y_{\pred{V}}) = \pi(y)$. Then, we have
        $\pi_{\pred{V}}(q_{\pred{V}}) \subseteq I$. Hence, by the inductive
        hypothesis, we have $\delta(\pi(y))\in \mathsf{A}_{\pred{V}}$. Due to
        $\delta(\pi(y_V)) =u$ and $u\in\aux{\dat_\K}$, we have $u\in
        \mathsf{a}_V$, as required.
     \end{itemize}

     For property $2$, we show that $\mathsf{A}_V$ can be computed in time
     polynomial in the size of $\dat_\K$ and $q$. By property B, set
     $\mathsf{c}_V$ can be computed in time polynomial in $q$ and $\dat_\K$,
     hence in the rest of this proof we focus on  sets $\mathsf{i}_V$ and $\mathsf{a}_V$.
     \begin{itemize}
         \item $\mathsf{i}_V$. Please note that the number of variables in
         $\pred{V}$ is bounded by the size of $q$. Then consider an arbitrary
         variable $y\in \pred{V}$. By the inductive hypothesis, set
         $\mathsf{A}_{\setof{y}}$ can be computed in time polynomial in $q$ and
         $\dat_\K$. Then, since each rule in $\dat_\K$ contains a constant
         number of variables, $\mathsf{i}_V$ can also be computed in time
         polynomial in the size of $q$ and $\dat_\K$.

         \item $\mathsf{a}_V$. By the inductive hypothesis, set
         $\mathsf{A}_{\pred{V}}$ can be computed in time polynomial in the size
         of $q$ and $\dat_\K$. Then, since each rule in $\dat_\K$ contains a
         constant number of variables, $\mathsf{a}_V$ can also be computed in time
         polynomial in the size of $q$ and $\dat_\K$.\qedhere
     \end{itemize}
 \end{proof}

\section{Proof of Theorem \ref{th:lowerbound-acyclic}}

In this section, we prove the lower bounds established in Theorem
\ref{th:lowerbound-acyclic}, all of which are proved by reducing the \np-hard
problem of checking the satisfiability of a CNF formula
$\psi$~\cite{Garey:1979:CIG:578533}. For the rest of this section, we fix a CNF
formula $\psi = \bigwedge_{j=1}^{m} C_j$ where each $C_j$ is a clause over
variables $\setof{v_1,\ldots, v_n}$.

\medskip

For convenience, we will assume that \elho TBoxes can contain axioms of the
form $A_1 \ISA \SOME{R}{\setof{a}}$ with $A_1\in\conceptnames\cup\setof{\top}$,
$R$ a role, and $\setof{a}$ a nominal. The translation into rules for this type
of axiom is given by $A_1(x) \rightarrow R(x,a)$. Moreover, in the following we
will specify DL knowledge bases using their equivalent formalisation as rule
bases. Furthermore, all the rules from this section do not contain equality,
and so each rule base has exactly one universal interpretation.

Then to prove the theorem, we will first define a rule base $\Xi_0$
containing only rules of types $1$ and $7$ from Table \ref{table:Xi}, and a
Boolean CQ $q_0$ over $\Xi_0$, after which we will show that
\begin{enumerate}
    \item a rule base $\Xi_1$ and a query $q_1$ exist such that $\Xi_1$ is in
    \elho, query $q_0 \wedge q_1$ is acyclic, and $\psi$ is
    satisfiable if and only if $\Xi_0 \cup \Xi_1 \models q_0
    \wedge q_1$,

    \item a rule base $\Xi_2$ and a query $q_2$ exist such that $\Xi_2$
    contains a single rule of type $8$, query $q_0 \wedge q_2$ is
    arborescent, and $\psi$ is satisfiable if and only if $\Xi_0 \cup \Xi_2
    \models q_0 \wedge q_2$, and

    \item a rule base $\Xi_3$ and a query $q_3$ exist such that $\Xi_3$
    contains a single rule of type $9$, query $q_0 \wedge q_3$ is
    arborescent, and $\psi$ is satisfiable if and only if $\Xi_0 \cup \Xi_3
    \models q_0 \wedge q_3$.
\end{enumerate}

\subsection{Construction of $\Xi_0$ and $q_0$} Our encoding uses a fresh role
$R$, a fresh concepts $G$, fresh concepts $T_i$, $F_i$, and $A_i$ uniquely
associated with each variable $v_i$ in $\psi$, and fresh concepts $C_j$
uniquely associated to each clause $C_j$ in $\psi$. Rule base $\Xi_0$ contains
a ground atom \eqref{eq:xi:psi:atom}, and rules
\eqref{eq:xi:psi:rule1}--\eqref{eq:xi:psi:rule8}. Let $I_0$ be the universal
interpretation of $\Xi_0$, we next describe how the rules in $\Xi_0$ model the
structure of $I_0$. Atom \eqref{eq:xi:psi:atom} and rules
\eqref{eq:xi:psi:rule1}--\eqref{eq:xi:psi:rule6} encode a binary tree of depth
$n+1$ in $I_0$ rooted in individual $a$ in which edges are labelled by role
$R$, each leaf node satisfies concept $G$, and each node $w$ at depth $1\leq
i\leq n$ satisfies exactly one of $T_i$ and $F_i$. Then node $w$ represents a
positive truth assignment to $v_i$, if $T_i(w)\in I_0$; otherwise, it
represents a negative truth assignment to $v_i$. Consequently, a path from $a$
to a leaf node in the tree represents a truth assignment to the variables in
$\psi$. Finally, rules \eqref{eq:xi:psi:rule7} and \eqref{eq:xi:psi:rule8}
ensure that, for each node $w$ at depth $1\leq i\leq n$ in the tree, if the
truth assignment to $v_i$ represented by $w$ makes clause $C_j$ evaluate to
true, then $C_j(w)\in I_0$.
\begin{align}
                & A_0(a)                                          &                                  \label{eq:xi:psi:atom}               \\
    A_{i-1}(x) \rightarrow & \exists z. R(x,z)\wedge T_i(z)  & \forall i\in\interval{1}{n}      \label{eq:xi:psi:rule1}\\
    A_{i-1}(x) \rightarrow & \exists z. R(x,z)\wedge F_i(z)  & \forall i\in\interval{1}{n}      \label{eq:xi:psi:rule2}\\
    T_{i}(x)   \rightarrow & A_i(x)                          & \forall i\in\interval{1}{n}      \label{eq:xi:psi:rule3}\\
    F_{i}(x)   \rightarrow & A_i(x)                          & \forall i\in\interval{1}{n}      \label{eq:xi:psi:rule4}\\
    A_{n}(x)   \rightarrow & \exists z. R(x,z)\wedge G(z)    &                                  \label{eq:xi:psi:rule6}\\
    T_i(x)     \rightarrow & C_j(x)                          & \forall i\in\interval{1}{n}\; \forall j\in\interval{1}{m} \text{ with } v_i\in C_j    \label{eq:xi:psi:rule7}\\
    F_i(x)     \rightarrow & C_j(x)                          & \forall i\in\interval{1}{n}\; \forall j\in\interval{1}{m} \text{ with } \neg v_i\in C_j  \label{eq:xi:psi:rule8}
\end{align}
Query $q_0$ is given in \eqref{eq:q:1:rule1}. It should be clear that, for
each substitution $\pi$ with $\dom{\pi} = \vec{p}$ such that $\pi(q_0)\subseteq
I_0$, substitution $\pi$ represents a truth assignment for $\psi$, $\pi(p_0) =
a$, and $\pi(p_{n+1})$ is a leaf.
\begin{align}
    q_0 = \exists p_0\ldots\exists p_{n+1}.\; \bigwedge_{i=0}^{n}[A_{i}(p_{i})\wedge R(p_{i}, p_{i+1})] \wedge G(p_{n+1}) \label{eq:q:1:rule1}
\end{align}

\subsection{Construction of $\Xi_{1}$ and $q_1$} Our rule base $\Xi_2$ uses
fresh roles $S_j$ and individuals $c_j$ uniquely associated to each clause
$C_j$ of $\psi$. Then rule base $\Xi_1$ contains atoms \eqref{eq:xi:1:atom},
and rules \eqref{eq:xi:1:rule1} and \eqref{eq:xi:1:rule2}. Let $I_1$ be the
universal interpretation of $\Xi_0\cup \Xi_1$, then we clearly have
$I_0\subseteq I_1$. We next describe how the rules and atoms in $\Xi_1$ modify
the tree encoded by $\Xi_0$. Rule \eqref{eq:xi:1:rule1} ensure that, for each
node $w$ in the tree whose truth assignment makes clause $C_j$ evaluate to
true, there exists an edge labelled by $S_j$ from $w$ to individual $c_j$. Atom
\eqref{eq:xi:1:atom} generates an edge labelled by $R$ connecting $c_j$ to
itself. Finally, rule \eqref{eq:xi:1:rule2} labels each edge in the tree and
each looping edge with $S_j$.
\begin{align}
                        & R(c_j, c_j)                   & \forall j\in\interval{1}{m}    \label{eq:xi:1:atom}\\
    C_j(x)\rightarrow   &S_j(x, c_j)     & \forall j\in\interval{1}{m}\label{eq:xi:1:rule1}\\
    R(x,y) \rightarrow  &    S_j(x,y)        & \forall j\in\interval{1}{m}\label{eq:xi:1:rule2}
\end{align}
\begin{figure}
    \centering
    \includegraphics[scale=0.6]{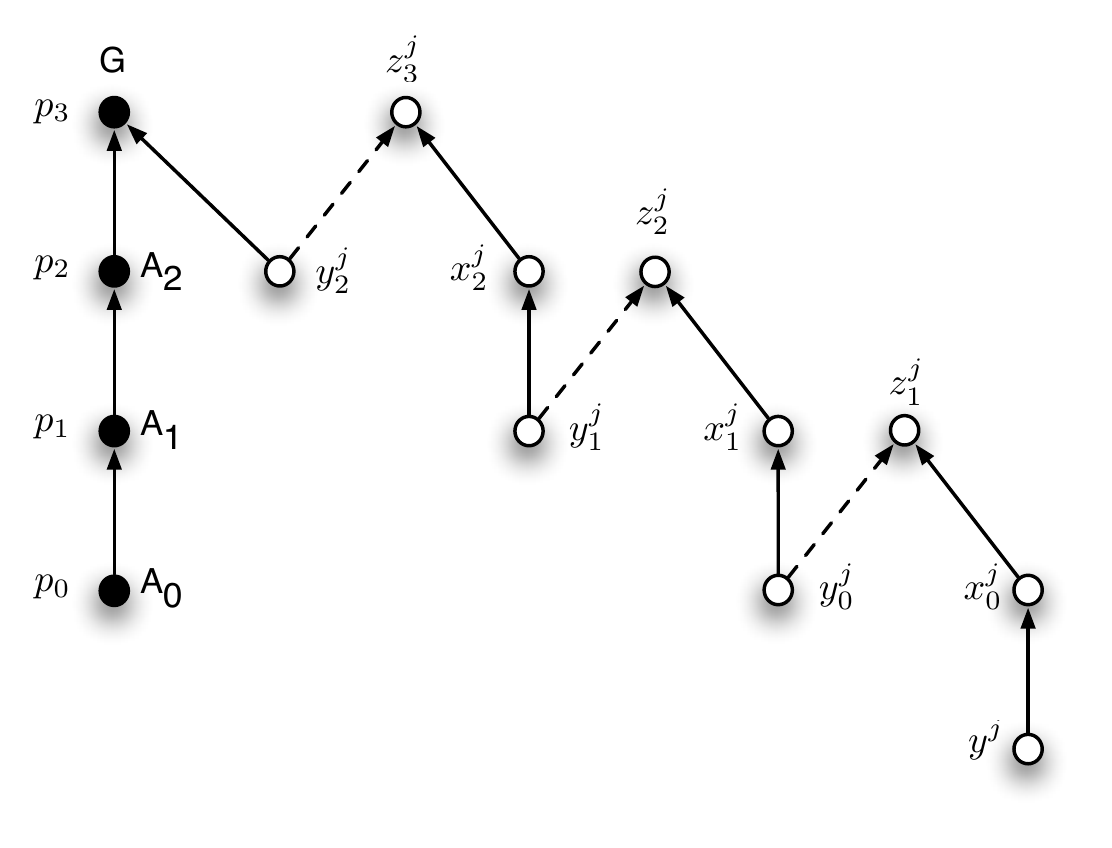}
    \caption{A graphical representation of query $q_0\wedge q_1$ for $n=2$ and an arbitrary $j\in\interval{1}{m}$}\label{fig:query01}
\end{figure}
Our encoding of query $q_{1}$ uses variable $p_{n+1}$ from query $q_0$, as well
as fresh variables $x_0^j,\ldots, x_{n}^j$, $y^j, y_0^j,\ldots, y_{n}^j$, and
$z_1^j,\ldots, z_{n+1}^j$ uniquely associated with each clause $C_j$. Next, we
associate with each $C_j$ a conjunction $\varphi^j$ as shown in \eqref{eq:q:1:rule4}.
Then query $q_1$ is the existential closure of ${\bigwedge_j\varphi^{j}}$.
\begin{align}
    \varphi^j & = R(y_j, x_0^j) \wedge \bigwedge_{i=1}^{n} [ \varphi_i^j \wedge R(y_{i-1}^j, x_i^j)]\wedge \varphi_{n+1}^j\wedge R(y_n^j, p_{n+1}) \label{eq:q:1:rule4}\\
    \varphi^j_i & = R(x_{i-1}^j, z_i^j)\wedge S_j(y_{i-1}^j, z_i^j)  \label{eq:q:1:rule5}
\end{align}
Figure \ref{fig:query01} shows a graphical representation of query $q_0\wedge
q_1$ for $n= 2$ and an arbitrary $j\in\interval{1}{m}$. The black edges denote
atoms over role $R$, and the dashed edges denote atoms over role $S_j$. It
should be clear that query $q_0\wedge q_1$ is acyclic, as required. Please note
that $q_0\wedge q_1$ is not arborescent: each variable $y_i^j$ in $q_0\wedge
q_1$ has two parent nodes.

Consider an arbitrary substitution $\pi$ such that
$\pi(q_0)\subseteq I_1$. Intuitively, each conjunct
$\varphi^j$ in $q_1$ ensures that the truth assignment represented by $\pi$
makes clause $C_j$ true. We next prove that our encoding is correct.
\begin{lemma}
    $\psi$ is satisfiable if and only if $\Xi_0 \cup \Xi_1\models
    q_0 \wedge q_1$.
\end{lemma}
\begin{proof}

    $(\Rightarrow)$ Assume that a truth assignment $\nu : \setof{v_1,\ldots,
    v_n} \mapsto \setof{\true,\false}$ exists such that $\nu(\psi) = \true$; we
    next show that a substitution $\pi$ with $\dom{\pi} = \vars{q_0 \wedge
    q_1}$ exists such that $\pi(q_0 \wedge q_1)\subseteq I_1$. To this end, let
    $\pi(p_0) = a$ and, for each $i\in\interval{1}{n}$, let $\pi(p_i)$ be the
    unique successor of $\pi(p_{i-1})$ in $I_1$ such that $T_i(\pi(p_i))\in
    I_1$ if and only if ${\nu(v_i) = \true}$. Consider an arbitrary clause
    $C_j$. Since $\nu(\psi) = \true$, a literal $l \in C_j$ exists such that
    $\nu(l) = \true$. Let $k \in \interval{1}{n}$ be the unique index such that
    literal $l$ is over variable $v_k$. Then, formulas $\Psi^j_{-}$ and
    $\Psi^j_{+}$ exist such that $\varphi^j$ is of the form $\varphi^j =
    \Psi^j_{-} \wedge \varphi_{k+1}^j \wedge \Psi^j_{+}$. Next, we extend
    substitution $\pi$ as follows:
    \begin{itemize}
        \item $\pi$ maps each variable occurring in $\Psi^j_{-}$ to $c_j$,
        \item $\pi(x_{k}^j) = c_j$, $\pi(z_{k+1}^j) = c_j$, and $\pi(y_{k}^j) = \pi(p_{k})$, and
        \item for each $\ell\in\interval{k+1}{n+1}$, $\pi$ maps each variable
        with subscript $\ell$ occurring in $\Psi^j_{+}$ to $\pi(p_k)$.
    \end{itemize}
    Please note that by rules \eqref{eq:xi:1:rule1} and \eqref{eq:xi:1:rule2},
    and by atoms \eqref{eq:xi:1:atom}, we have ${\pi(\Psi^j_{-}) \cup
    \pi(\varphi_{k+1}^j) \cup \pi(\Psi^j_{+}) \subseteq I_1}$.

    \medskip

    $(\Leftarrow)$ Assume that a substitution $\pi$ with ${\dom{\pi} =
    \vars{q_0\wedge q_1}}$ exists such that
    $\pi(q_0\wedge q_1)\subseteq I_1$; we next show that a truth
    assignment ${\nu:\setof{v_1,\ldots, v_n}\mapsto \setof{\true,\false}}$
    exists such that $\nu(\psi) = \true$. Please note that $\pi(p_0) = a$ and,
    for each $i\in\interval{1}{n}$, we have $\pi(p_i)$ is a successor of
    $\pi(p_{i-1})$ in the binary tree encoded by $\Xi_0$; moreover, for
    each $j\in\interval{1}{m}$, due to atom $R(y_n^j, p_{n+1})$ in
    \eqref{eq:q:1:rule4}, we have $\pi(y_n^j)$ is the unique parent of
    $\pi(p_{n+1})$ in the tree, and so $\pi(y_n^j)= \pi(p_n)$. Then, let $\nu$
    be the truth assignment such that, for each $i\in\interval{1}{n}$, we have
    $\nu(v_i) = \true$ if and only if $T_i(\pi(p_i))\in I_1$. We next show that
    $\nu(\psi)=\true$; that is, for each clause $C_j$, we have
    $\nu(C_j)=\true$. Consider an arbitrary clause $C_j$. By rules \eqref{eq:xi:psi:rule6} and
    \eqref{eq:xi:psi:rule7}, it suffices to find an index $\ell
    \in\interval{1}{n}$ such that $C_j(\pi(p_\ell))\in I_1$. Towards this goal,
    we first show that substitution $\pi$ satisfies the following property for
    each $i\in\interval{0}{n}$.
    \begin{itemize}
        \item[$(\diamondsuit)$] If for each $\ell \in \interval{i}{n}$ we
        have $\pi(z_{\ell+1}^j)\not\in\indnames$, then $\pi(y_i^j)= \pi(p_i)$
        and $\pi(x_i^j)= \pi(p_i)$.
    \end{itemize}
    We proceed by reverse induction on $i\in\interval{0}{n}$.

    \basecase Let $i= n$. Assume that $\pi(z_{n+1}^j)\not\in\indnames$. As
    stated above, we have $\pi(y_n^j) = \pi(p_{n})$. Due to atom $S_j(y_n^j,
    z_{n+1}^j)$ in $\varphi_{n+1}^j$, we have $\pi(z_{n+1}^j)$ is a successor
    of $\pi(p_n)$. Due to atom $R(x_n^j, z_{n+1}^j)$ in $\varphi_{n+1}^j$,
    we have $\pi(x_n^j)$ is the parent of $\pi(z_{n+1}^j)$, thus $\pi(x_n^j)=
    \pi(p_n)$, as required.

    \indstep Consider an arbitrary $i\in\interval{0}{n-1}$. Assume that the
    property holds for $i+1$, and that $\pi(z_{\ell+1}^j)\not\in\indnames$ for
    each $\ell \in \interval{i}{n}$. By the inductive hypothesis, we have
    $\pi(y_{i+1}^j)= \pi(p_{i+1})$ and $\pi(x_{i+1}^j)= \pi(p_{i+1})$. Due to
    atom $R(y_i^j, x_{i+1}^j)$ in \eqref{eq:q:1:rule4}, we have $\pi(y_i^j)$ is
    the unique parent of $\pi(p_{i+1})$ in the tree, and so $\pi(y_i^j) =
    \pi(p_i)$. Due to atom $S_j(y_{i}^j, z_{i+1}^j)$ in $\varphi_{i+1}^j$, we
    have $\pi(z_{i+1}^j)$ is a successor of $\pi(p_i)$. Finally, due to atom
    $R(x_{i}^j, z_{i+1}^j)$ in $\varphi_{i+1}^j$, we have $\pi(x_i^j) =
    \pi(p_i)$, as required.

    \smallskip

    We next show that an index $i\in\interval{0}{n}$ exists such that
    $\pi(z_{i+1}^j)$ is mapped to an individual. Assume the opposite, hence,
    for each $i \in \interval{0}{n}$, we have $\pi(z_{i+1}^j)\not\in\indnames$.
    By property $(\diamondsuit)$, we then have $\pi(x_0^j) = \pi(p_0)$. Since
    atom $S(y^j, x_0^j)$ occurs in $q_1$, $\pi(x_0^j) = \pi(p_0) = a$, and $a$
    does not have incoming edges in $I_1$, this is a contradiction.

    \smallskip Finally, let $\ell\in\interval{0}{n}$ be the largest index such
    that $\pi(z_{\ell+1}^j)\in\indnames$. We show that $\pi(y_{\ell}^j) =
    \pi(p_\ell)$ by considering two cases.
    \begin{itemize}
        \item $\ell = n$. As stated above, we have $\pi(y_n^j) = \pi(p_n)$.
        \item $\ell < n$. By property $(\diamondsuit)$, we have
    $\pi(x_{\ell+1}^j) = \pi(p_{\ell+1})$. Due to atom $R(y_\ell^j,
    x^j_{\ell+1})$, we have $\pi(y_{\ell}^j) = \pi(p_\ell)$.
    \end{itemize}
    In either cases, we have $\pi(y_{\ell}^j) = \pi(p_\ell)$.

    We next show that $\ell >0$. Assume the opposite, hence $\ell = 0$. Since
    $\pi$ maps $\pi(z^j_{1})\in\indnames$ and $\pi(y_0^j) = a$, atom
    ${S_j(y_0^j, z^j_{1})}$ occurs in $q_1$, and $a$ is not connected to an
    individual in $I_1$, we have ${S_j(\pi(y_0^j), \pi(z^j_{1}))\not\in I_1}$,
    which is a contradiction.

    Therefore, we have $\pi(y_{\ell}^j) = \pi(p_\ell)$ and $\ell >1$. By rules
    \eqref{eq:xi:1:rule1} and \eqref{eq:xi:1:rule2}, and since atom
    $S_j(y_\ell^j, z_{\ell+1}^j)$ occurs in $q_1$, we have $\pi(z_{\ell+1}^j) =
    c_j$ and $C_j(\pi(p_\ell))\in I_1$, as required.
\end{proof}

\subsection{Construction of $\Xi_{2}$ and $q_2$} Rule base $\Xi_2$ consists
only of rule \eqref{eq:xi:2:rule1}. Let $I_2$ be the universal interpretation
of $\Xi_0\cup \Xi_2$, then we clearly have $I_0\subseteq I_2$. Note that rule
\eqref{eq:xi:2:rule1} modifies the tree encoded by $\Xi_0$ by connecting each
node $w$ in the tree via an $R$ edge to all nodes that occur on the path
connecting $w$ to the root $a$.
\begin{align}
    R(x,y)\wedge R(y,z)\rightarrow R(x,z)\label{eq:xi:2:rule1}
\end{align}
Our encoding of query $q_{2}$ uses variable $p_{n+1}$ from query $q_0$, as well
as a fresh variable $x_j$ uniquely associated to each clause $C_j$. Query
$q_{2}$ is given in \eqref{eq:q:2:rule1}.
\begin{align}
    \exists x_1\ldots \exists x_m \bigwedge_{i=1}^{m}C_j(x_j) \wedge R(x_j, p_{n+1}) \label{eq:q:2:rule1}
\end{align}
It should be clear that $q_0 \wedge q_2$ is arborescent. Now, consider a
substitution $\pi$ such that $\pi(q_0\wedge q_2)\subseteq I_2$. By the
definition of $q_2$, for each $j\in\interval{1}{m}$, a term $w_j$ exists such
that $w_j$ occurs on the unique path connecting leaf $\pi(p_{n+1})$ to the root
of the tree $a$ and $C_j(w_j) \in I_2$; therefore, an index
$\ell_j\in\interval{1}{n}$ exists such that $\pi(p_{\ell_j}) = \pi(x_j)$. By
the definition of $\Xi_0$ and $q_0$, the truth assignment represented by $\pi$
then makes $C_j$ evaluate to true. Thus, the following holds.
\begin{lemma}
    $\psi$ is satisfiable if and only if $\Xi_0 \cup \Xi_2\models
    q_0 \wedge q_2$.
\end{lemma}

\subsection{Construction of $\Xi_{3}$ and $q_3$}
Rule base $\Xi_3$ consists of rules \eqref{eq:xi:3:rule1}. Let $I_3$ be the
universal interpretation of $\Xi_0\cup \Xi_3$, then we clearly have
$I_0\subseteq I_3$. Rules \eqref{eq:xi:3:rule1} modify the tree encoded by
$\Xi_0$ by generating an edge labelled by $R$ connecting each node  to itself.
\begin{align}
    \top(x)  & \rightarrow \SELF_{R}(x)\wedge R(x,x) & \forall j\in\interval{1}{m}\label{eq:xi:3:rule1}
\end{align}
Our encoding of query $q_{3}$ uses variable $p_{n+1}$ from query
$q_0$, as well as fresh variables $x_0^j,\ldots, x_n^j$ uniquely associated to
each clause $C_j$. Furthermore, we associate to each clause $C_j$ the formula
$\varphi^j$ in \eqref{eq:q:3:rule1}. Then $q_{3}$ is the existential closure of
$\bigwedge_j \varphi^j$.
\begin{align}
    \varphi^j = C_j(x_0^j)\wedge R(x^j, x^j_0)\wedge\bigwedge_{i=1}^{n} R(x^j_{i-1}, x_i^j) \wedge R(x_n^j, p_{n+1})\label{eq:q:3:rule1}
\end{align}
It should be clear that $q_0 \wedge q_3$ is arborescent. Consider an arbitrary
substitution $\pi$ such that ${\pi(q_0 \wedge q_3)\subseteq I_3}$ and an
arbitrary clause $C_j$. Then, formula $\varphi^j$ ensures that a path of length
$n+1$ exists connecting leaf node $\pi(p_{n+1})$ to one arbitrary node
occurring in the unique path connecting $\pi(p_{n+1})$ to root $a$. Since the
tree has depth $n+1$ and the leaves and the root of the tree do no satisfy
concept $C_j$, variable $x_0^j$ must be mapped to some node at depth $1\leq i
\leq n$. By the definition of $\Xi_0$ and $q_0$, the truth assignment
represented by $\pi$ then makes $C_j$ evaluate to true. Thus, the following
holds.
\begin{lemma}
    $\psi$ is satisfiable if and only if $\Xi_0 \cup \Xi_3\models
    q_0 \wedge q_3$.
\end{lemma}
}{}

\end{document}